\documentclass[journal,12pt,onecolumn,doublespacing,twoside]{IEEEtran}

\usepackage{versions}

\excludeversion{ttwocolumn}
\includeversion{oonecolumn}

\usepackage[utf8]{inputenc}
\usepackage[T1]{fontenc}
\usepackage[english]{babel}
\usepackage{amsmath, amsfonts, amssymb}
\usepackage[ruled,vlined]{algorithm2e}
\usepackage{array}
\usepackage{textcomp}
\usepackage{stfloats}
\usepackage{verbatim}
\usepackage{graphicx}
\usepackage{cite}
\usepackage{import}
\usepackage{bm}
\usepackage{setspace}
\usepackage{dsfont}
\usepackage{indentfirst}
\usepackage[hidelinks]{hyperref}
\usepackage{mathtools}
\usepackage{extarrows}

\usepackage{enumitem}
\usepackage{amsthm}
\usepackage{dsfont}
\usepackage[caption=false]{subfig}
\usepackage[export]{adjustbox}
\usepackage[all]{xy}
\usepackage{tabularx,booktabs,colortbl}
\usepackage{float}
\usepackage{multirow}
\usepackage[most]{tcolorbox}
\usepackage{etoolbox}

\def\BibTeX{{\rm B\kern-.05em{\sc i\kern-.025em b}\kern-.08em
    T\kern-.1667em\lower.7ex\hbox{E}\kern-.125emX}}
\AtBeginDocument{\definecolor{tmlcncolor}{cmyk}{0.93,0.59,0.15,0.02}\definecolor{TMLCNCOLOR}{cmyk}{0.93,0.59,0.15,0.02}\definecolor{NavyBlue}{RGB}{0,86,125}}

\newcommand{\rbrk}[1]{\left( #1 \right)}
\newcommand{\cbrk}[1]{\left\{ #1 \right\}}
\newcommand{\sbrk}[1]{\left[ #1 \right]}
\newcommand{\herm}{{\mathsf{H}}}
\newcommand{\trans}{{\mathsf{T}}}

\newtheorem{proposition}{Proposition}

\newtheorem{lemma}{Lemma}
\theoremstyle{definition}
\newtheorem{remark}{Remark}

\usepackage{soul, xcolor}

\usepackage[acronym]{glossaries}

\newcommand{\channel}{\mathbf{h}}
\newcommand{\userset}{\mathcal{K}}
\newcommand{\agentset}{\mathcal{A}}
\newcommand{\agentindex}{a}
\newcommand{\userindex}{k}
\newcommand{\numantennas}{N}
\newcommand{\numsubcarriers}{S}

\newcommand{\beamset}{\mathcal{B}}
\newcommand{\beam}{\mathbf{b}}
\newcommand{\numbeams}{M}
\newcommand{\oversampling}{\vartheta}
\newcommand{\beammatrix}{\mathbf{B}}
\newcommand{\receivedsignal}{y}
\newcommand{\informationsignal}{s}
\newcommand{\noisesignal}{n}
\newcommand{\noisepower}{\sigma}
\newcommand{\sinr}{\psi}
\newcommand{\setcameras}{\mathcal{C}}
\newcommand{\cameraindex}{c}
\newcommand{\image}{\mathbf{i}}
\newcommand{\imageheight}{H}
\newcommand{\imagewidth}{W}
\newcommand{\bandwidth}{\beta}
\newcommand{\bitrate}{\varrho}
\newcommand{\envobservation}{\tilde{\mathbf{o}}}
\newcommand{\statespace}{\mathcal{S}}
\newcommand{\state}{\mathbf{s}}
\newcommand{\observationspace}{\mathcal{O}}
\newcommand{\observation}{\mathbf{o}}
\newcommand{\reward}{R}
\newcommand{\discountfactor}{\gamma}
\newcommand{\transitionfunction}{T}
\newcommand{\commgraph}{\mathcal{G}}
\newcommand{\commgraphedgeset}{\mathcal{E}}
\newcommand{\userposition}{\mathbf{p}}
\newcommand{\agentposition}{\mathbf{r}}
\newcommand{\userpositionxy}{p}
\newcommand{\estimateduserposition}{\bar{\userposition}}
\newcommand{\edgefeatures}{\mathbf{e}}
\newcommand{\edgefeaturesset}{\mathcal{F}}
\newcommand{\policy}{\pi}
\newcommand{\valuefunction}{V}
\newcommand{\group}{G}
\newcommand{\groupidentityelement}{\epsilon}
\newcommand{\binaryop}{\cdot}
\newcommand{\cyclicgroup}{C}
\newcommand{\rotationmatrix}{R}
\newcommand{\groupaction}{L}
\newcommand{\groupactiontwo}{K}
\newcommand{\groupactionedge}{U}
\newcommand{\permutation}{P}
\newcommand{\cameraheight}{h}
\newcommand{\azimuthangle}{\varphi}
\newcommand{\elevationangle}{\phi}
\newcommand{\pixelcoordinate}{l}
\newcommand{\dataset}{\mathcal{D}}
\newcommand{\unlabeledcsidataset}{\mathcal{U}}
\newcommand{\lowdimdataset}{\mathcal{Z}}
\newcommand{\matriximagesensing}{\mathbf{Z}}
\newcommand{\offlinetrainingperiod}{t_{\text{train}}}
\newcommand{\distancematrix}{\mathbf{D}}
\newcommand{\numsamples}{D}
\newcommand{\channeldistance}{d}
\newcommand{\matchingmatrix}{\mathbf{M}}
\newcommand{\matchingmatrixset}{\mathcal{M}}
\newcommand{\matchingscalar}{\eta}
\newcommand{\chartingfunction}{\zeta}
\newcommand{\chartingfunctionparams}{\omega}
\newcommand{\ccregularizer}{\lambda}
\newcommand{\sensingduplicateparameter}{\Delta}
\newcommand{\neuralnetweight}{\mathbf{W}}
\newcommand{\neuralnetweightset}{\mathcal{W}}
\newcommand{\neuralnetbias}{\mathbf{w}}
\newcommand{\neuralnetinput}{\mathbf{y}}
\newcommand{\symmetrizer}{S}
\newcommand{\neuralnetbasis}{\mathbf{V}}
\newcommand{\neuralnetbasisfactor}{c}
\newcommand{\policyparams}{\theta}
\newcommand{\criticparams}{\xi}
\newcommand{\obsencoder}{x}
\newcommand{\encodingdim}{X}
\newcommand{\messages}{m}
\newcommand{\aggregatedmessages}{\mathbf{m}}
\newcommand{\messagesize}{M}
\newcommand{\update}{u}
\newcommand{\localpolicy}{\mu}
\newcommand{\localvalue}{\nu}
\newcommand{\estimatedstate}{\bar{\state}}
\newcommand{\numlayers}{Q}
\newcommand{\layerindex}{q}
\newcommand{\encoding}{\mathbf{x}}
\newcommand{\advantage}{A}
\newcommand{\advantagefactor}{\aleph}
\newcommand{\advantagediscount}{\chi}
\newcommand{\ppoclip}{\kappa}
\newcommand{\ppoentropyfactor}{\alpha}

\newcommand{\actionvaluefunction}{\mathcal{Q}}
\newcommand{\partialsymmetryR}{\varepsilon_{_R}}
\newcommand{\partialsymmetryT}{\varepsilon_{_T}}

\begin{document}

\title{Equivariant Multi-agent Reinforcement Learning for Multimodal Vehicle-to-Infrastructure Systems}

\author{%
\IEEEauthorblockN{%
Charbel Bou Chaaya,~\IEEEmembership{Student Member,~IEEE}, and Mehdi Bennis,~\IEEEmembership{Fellow,~IEEE}
}
\thanks{%
The authors are with the Centre for Wireless Communications, University of Oulu, Finland (email: charbel.bouchaaya@oulu.fi; mehdi.bennis@oulu.fi).
}
}

\markboth{ }{ }

\maketitle

\newacronym{ml}{ML}{machine learning}
\newacronym{ai}{AI}{artificial intelligence}
\newacronym{csi}{CSI}{channel state information}
\newacronym{ofdm}{OFDM}{orthogonal frequency division multiplex}
\newacronym{rnn}{RNN}{recurrent neural network}
\newacronym{ema}{EMA}{exponential moving average}
\newacronym{mlp}{MLP}{multi-layer perceptron}
\newacronym{gru}{GRU}{gated recurrent unit}
\newacronym{lstm}{LSTM}{long short-term memory}
\newacronym{1nn}{$1$--NN}{nearest neighbor}
\newacronym{adp}{ADP}{angle-delay profile}
\newacronym{mdp}{MDP}{Markov decision process}
\newacronym{mmdp}{MMDP}{multi-agent Markov decision process}
\newacronym{snr}{SNR}{signal-to-noise ratio}
\newacronym{sinr}{SINR}{signal-to-interference-plus-noise ratio}
\newacronym{rssm}{RSSM}{recurrent state-space model}
\newacronym{rl}{RL}{reinforcement learning}
\newacronym{marl}{MARL}{multi-agent reinforcement learning}
\newacronym{v2i}{V2I}{vehicle-to-infrastructure}
\newacronym{ris}{RIS}{reconfigurable intelligent surface}
\newacronym{stl}{STL}{signal temporal logic}
\newacronym{mmwave}{mmWave}{millimeter wave}
\newacronym{rf}{RF}{radio frequency}
\newacronym{gps}{GPS}{global positioning system}
\newacronym{cnn}{CNN}{convolutional neural network}
\newacronym{bs}{BS}{base station}
\newacronym{nn}{NN}{neural network}
\newacronym{ipm}{IPM}{inverse perspective mapping}
\newacronym{elm}{ELM}{extreme learning machine}
\newacronym{moa}{MOA}{model of other agent}
\newacronym{rsu}{RSU}{road-side unit}
\newacronym{dft}{DFT}{discrete Fourier transform}
\newacronym{gnn}{GNN}{graph neural network}
\newacronym{ppo}{PPO}{proximal policy optimization}
\newacronym{gae}{GAE}{generalized advantage estimation}

\glsdisablehyper

\begin{abstract}
In this paper, we study a \gls{v2i} system where distributed \glspl{bs} acting as \glspl{rsu} collect multimodal (wireless and visual) data from moving vehicles.
We consider a decentralized rate maximization problem, where each \gls{rsu} relies on its local observations to optimize its resources, while all \glspl{rsu} must collaborate to guarantee favorable network performance.
We recast this problem as a distributed \gls{marl} problem, by incorporating rotation symmetries in terms of vehicles' locations.
To exploit these symmetries, we propose a novel self-supervised learning framework where each \gls{bs} agent aligns the latent features of its multimodal observation to extract the positions of the vehicles in its local region.
Equipped with this sensing data at each \gls{rsu}, we train an equivariant policy network using a \gls{gnn} with message passing layers, such that each agent computes its policy locally, while all agents coordinate their policies via a signaling scheme that overcomes partial observability and guarantees the equivariance of the global policy.
We present numerical results carried out in a simulation environment, where ray-tracing and computer graphics are used to collect wireless and visual data.
Results show the generalizability of our self-supervised and multimodal sensing approach, achieving more than two-fold accuracy gains over baselines, and the efficiency of our equivariant \gls{marl} training, attaining more than 50\% performance gains over standard approaches.
%
%
\end{abstract}


\glsresetall

\maketitle

\section{Introduction}
Beyond 5G and 6G wireless systems will support a plethora of advanced and emerging applications and services, such as intelligent factories, autonomous vehicles, and digital twins~\cite{saad2019vision}.
While researchers are exploring possible enabling technologies, the trend of self-driven networks is becoming more and more conspicuous, where wireless systems need to adapt, reconfigure and optimize their functions to cater for user demands under unspecified conditions~\cite{chafii2023twelve}.
This comes in tandem with the remarkable advances in the \gls{ai} field and their impact on wireless communications in particular.

The migration towards higher frequency ranges, such as the \gls{mmwave} band, presents a promising solution for the stringent demands of upcoming technologies.
This is due to their abundant bandwidths that can provide high data rates and precise sensing capabilities~\cite{rappaport2019wireless}.
However, sustaining a communication system on these bands comes with many challenges, mainly due to the intermittent channel quality.
Thus, extensive signaling and frequent channel estimations are required for beam alignment, which prevents achieving low latency communications.
Added to that, the increasing amount of antennas and reflective elements, scaling from hundreds to thousands, and their ubiquitous distribution to relax any beamforming constraints, complicates the physical layer optimization and coordination procedures to ensure a seamless user experience.

An emerging trend that could solve some of these problems is the exploitation of overhead free modalities, such as images, point clouds for downstream wireless optimization~\cite{bariah2023large}.
In fact, the optimization of wireless systems relies mainly on \gls{csi} to tune the network's parameters.
Equipping the system with different observations from various modalities can provide more degrees of freedom and cost-effective solutions.
Nevertheless, this posits new questions in terms of how to leverage these inherently different perception modules, and whether the gains they yield justify the energy and computation footprints of running new equipment for collecting and processing such data.

Substantially, different modalities capture the state of the environment in essentially different formats (\gls{csi}, images, point clouds, etc).
Hence, fusing various modalities to extract common representations facilitating a downstream task is crucial.
However, unlike image captioning with text, annotating multimodal data for wireless tasks is highly challenging, since this requires labeling images with their optimal beam direction for instance.
Accordingly, compiling such datasets and training supervised models is an expensive and unscalable approach, calling for novel self-supervised fusion methods.

On the other hand, heavy volumes of multimodal data are collected by sensors distributed in the wireless environment.
Communicating such data to a central processing server is also not feasible since this would incur intolerable delays and degrade the network's performance.
Therefore, sensors must locally process their observations and tune their parameters, while maintaining coordination with other sensors to ensure the system's reliability.

In addition, many wireless optimization problems, and \gls{v2i} systems in particular, exhibit inherent geometric symmetries that are often overlooked by learning-based solutions.
For instance, the relative positioning of users around a \gls{bs}, governs key physical layer quantities such as path losses, wireless channels, and beam directions.
As a result, rotating the spatial layout of vehicles around a \gls{rsu} in a \gls{v2i} network, induces a corresponding permutation of optimal beamforming and resource allocation decisions, as the system's infrastructure is typically symmetrical.
Explicitly leveraging these symmetries thus allows developing learning architectures that simplify physical layer optimization, by constraining the search space to solutions satisfying the desired symmetrical property.

Jointly, these considerations motivate a decentralized learning problem where distributed \glspl{rsu} must rely on unlabeled multimodal observations to establish their situational awareness of the \gls{v2i} environment and optimize the network's rate through beam alignment.
The resulting challenges is thus to design scalable learning methods that can \emph{self-supervise the fusion} of heterogeneous sensing modalities at each \gls{rsu} node, and \emph{exploit distributed symmetries} in the network to enable coordinated decision making without centralized exchange of heavy multimodal data.
All aforementioned open research questions motivate our study.
%
\subsection{Contributions}
In this work, we study a wireless \gls{v2i} network where distributed \glspl{bs} collect wireless \gls{csi} from navigating vehicles and image snapshots of the road sections they serve.
Unlike previous works, we do not assume neither the knowledge of the matching between the modalities, nor task labels for the modalities.
We consider a decentralized rate maximization problem where each \gls{bs} observes only its local multimodal data, and all \glspl{bs} must coordinate their resource management decisions to satisfy the users' utilities.
Our main contributions are summarized as follows:
\begin{itemize}[nolistsep, leftmargin=*]
    \item
        Taking advantage of rotation symmetries occuring in \gls{v2i} networks, we recast our problem as a distributed~\gls{mmdp} with symmetries, a particular class of \glspl{mmdp} that admits symmetric optimal policies, significantly reducing the solution search space.
        We argue that the optimal policy undergoes a permutation when the positions of vehicles rotate in the network,
    \item Our \emph{first key contribution} is a novel self-supervised multimodal learning framework for sensing and imputation, allowing each \gls{bs} agent to align the features of its unlabeled image and \gls{csi} data to extract the locations of the users in its region,
    \item Given the locally estimated locations for each agent, our \emph{second key contribution} is a \gls{gnn} based equivariant policy network that leverages the symmetries of the environment for improved \gls{marl} training,
    \item 
        We provide extensive simulation results from a synthetic dataset with co-existing visual and wireless data.
        Our results show the effectiveness and generalizability of our self-supervised approach in aligning multimodal data in latent space, achieving less than 1.5 m average localization error, and 50\% performance gains over benchmarks for our symmetric \gls{marl} approach.
\end{itemize}

\subsection{Prior Works}
Solving wireless problems relying on non wireless observations (other than \gls{csi}) has been tackled from multiple directions~\cite{roy2023going}.
Users' locations obtained from a \gls{gps} system have been used for a multitude of tasks, such as \gls{mmwave} beam direction prediction in \gls{v2i} networks~\cite{va2016beam}.
Camera-based beamforming exploits recent advances in \gls{ml} and computer vision to process images, mainly using \gls{cnn} based architectures for detecting users.
For instance,~\cite{tian2020applying} proposed an autoregressive framework to predict future beam indices from sequences of images based on \glspl{cnn} and \gls{lstm} modules.
Camera images were used in~\cite{shaifeng2023sensing} to reduce the search for \gls{ris} phase configurations during initial access.
Lidar point clouds were used in~\cite{jiang2022lidar} to train a \gls{rnn} to predict favorable beam directions in a \gls{v2i} system.
The primary limitation of these works stems from their supervised learning core, using labeled training datasets for \gls{ml} model training.

Using more than one modality has also received attention in the recent literature.
In~\cite{charan2021vision}, objects detected in snapshots of the environment are paired with a sequence of previous beam indices and fed to a \gls{rnn} to proactively anticipate possible user blockage.
In~\cite{gouranga2022vision}, the user's position is combined with its location in the image to form an input of a beam prediction network.
Similarly,~\cite{reus2021deep} trained a model that fuses features extracted from an image with the user's \gls{gps} location to predict any blockage event and the optimal beam direction.
The authors of~\cite{salehi2022deep} considered a \gls{v2i} scenario where a vehicle is equipped with Lidar and \gls{gps} sensors, while the \gls{bs} tracks it using a camera whose goal is to fuse all three modalities for optimal beam selection.
A distributed approach is studied in~\cite{imran2024environment}, where multimodal sensors locally process their data with pre-trained models, and communicate the obtained representations instead of the raw data to a central \gls{bs} to optimize its beam direction.
While interesting, all these works still rely on a supervised training curriculum that combines the features extracted from each modality, assuming the modalities are perfectly matched, and accessing downstream task labels.

Recent attempts~\cite{huan2024multi},~\cite{huan2024multi2} have also examined the case where the matching between the modalities (which \gls{csi} corresponds to which user in the image) is unknown.
A self-supervised learning approach using unlabeled \gls{csi} and bounding box detection from images is applied to cluster representations, followed by a supervised fine-tuning on a smaller labeled dataset.
In~\cite{alloulah2022self}, self-supervised pre-training is used to align the representations of radar heatmaps with corresponding visual image data while assuming the pairings are known.

Concurrently with the study of multimodality, our work is also related to research on group symmetries in distributed decision making problems.
In our setting, symmetries constitute spatial transformations that occur to the environment's state (such as rotations, translations, etc.), for which the agent's corresponding action must be a transformed version of its action in the original state (for example, a permuted action).
Although the literature is still limited on this topic, accounting for symmetries occurring in the environment has been shown to significantly enhance the training efficiency in both single and \gls{marl} tasks~\cite{van2020mdp},~\cite{jianye2022boosting},~\cite{yu2023esp},~\cite{pol2022multiagent}.
Recently,~\cite{zhou2024symmetry} and~\cite{shi2025symmetry} proposed novel frameworks for aerial \glspl{bs} trajectory design and coverage problems, while incorporating symmetries in their training scheme to improve sample efficiency.
In~\cite{zhou2024symmetry}, data augmentations are used to expand the experience of learning \glspl{bs} agents, hence reducing the need for extensive environment interactions.
The authors of~\cite{shi2025symmetry} proposed a policy neural network that satisfies the desired symmetric property of the actions, hence constraining the agents to appropriately transform their actions whenever symmetries occur.
The main limitation of such works is the assumption that symmetries directly act on the agents' observations.
In practice however, \glspl{bs} do not access the locations of end users, and rather estimate their high-dimensional \gls{csi}, which do not obey symmetric properties (such as rotations) like the locations.
Hence, the symmetries in such wireless networks are \emph{latent}, since they act on low-dimensional features (user positions) which must be extracted from the agents' observations.

\section{System Model and Problem Formulation}
\label{section:mm_paper_system_model}
%
We consider the downlink of a \gls{mmwave} wireless system where multiple \glspl{bs} are deployed as \glspl{rsu} serving the \gls{v2i} communication with passing vehicles, as shown in Fig.~\ref{fig:mm_paper_system_model}.
We denote the set of \glspl{bs} and users by $\agentset$ and $\userset$, respectively.%
\begin{oonecolumn}%
\begin{figure*}%
\centering
\includegraphics[width=.8\textwidth]{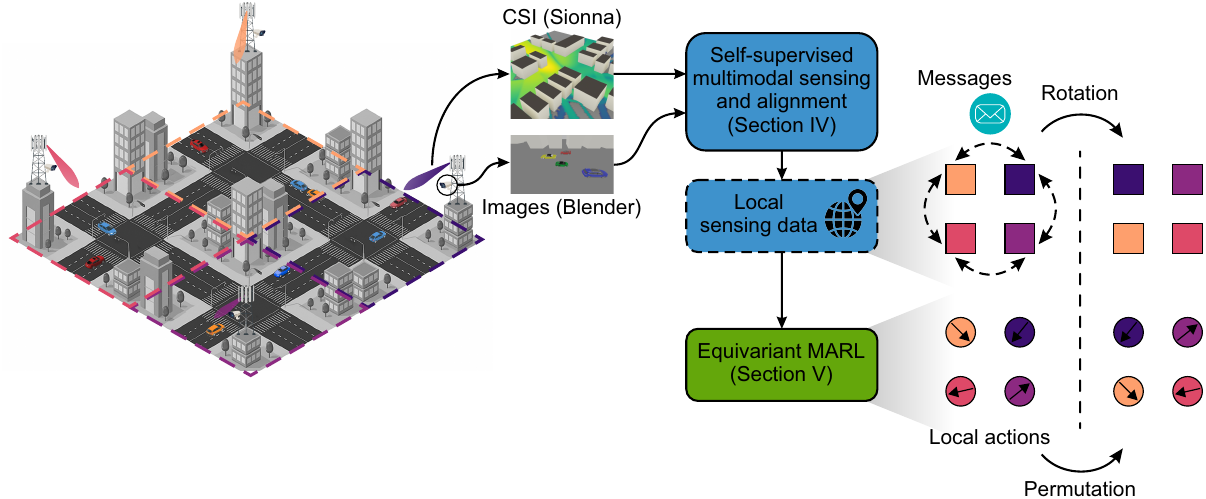}
\caption{System model showcasing multiple \gls{bs} agents, each equipped with a camera, deployed in a symmetric \gls{v2i} environment. Each agent aligns latent features of its multimodal observations to retrieve the locations of the vehicles in its local region, which are used to communicate with its neighbors and train an equivariant \gls{marl} policy.}
\label{fig:mm_paper_system_model}
\end{figure*}%
\end{oonecolumn}%
\begin{oonecolumn}
\begin{table}[H]
\centering
\caption{Definition of symbols used throughout this paper}
\label{tab:mm_paper_symbol_definition}
\renewcommand{\arraystretch}{0.8}
\begin{tabularx}{\textwidth} { 
  >{\raggedright\arraybackslash\hsize=.15\hsize}X 
  >{\raggedright\arraybackslash}X
  >{\raggedright\arraybackslash\hsize=.15\hsize}X 
  >{\raggedright\arraybackslash}X
  }
    \toprule
    Symbol & Definition & Symbol & Definition \\
    \midrule
    $\agentset, \userset$ & Sets of \gls{bs}-\gls{rsu} agents and users &
    $\agentindex, \userindex$ & Agent and user index \\
    \rowcolor{gray!15}
    $\userset_\agentindex$ & Set of users served by agent $\agentindex$ &
    $\bandwidth, \numantennas$ & Bandwidth and number of antennas \\
    $\beamset_\agentindex$ & Set of beams / actions of agent $\agentindex$ &
    $\beam_\agentindex$ & Beam / action of agent $\agentindex$ \\
    \rowcolor{gray!15}
    $\channel_{\agentindex, \userindex}$ & Wireless channel between \gls{bs} $\agentindex$ and user $\userindex$ &
    $\sinr_\userindex, \bitrate_\userindex$ & Downlink SINR and rate of user $\userindex$ \\
    $\setcameras_\agentindex, \cameraindex$ & Set of cameras of agent $\agentindex$ and camera index &
    $\image_{\agentindex, \cameraindex}$ & Image taken by camera $\cameraindex$ of agent $\agentindex$ \\
    \rowcolor{gray!15}
    $\cameraheight_\cameraindex$ & Camera height &
    $\azimuthangle^{\text{LoS}}_\cameraindex, \elevationangle^{\text{LoS}}_\cameraindex$ & Camera line-of-sight viewing angles \\
    $\azimuthangle^{\text{max}}_\cameraindex, \elevationangle^{\text{max}}_\cameraindex$ & Maximal horizontal and vertical viewing angles &
    $\userset_\agentindex^{\text{csi}}$ & Set of users with \gls{csi} modality available to \gls{bs} $\agentindex$ \\
    \rowcolor{gray!15}
    $\userset_\agentindex^{\text{im}}$ & Set of users with image modality available to \gls{bs} $\agentindex$ &
    $\envobservation_\agentindex$ & Multimodal environmental observation of agent $\agentindex$ \\
    $\agentposition_\agentindex, \userposition_\userindex$ & Locations of \gls{rsu} $\agentindex$ and user $\userindex$ &
    $\transitionfunction, \reward, \discountfactor$ & Transition and reward function and discount factor \\
    \rowcolor{gray!15}
    $\statespace, \observationspace$ & State and observation spaces &
    $\commgraph, \commgraphedgeset$ & Communication graph and its set of edges \\
    $\edgefeatures_{\agentindex, \agentindex^\prime}$ & Edge feature between agents $\agentindex$ and $\agentindex^\prime$ &
    $\policy_\agentindex$ & Policy of agent $\agentindex$ \\
    \rowcolor{gray!15}
    $\group, C_4$ & Group and rotation group &
    $\groupaction_g, \groupactiontwo_g$ & Group action on states and actions \\
    $\permutation_g$ & Pemrutation &
    $\dataset_\agentindex^{\text{im}}, \dataset_\agentindex^{\text{csi}}$ & Unlabed image and \gls{csi} data sets \\
    \rowcolor{gray!15}
    $\distancematrix_\agentindex^{\text{im}}, \distancematrix_\agentindex^{\text{csi}}$ & Distances matrices &
    $\matchingmatrix_\agentindex, \matchingscalar_\agentindex$ & Matching matrix and scaling factor \\
    $\chartingfunction_\agentindex$ & \gls{csi} sensing function &
    $\obsencoder_\agentindex$ & State encoder \\
    \rowcolor{gray!15}
    $\messages_\agentindex, \update_\agentindex$ & Message and update functions &
    $\messages_{\agentindex \to \agentindex^\prime}^\layerindex$ & Message from $\agentindex$ to $\agentindex^\prime$ at layer $\layerindex$ \\
    $\localpolicy_\agentindex, \localvalue_\agentindex$ & Policy and value heads &
    $\encoding^\layerindex_\agentindex, \aggregatedmessages^\layerindex_\agentindex$ & Features and aggregated messages at layer $\layerindex$ \\
    \rowcolor{gray!15}
    $\numlayers$ & Message passing layers / rounds \\
    \bottomrule
\end{tabularx}
\end{table}
\end{oonecolumn}

\begin{ttwocolumn}
\begin{figure*}
\centering
\begin{table}[H]
\centering
\caption{Definition of symbols used throughout this paper}
\label{tab:mm_paper_symbol_definition}
\renewcommand{\arraystretch}{0.9}
\begin{tabularx}{\textwidth} { 
  >{\raggedright\arraybackslash\hsize=.15\hsize}X 
  >{\raggedright\arraybackslash}X
  >{\raggedright\arraybackslash\hsize=.15\hsize}X 
  >{\raggedright\arraybackslash}X
  }
    \toprule
    Symbol & Definition & Symbol & Definition \\
    \midrule
    $\agentset, \userset$ & Sets of \gls{bs}-\gls{rsu} agents and users &
    $\agentindex, \userindex$ & Agent and user index \\
    %
    $\userset_\agentindex$ & Set of users served by agent $\agentindex$ &
    $\bandwidth, \numantennas$ & Bandwidth and number of antennas \\
    $\beamset_\agentindex$ & Set of beams / actions of agent $\agentindex$ &
    $\beam_\agentindex$ & Beam / action of agent $\agentindex$ \\
    %
    $\channel_{\agentindex, \userindex}$ & Wireless channel between \gls{bs} $\agentindex$ and user $\userindex$ &
    $\sinr_\userindex, \bitrate_\userindex$ & Downlink SINR and rate of user $\userindex$ \\
    $\setcameras_\agentindex, \cameraindex$ & Set of cameras of agent $\agentindex$ and camera index &
    $\image_{\agentindex, \cameraindex}$ & Image taken by camera $\cameraindex$ of agent $\agentindex$ \\
    %
    $\cameraheight_\cameraindex$ & Camera height &
    $\azimuthangle^{\text{LoS}}_\cameraindex, \elevationangle^{\text{LoS}}_\cameraindex$ & Camera line-of-sight viewing angles \\
    $\azimuthangle^{\text{max}}_\cameraindex, \elevationangle^{\text{max}}_\cameraindex$ & Maximal horizontal and vertical viewing angles &
    $\userset_\agentindex^{\text{csi}}$ & Set of users with \gls{csi} modality available to \gls{bs} $\agentindex$ \\
    %
    $\userset_\agentindex^{\text{im}}$ & Set of users with image modality available to \gls{bs} $\agentindex$ &
    $\envobservation_\agentindex$ & Multimodal environmental observation of agent $\agentindex$ \\
    $\agentposition_\agentindex, \userposition_\userindex$ & Locations of \gls{rsu} $\agentindex$ and user $\userindex$ &
    $\transitionfunction, \reward, \discountfactor$ & Transition and reward function and discount factor \\
    %
    $\statespace, \observationspace$ & State and observation spaces &
    $\commgraph, \commgraphedgeset$ & Communication graph and its set of edges \\
    $\edgefeatures_{\agentindex, \agentindex^\prime}$ & Edge feature between agents $\agentindex$ and $\agentindex^\prime$ &
    $\policy_\agentindex$ & Policy of agent $\agentindex$ \\
    %
    $\group, C_4$ & Group and rotation group &
    $\groupaction_g, \groupactiontwo_g$ & Group action on states and actions \\
    $\permutation_g$ & Pemrutation &
    $\dataset_\agentindex^{\text{im}}, \dataset_\agentindex^{\text{csi}}$ & Unlabed image and \gls{csi} data sets \\
    %
    $\distancematrix_\agentindex^{\text{im}}, \distancematrix_\agentindex^{\text{csi}}$ & Distances matrices &
    $\matchingmatrix_\agentindex, \matchingscalar_\agentindex$ & Matching matrix and scaling factor \\
    $\chartingfunction_\agentindex$ & \gls{csi} sensing function &
    $\obsencoder_\agentindex$ & State encoder \\
    %
    $\messages_\agentindex, \update_\agentindex$ & Message and update functions &
    $\messages_{\agentindex \to \agentindex^\prime}^\layerindex$ & Message from $\agentindex$ to $\agentindex^\prime$ at layer $\layerindex$ \\
    $\localpolicy_\agentindex, \localvalue_\agentindex$ & Policy and value heads &
    $\encoding^\layerindex_\agentindex, \aggregatedmessages^\layerindex_\agentindex$ & Features and aggregated messages at layer $\layerindex$ \\
    %
    $\numlayers$ & Message passing layers / rounds \\
    \bottomrule
\end{tabularx}
\end{table}
\end{figure*}
\end{ttwocolumn}%
We let $\userset_\agentindex$ represent the set of users served by \gls{bs} $\agentindex \in \agentset$, while satisfying $\cup_{\agentindex \in \agentset} \userset_\agentindex = \userset$.
Each \gls{bs} is capable of sensing the wireless channel between itself and the vehicles in its region, while also controlling one or more mounted RGB cameras.
To estimate the channel, the \gls{bs} decodes reverse link pilots sent by each user during predefined time slots.
Simultaneously, the RGB cameras capture a snapshot of the scene in front of the \gls{bs}.
As typically assumed in prior studies~\cite{shaifeng2023sensing,jiang2022lidar,gouranga2022vision, reus2021deep,koda2020cam}, we consider that the channel estimation module is fully synchronized with the camera hardware, hence the users' channels and cameras' images are collected simultaneously.
Note that equipping \glspl{rsu} with cameras is not uncommon in the literature~\cite{shaifeng2022computer, koda2020cam} and has shown various benefits such as low overhead beam training and proactive handover.
We further note that the images taken by the different cameras of each \gls{bs} depict only the road sections served by the corresponding \gls{bs}.
Therefore, each \gls{bs} $\agentindex \in \agentset$ holds only a partial observation comprising channels and image realizations from its users $\userset_\agentindex$.
However, all \glspl{bs} communicate with all users $\userset$ on the same \gls{mmwave} band, denoted by $\bandwidth$.

Table~\ref{tab:mm_paper_symbol_definition} summarizes the symbol notations used in the remainder of this paper.
We now formalize the wireless and image signal models, before formulating our problem.
%

\subsection{Communication and Channel Models}
Each vehicle $\userindex \in \userset$ is modeled as a single antenna receiver%
\footnote{%
We adopt this simplifying assumption similarly to~\cite{shaifeng2023sensing},~\cite{jiang2022lidar},~\cite{charan2021vision},~\cite{gouranga2022vision},~\cite{imran2024environment},~\cite{huan2024multi},~\cite{huan2024multi2}, while noting that in practical \gls{mmwave} networks, users are equipped with antenna arrays to achieve sufficient beamforming gains.
We however stress that our proposed optimization framework is not fundamentally tied to this assumption, and can be extended to account for multi-antenna terminals by augmenting the decision variables to account for user side beamforming.
}%
.
Each \gls{bs} $\agentindex \in \agentset$ is modeled as a uniform planar array comprising $\numantennas = \numantennas_h \times \numantennas_v$ antennas, where $\numantennas_h$ and $\numantennas_v$ represent the number of antennas in the horizontal and vertical dimensions respectively.
We assume that each \gls{bs} employs a two dimensional \gls{dft} codebook of beams $\beamset_\agentindex = \cbrk{\beam_i \in \mathbb{C}^{\numantennas} \mid i = 1, \dots, \numbeams}$, where the number of beams $\numbeams=\numantennas \times \oversampling_h \times \oversampling_v$, $\oversampling_h$ and $\oversampling_v$ are oversampling factors in the horizontal and vertical dimensions.
The possible beams $\beam_i$ are the rows of the matrix $\frac{1}{\sqrt{\numantennas}} \beammatrix_v \odot \beammatrix_h$, where row $r$ of the one dimensional \gls{dft} matrix $\beammatrix_h$ is defined as ($1 \leq r \leq \numantennas_h \oversampling_h$):
\begin{equation}
    \sbrk{\beammatrix_h}_r = \sbrk{1 \quad e^{j \frac{2 \pi \rbrk{r-1}}{\numantennas_h \oversampling_h}} \quad ... \quad e^{j \frac{2 \pi \rbrk{r-1} \rbrk{\numantennas_h-1}}{\numantennas_h \oversampling_h}}},
\end{equation}
while $\beammatrix_v$ is similarly defined using $\rbrk{\numantennas_v, \oversampling_v}$, and $\odot$ is the Kronecker product.

We denote by $\channel_{\agentindex, \userindex}^t \in \mathbb{C}^{\numantennas \times \numsubcarriers}$ as the wireless channel between \gls{bs} $\agentindex \in \agentset$ and user $\userindex \in \userset$ at time slot $t$, estimated at $\numsubcarriers$ \gls{ofdm} subcarriers.
Thus, the received base-band signal at terminal $\userindex$ is%
\footnote{In~\eqref{eq:mm_paper_received_signal} and~\eqref{eq:mm_paper_sinr}, we abuse the channel notation by letting $\channel_{\agentindex, \userindex}$ represent the channel at the central subcarrier only, for simplicity. However the downlink rate optimization framework we develop can be extended to a sum rate over all subcarriers in a straightforward manner.}%
:
\begin{equation}\label{eq:mm_paper_received_signal}
    \receivedsignal_\userindex^t = \sum_{\agentindex \in \agentset} \channel^{t \,\herm}_{\agentindex, \userindex} \, \beam_{\agentindex}^t {\sum_{j \in \userset_\agentindex} \informationsignal_j^t} + \noisesignal_\userindex^t,
\end{equation}
where $\beam_{\agentindex}^t \in \beamset_\agentindex$ is the beam chosen by agent $\agentindex$, $\informationsignal^t_\userindex$ is the information signal designated to user $\userindex$, $\noisesignal^t_\userindex \sim \mathcal{CN}\rbrk{0, \noisepower^2}$ is the additive noise with power $\noisepower^2$.
As such, the \gls{sinr} of vehicle $\userindex$ at time $t$ can be written as:
\begin{equation}\label{eq:mm_paper_sinr}
    \sinr_\userindex^t = \frac
    {\left\lvert \channel_{\agentindex_\userindex, \userindex}^{t \, \herm} \, \beam^t_{\agentindex_\userindex} \right\rvert^2}
    {\left\lvert \sum_{j \neq \userindex} \channel_{\agentindex_j, \userindex}^{t \, \herm} \, \beam^t_{\agentindex_j} \right\rvert^2 + \noisepower^2}
\end{equation}
where $\agentindex_\userindex$ denotes the \gls{bs} serving the zone of vehicle $\userindex$.
Finally, the bitrate of terminal $\userindex$ can be expressed as: $\bitrate_\userindex^t = \bandwidth \log_2\rbrk{1 + \sinr_\userindex^t}$.
Note that while directional \gls{mmwave} channels generally limit inter-\gls{rsu} interference, in dense \gls{v2i} deployments with mobile vehicles, interference and strong coupling between adjacent \glspl{rsu} can still arise mainly due to beam misalignment, and different \gls{rsu} beam collisions when vehicles are near the edge of an \gls{rsu} coverage region.

In this work, we do not assume any particular channel model, and generate the channels using the ray-tracing software Sionna~\cite{hoydis2022sionna}.
To do so, we first create a \gls{v2i} scene in the gaming engine Blender, which incorporates different aspects of the scenarios, such as buildings and vehicles.
This scene is then taken as input to Sionna, so that realistic \gls{csi} is generated using ray-tracing.
We detail our simulation setting in Section~\ref{section:mm_paper_simulation}.

\subsection{Image Model}
Each \gls{bs} $\agentindex \in \agentset$ controls a set $\setcameras_\agentindex$ of cameras which records RGB images, denoted by $\image_{\agentindex, \cameraindex}^t \in [0, 1]^{3 \times \imageheight_\cameraindex \times \imagewidth_\cameraindex}$ at slot $t$, where $\imageheight_\cameraindex$ and $\imagewidth_\cameraindex$ denote their height and width, respectively.
For instance, each camera $\cameraindex \in \setcameras_\agentindex$ takes a snapshot of a particular section of the zone served by its corresponding \gls{bs}.
We consider that each \gls{rsu} $\agentindex$ knows the height $\cameraheight_\cameraindex$, the line-of-sight viewing angles $\rbrk{\azimuthangle^{\text{LoS}}_\cameraindex, \elevationangle^{\text{LoS}}_\cameraindex}$, and the maximal horizontal and vertical viewing angles $\rbrk{\azimuthangle^{\text{max}}_\cameraindex, \elevationangle^{\text{max}}_\cameraindex}$ of its cameras $\cameraindex \in \setcameras_\agentindex$.
The heights are taken with respect to the vehicular roads, which we assume are all flat, whereas the angles can be in an arbitrary reference system.
As in~\cite{shaifeng2023sensing,gouranga2022vision, reus2021deep,koda2020cam}, we assume that the image collection is fully synchronized with the channel estimation phase, which occurs at the beginning of every communication slot.
While practical camera and CSI estimation hardware may exhibit clock offsets and processing delays, this assumption is adopted to simplify the system description and facilitate our presentation.
To assess the impact of such imperfections, we relax this assumption in Section~\ref{section:mm_paper_simulation} and evaluate the effect of timing offsets between the two modalities through simulations.

Similarly to wireless \gls{csi} obtained from Sionna, we capture images from \gls{bs} cameras positioned in the Blender scene.
Hence, the wireless channels are generated by Sionna simultaneously with the images collected by the cameras.

\subsection{Data Acquisition}
In this paper, we consider a general case where not necessarily all users served by a \gls{bs} $\agentindex \in \agentset$ appear in both modalities collected by $\agentindex$ (image and \gls{csi}).
In other words, it could be the case that the channel of a certain vehicle is not estimated by its \gls{bs}; or that a vehicle does not appear in the images of its \gls{bs}.
This could be due to a region not covered by the \gls{bs} cameras.
However, to keep our work tractable, we assume that each vehicle $\userindex\in\userset$ appears in at least one modality.

To formalize this idea, we let $\userset_\agentindex^{t,\text{csi}}$ denote the set of users from $\userset_\agentindex$ for which \gls{bs} $\agentindex$ estimates its \gls{csi} at time $t$.
Likewise, $\userset_\agentindex^{t,\text{im}}$ denotes the set of users from $\userset_\agentindex$ appearing in one of the images captured by the \gls{bs} cameras $\setcameras_\agentindex$.
Our assumption is therefore: $\userset_\agentindex^{t,\text{csi}} \cup \userset_\agentindex^{t,\text{im}} = \userset_\agentindex, \forall \agentindex, t$.

Thus, the multimodal observation of \gls{bs} $\agentindex \in \agentset$ at each time slot $t$ is: $\envobservation_\agentindex^t = \sbrk{\rbrk{\image_{\agentindex, \cameraindex}^t}_{\cameraindex \in \setcameras_\agentindex}, \rbrk{\channel^{t}_{\agentindex, \userindex}}_{\userindex \in \userset_\agentindex^{t,\text{csi}}}}$, which corresponds to images taken by the \gls{bs} cameras covering users $\userset_\agentindex^{t,\text{im}}$, as well as \gls{csi} estimated from users $\userset_\agentindex^{t,\text{csi}}$.
We consider that \gls{bs} $\agentindex$ has no knowledge of the matching between the \gls{csi} and the image modalities.
This means that the \gls{bs} has no information on which wireless channel corresponds to which vehicle in the images, and vice-versa.

Lastly, we denote by $\agentposition_\agentindex \in \mathbb{R}^2$ as the position of \gls{rsu} $\agentindex\in \agentset$ in an arbitrary coordinate system, and for each $\userindex \in \userset_\agentindex$, $\userposition_\userindex \in \mathbb{R}^2$ represents the position of vehicle $\userindex$ with respect to its corresponding \gls{rsu} $\agentindex$.
Each \gls{rsu} has \emph{no knowledge} of the positions $\sbrk{\userposition_\userindex}_{\userindex \in \userset_\agentindex}$ of the terminals in its region.

\subsection{Symmetric Environment}
Due to the \gls{v2i} communication scenario under consideration, it is typical to assume that the roads on which the vehicles navigate exhibit inherent symmetries.
In fact, \glspl{rsu} are commonly deployed at crossroads, which are ultimately symmetric road sections.
Consider the scenario shown in Fig.~\ref{fig:mm_paper_system_model} for example depicting a typical \gls{v2i} system, which we study as a use-case in the remainder of this paper.
As such, our work focuses on a class of urban road sections whose geometry and traffic interactions exhibit rotational symmetry.
This assumption captures common structured layouts such as orthogonal intersections and grid-based urban designs.
In such settings, the spatial arrangement of lanes, buildings, and roadside infrastructure is largely preserved under rotations of the scene, leading to repeated sensing patterns and interaction dynamics across different \glspl{rsu}.
This structure of the environment implies that, when the locations of all vehicle users transform by a rotation, i.e., the locations of vehicles corresponding to each \gls{rsu} are rotated versions of the original locations of the vehicles for an adjacent \gls{rsu} (see Fig.~\ref{fig:mm_paper_system_model}), then the downstream optimal resource management action of each \gls{rsu} is a permuted version of the action of the adjacent \gls{rsu}.
Adopting this symmetry assumption allows the system model to capture these recurring structures in a compact and consistent manner, facilitating scalable coordinated decision-making, as we will show hereafter.

While this assumption considers the idealized case of exact rotational symmetry in the environment for clarity and tractability, it primarily serves as a modeling abstraction to make our contribution amenable for networks adhering to such symmetries.
However, we will later demonstrate, through both analytical and empirical results, that the proposed framework remains effective when such symmetry is only approximate or partial, where the environment only partially adheres to the structure.

\subsection{Problem Formulation}
Given the system model elaborated above, we formalize the problem we seek to solve as follows:
\begin{equation}\label{eq:mm_paper_problem}
    \underset{\rbrk{\beam_\agentindex \in \, \beamset_\agentindex}_{\agentindex \in \agentset}}{\text{maximize}} \qquad \lim_{\tau \to \infty} \frac{1}{\tau} \sum_{t=1}^{\tau} \sum_{\userindex \in \userset} \bitrate_\userindex^t.
\end{equation}
Our aim in problem~\eqref{eq:mm_paper_problem} is to find a beam selection policy that maximizes the long term average sum rate of all users in the network.
This problem is challenging for several reasons.
Essentially, the beam selected by each \gls{bs} impacts the bitrates of all users in the system (inside and outside its zone) due to interference, hence, all \glspl{bs} collaborate to guarantee an overall favorable solution.
However, each \gls{bs} only observes users in its region (through \gls{csi} and images), and therefore cannot estimate the impact of its decision on the network.
Added to that, the partial observation of each \gls{bs} does not necessarily include the \gls{csi} of all its users, so it must exploit the camera images to guide its policy.
Moreover, problem~\eqref{eq:mm_paper_problem} must be solved in a distributed manner.
This is due to the multimodal data collected by each \gls{bs} that is large to communicate to a single server that coordinates all the beamforming decisions.
Lastly, the discrete nature of the feasibility set complicates the search for a favorable solution, and motivates the use of data-driven methods.
\section{Reformulation with Distributed MARL Leveraging Latent Symmetries}
\label{section:mm_paper_reformulation}
%
In this section, we reformulate problem~\eqref{eq:mm_paper_problem} through the lens of \gls{marl}.
First, we start by casting the problem as a \emph{distributed \gls{mmdp}}, where \gls{rsu} agents distributed on a graph receive partial observations of a common environment, and communicate over edges of the graph to take actions and cooperate in solving a given task.
Given our system model, we show that our problem corresponds to a distributed \gls{mmdp} with \emph{symmetries}, meaning that the optimal joint action policy of the agents is symmetric, which can be exploited to enhance its training efficiency (see Fig.~\ref{fig:mm_paper_policy_permutation}, detailed hereafter).
Hence, we review the tools needed to study equivalences in decision making problems, and recast our problem as a distributed \gls{mmdp} with symmetries.
However, unlike previous works, we demonstrate that symmetries do not apply in the agent's observation space, rather in a \emph{latent space} which we identify as the sensing space corresponding to the vehicles' locations; thus terming the symmetries as `\emph{latent symmetries}'.

\subsection{Distributed MMDP}
We model the decentralized rate maximization problem from multimodal observations elaborated in the previous section as a distributed \gls{mmdp}, represented by the tuple $\rbrk{\agentset, \statespace, \cbrk{\beamset_\agentindex}_{\agentindex\in\agentset}, \cbrk{\observationspace_\agentindex}_{\agentindex\in\agentset}, \transitionfunction, \reward, \discountfactor}$:
\begin{itemize}[nolistsep, leftmargin=*]
    \item $\agentset$ is the set of \gls{bs} agents,
    \item $\statespace = \statespace_1 \times \dots \times \statespace_{\lvert\agentset\rvert}$ is the state space, and the (unobserved) ground-truth state of the environment at time $t$ is $\state^t=\sbrk{\userposition^t_\userindex}_{\userindex\in\userset}$ is the set of user positions, while for each agent $\state_\agentindex^t=\sbrk{\userposition^t_\userindex}_{\userindex\in\userset_\agentindex}$ denotes the state restricted to its region,
    \item $\beamset_\agentindex$ is the set of actions available to agent $\agentindex$ corresponding its beam codebook, from which it selects beam $\beam_\agentindex^t$ at time $t$, and $\beamset = \beamset_1 \times \dots \times \beamset_{\lvert \agentset \rvert}$,
    \item $\observationspace_\agentindex = \widetilde{\observationspace}_\agentindex \times \widetilde{\edgefeaturesset}_\agentindex$ is the set of observations of agent $\agentindex$, which corresponds to two features $\observation_\agentindex^t = \sbrk{\envobservation_\agentindex^t, \edgefeatures_\agentindex^t}$:
    \begin{itemize}[nolistsep, leftmargin=*, label=\scalebox{0.35}{$\blacksquare$}]
        \item $\envobservation_\agentindex^t = \sbrk{\rbrk{\image_{\agentindex, \cameraindex}^t}_{\cameraindex \in \setcameras_\agentindex}, \rbrk{\channel^{t}_{\agentindex, \userindex}}_{\userindex \in \userset_\agentindex^{t,\text{csi}}}} \in \widetilde{\observationspace}_\agentindex$ is the environmental sensory data corresponding to images recorded by the agent's cameras and \gls{csi} estimated from a subset of vehicles,
        \item The agents communicate through a graph $\commgraph=\rbrk{\agentset, \commgraphedgeset}$, where each agent is a vertex $\agentindex\in\agentset$, and edge $\rbrk{\agentindex, \agentindex^\prime} \in \commgraphedgeset$ indicates that agents $\agentindex$ and $\agentindex^\prime$ can exchange messages%
        \footnote{We assume that nearby \glspl{rsu} can communicate reliably with each other with no transmission error. In fact, we assume that communication is limited to messages whose size is substantially less than the observation of each agent.}%
        . The edge attributes $\edgefeatures_{\agentindex, \agentindex^\prime} \in \widetilde{\edgefeaturesset}_{\rbrk{\agentindex, \agentindex^\prime}}$ correspond to the differences between two \gls{rsu} agents' locations $\edgefeatures_{\agentindex, \agentindex^\prime} = \agentposition_\agentindex - \agentposition_{\agentindex^\prime}$.
        At each time slot, given $\envobservation_\agentindex^t$ and $\edgefeatures_{\agentindex, \agentindex^\prime}$, agent $\agentindex$ can communicate limited messages $\messages_{\agentindex \to \agentindex^\prime}$ to its neighbors.
        The aim of this communication mechanism is to overcome the partial observability of each agent, hence endowing agents with the ability to coordinate while keeping the execution of the policy decentralized~\cite{chafii2023emergent}.
        As we will show later, this communication protocol is crucial to allow agents to detect global state symmetries given only their limited environmental observations.
    \end{itemize}
    \item $\transitionfunction: \statespace \times \beamset_1 \times \dots \times \beamset_{\lvert\agentset\rvert} \times \statespace \to \sbrk{0, 1}$ is the transition function,
    \item $\reward: \statespace \times \beamset_1 \times \dots \times \beamset_{\lvert\agentset\rvert} \to \mathbb{R}$ is the reward function, which we set as the sum rate $\reward_t = \sum_{\userindex\in\userset}\bitrate_\userindex^t$.
    Finally, $\discountfactor \in \sbrk{0, 1}$ is a discount factor.
\end{itemize}
The \gls{mmdp} dynamics proceed as follows.
At timestep $t$, given the state%
\footnote{Note that the vehicles' positions are unknown to the \glspl{bs} which only access unlabeled high-dimensional images and \gls{csi} from users in their area. We refer to the vehicles' positions as the ground-truth state as its realization generates the images and channels observed by the agents.}
of the environment $\state_t$ corresponding to the vehicles' positions, each \gls{bs} agent locally observes multimodal data $\envobservation_\agentindex^t$.
Each agent can exchange limited messages $\messages_{\agentindex \to \agentindex^\prime}$ with its neighbors over the graph edges $\edgefeatures_{\agentindex, \agentindex^\prime}$, and then selects its beam $\beam_\agentindex^t$ from its policy $\policy_\agentindex: \observationspace_\agentindex \times \beamset_\agentindex \to \sbrk{0, 1}$.
All actions $\beam^t = \sbrk{\beam_\agentindex^t}_{\agentindex\in\agentset}$ are executed in the environment, and agents receive reward $\reward_t$.
The environment then transitions to a new state $\state_{t+1}$ according to $\transitionfunction$, the agents collect new observations and so on.
The aim of the \gls{mmdp} is to find a joint action policy $\rbrk{\policy_1, \dots, \policy_{\lvert\agentset\rvert}}$ maximizing the expected discounted return $\mathbb{E}\bigl[\sum_{i=1}^\infty \gamma^i \reward_{t+i}\bigr]$, which is equivalent to solving~\eqref{eq:mm_paper_problem}.

\begin{oonecolumn}
\begin{figure*}[!t]
\centering
\includegraphics[width=.5\textwidth]{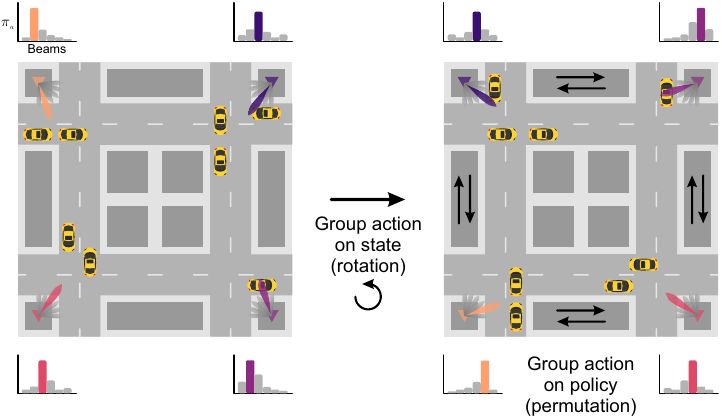}
\caption{A global symmetry: when the vehicles' positions rotate, the optimal policy is permuted between and within agents.}
\label{fig:mm_paper_policy_permutation}
\end{figure*}
\end{oonecolumn}

Approximating the optimal joint policy can be done using off-the-shelf \gls{marl} algorithms.
However, such methods are highly inefficient in terms of the training data needed for convergence.
Added to that, the observation of agents is captured through high-dimensional images and \gls{csi}, and do not necessarily capture both modalities for all users.
To overcome those drawbacks, we exploit the symmetric nature of the considered environment.
Looking at Fig.~\ref{fig:mm_paper_policy_permutation}, showing a top-down view of Fig.~\ref{fig:mm_paper_system_model}, notice that whenever the full state of the system, i.e. positions of the vehicles, rotates by $90$\textdegree, the optimal joint policy is permuted between (and within) agents.
Such instances are called transformation equivalent global state-action pairs.
Having such information as a prior, we enforce those symmetry constraints on our policy, and greatly reduce the search space for an optimal policy.

We proceed by grounding our intuitive observation in rigorous definitions of symmetry under groups and transformations~\cite{pinter2010book},~\cite{barto2001},~\cite{worrall2017harmonic}.

\begin{ttwocolumn}
\begin{figure}[!t]
\centering
\includegraphics[width=.45\textwidth]{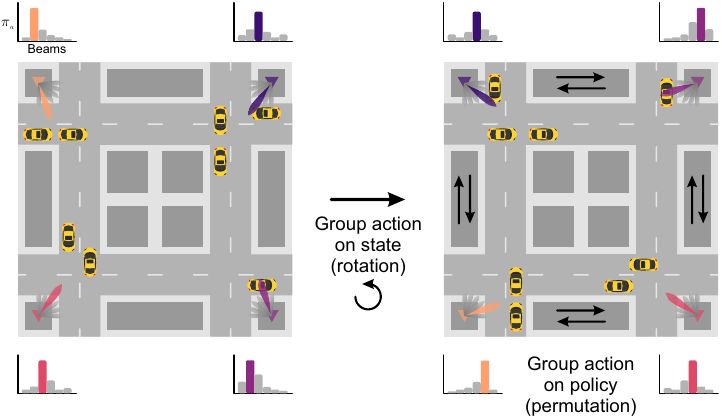}
\caption{A global symmetry: when the vehicles' positions rotate, the optimal policy is permuted between and within agents.}
\label{fig:mm_paper_policy_permutation}
\end{figure}
\end{ttwocolumn}

\subsection{Groups and Transformations}
A \emph{group} is a pair $\rbrk{\group, \binaryop}$, where $\group$ is a set and $\binaryop$ is a binary operation on $\group$ satisfying the conditions: associativity, identity, closure and inverse~\cite{pinter2010book}.
In this paper, we will discuss the symmetries of the rotation group.
For instance, consider the set $\cyclicgroup_n = \cbrk{\rotationmatrix\rbrk{\theta} \mid \theta \in \cbrk{\frac{2 \pi k}{n}, k=0,\dots,n-1}}$ of $\frac{2\pi}{n}$ rotations, equipped the matrix composition operation.
In particular, the case $n=4$ corresponds to the set of $90$\textdegree\ rotations $\cbrk{0\text{\textdegree}, 90\text{\textdegree}, 180\text{\textdegree}, 270\text{\textdegree}}$ where:
\begin{oonecolumn}
\begin{align}
    \label{eq:mm_paper_rotation_matrix} 
    \rotationmatrix\rbrk{\theta} = \biggl[ \;
    \begin{matrix*}[c]
        \cos\theta & -\sin\theta \\[-10pt] \sin\theta & \cos\theta
    \end{matrix*}
    \biggr],
\end{align}
\end{oonecolumn}
\begin{ttwocolumn}
\begin{align}
    \label{eq:mm_paper_rotation_matrix} 
    \rotationmatrix\rbrk{\theta} = \biggl[ \;
    \begin{matrix*}[c]
        \cos\theta & -\sin\theta \\ \sin\theta & \cos\theta
    \end{matrix*}
    \biggr],
\end{align}
\end{ttwocolumn}%
is a rotation matrix representing the act of rotating within Euclidean space.
Matrix multiplication is an associative operation, and $\rotationmatrix\rbrk{0}$ is the identity element, while composing two elements of the set yields another element and each member has an inverse that is also an element of the set.
Hence, the set $\cyclicgroup_4$ is a group under composition.
In fact, $\cyclicgroup_n$ is the cyclic subgroup of the group of continuous rotations $\mathrm{SO}\rbrk{2}=\cbrk{\rotationmatrix \rbrk{\theta} \mid 0 \leq \theta < 2\pi}$.
Notice that each element of the group represents a specific \emph{transformation}, for instance rotation, which is formalized using group actions.

The action of a group $\group$ on a set $X$ is a mapping $\groupaction\!\!: \group \times X \to X$ that satisfies: $\groupaction_g \sbrk{\groupaction_h \sbrk{x}} = \groupaction_{g \binaryop h}\sbrk{x} \;\forall g, h \in \group, x \in X$ and $\groupaction_\groupidentityelement \sbrk{x} = x$ where $\groupidentityelement$ is the identity element.
For instance, the group of $90$\textdegree\ rotations acts on plane vectors by rotating them.
For a given $g \in \group$, the mapping $\groupaction_g\!: X \to X$ is called a \emph{transformation} operator.

Given a mapping $f\!\!: X \to Y$ and a transformation operator $\groupaction_g\!: X \to X$, we say that $f$ is \emph{equivariant} with respect to $\groupaction_g$ if there exists another transformation $\groupactiontwo_g\!: Y \to Y$ (in the output space of $f$), such that $\forall g \in \group, x \in X$:
\begin{equation}\label{eq:mm_paper_equivariance} 
    \groupactiontwo_g\sbrk{f\rbrk{x}} = f\rbrk{\groupaction_g\sbrk{x}}.
\end{equation}
For example, in our use-case, the optimal policy is equivariant to global state rotations, i.e., whenever the state $\state$ undergoes a rotation via $\groupaction_g$, the optimal policy permutes as $\groupactiontwo_g\sbrk{\policy^\star\rbrk{\state}} = \policy^\star\rbrk{\groupaction_g\sbrk{\state}}$.
Here, since we consider discrete actions, $\groupactiontwo_g$ represents the multiplication of the policy with a permutation matrix (depending on $g$, i.e. a rotation).

Furthermore, a particular case of equivariance is \emph{invariance}.
If for all $g \in \group$, $\groupactiontwo_g$ is the identity map of $Y$, i.e., $\forall g \in \group, x \in X, \; f\rbrk{x} = f\rbrk{\groupaction_g\sbrk{x}}$, then we say that $f$ is \emph{invariant} or \emph{symmetric} to $\groupaction_g$.
For instance, in our \gls{mmdp}, the optimal value function is invariant to rotation transformations: $\valuefunction^\star\rbrk{\state} = \valuefunction^\star\rbrk{\groupaction_g\sbrk{\state}}$, in other words, two globally rotated states lead to the same optimal value.

\subsection{Distributed MARL Using Latent Symmetries}
We are interested in utilizing the above symmetry notions to recast our distributed \gls{mmdp} by exploiting the symmetries occurring in the environment to simplify the search for a favorable action policy.
Nevertheless, we must allow for our policy to be executed in a decentralized manner, since we cannot afford sending large multimodal data from all agents to a central server that determines the agents' actions.
Therefore, while each agent relies on its own observation, all agents must coordinate to detect \emph{global} state transformations.

Notice that with only access to its local information, each agent cannot identify symmetries of the global state.
For example, designing the agents' local policies to be equivariant to local state transformations -- when the positions of the vehicles in their respective zones rotate, their local policies transform by a permutation -- does not yield correct global transformation, like the one shown in Fig.~\ref{fig:mm_paper_policycomm_permutation}.

To identify global state transformations with only local information, we remark that when such transformations occur, the local states are transformed and their positions are permuted.
Hence, a global symmetry affects both the agents' local states $\rbrk{\state_\agentindex}$, as well as the features on the communication graph $\rbrk{\edgefeatures_{\agentindex, \agentindex^\prime}}$, which relate the agents' locations.
Consider as a running example the top right \gls{rsu} in Fig.~\ref{fig:mm_paper_policycomm_permutation}, in a given state of the environment.
The agent selects an action from its policy given its local state and communication from its neighbors.
Whenever the global environment transitions to a transformed state, we wish to constrain the policy of the top left agent such that it selects a transformed version of the previous action (previously taken by the top right agent), given its locally transformed state and communication.
In other words, from each agent's local perspective, if the vehicles' positions in its region are rotated, and the messages received from its neighbors are transformed similarly, then the agent must execute an equivalently transformed action from the same policy.
As such, by virtue of the communication graph, we can design our policy networks to be globally equivariant with distributed execution depending on agents' local states.

\begin{oonecolumn}
\begin{figure*}[!t]
\centering
\subfloat[Original state.]{\includegraphics[width=.25\textwidth]{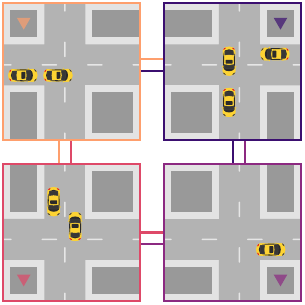}}
\hspace{5em}
\subfloat[Transformed state.]{\includegraphics[width=.25\textwidth]{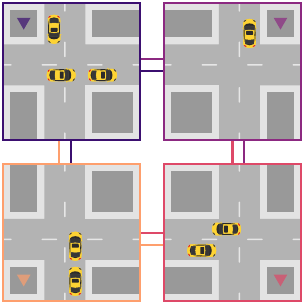}}
\caption{Two equivalent states: a global state transformation is equivalent to a rotation of local states, followed by a permutation of local states and communication graph edges.}
\label{fig:mm_paper_policycomm_permutation}
\end{figure*}
\end{oonecolumn}
\begin{ttwocolumn}
\begin{figure}[!t]
\centering
\subfloat[Original state.]{\includegraphics[width=.475\linewidth]{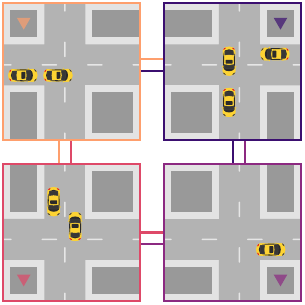}}
\hfill
\subfloat[Transformed state.]{\includegraphics[width=.475\linewidth]{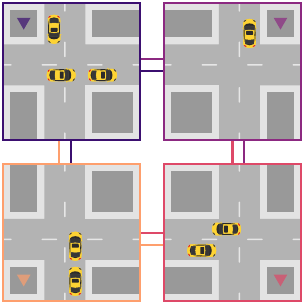}}
\caption{Two equivalent states: a global state transformation is equivalent to a rotation of local states, followed by a permutation of local states and communication graph edges.}
\label{fig:mm_paper_policycomm_permutation}
\end{figure}
\end{ttwocolumn}

With the above reasoning, our distributed \gls{mmdp} belongs to the class of \emph{distributed \glspl{mmdp} with symmetries}~\cite{pol2022multiagent}.
Formally, a distributed \gls{mmdp} with symmetries is a distributed \gls{mmdp} for which the following equations hold for at least one non-trivial set of group transformations $\groupaction_g\!: \statespace \to \statespace$, and for every state $\state$, $\groupactiontwo_g^\state\!: \beamset \to \beamset$, such that:
\begin{oonecolumn}
\begin{align}
    \label{eq:mm_paper_mmdp_symmetric_reward}
    \reward\rbrk{\state, \beam} &= \reward\rbrk{\groupaction_g \sbrk{\state}, \groupactiontwo_g^ \state\sbrk{\beam}} &\qquad& \forall g \in \group, \state \in \statespace, \beam \in \beamset, \\
    \label{eq:mm_paper_mmdp_symmetric_transition}
    \transitionfunction\rbrk{\state, \beam, \state^\prime} &= \transitionfunction\rbrk{\groupaction_g \sbrk{\state}, \groupactiontwo_g^ \state\sbrk{\beam}, \groupaction_g \sbrk{\state^\prime}} &\qquad& \forall g \in \group, \state, \state^\prime \in \statespace, \beam \in \beamset,
\end{align}
\end{oonecolumn}%
\begin{ttwocolumn}
\begin{align}
    \label{eq:mm_paper_mmdp_symmetric_reward}
    \reward\rbrk{\state, \beam} &= \reward\rbrk{\groupaction_g \sbrk{\state}, \groupactiontwo_g^ \state\sbrk{\beam}}, \\
    \label{eq:mm_paper_mmdp_symmetric_transition}
    \transitionfunction\rbrk{\state, \beam, \state^\prime} &= \transitionfunction\rbrk{\groupaction_g \sbrk{\state}, \groupactiontwo_g^ \state\sbrk{\beam}, \groupaction_g \sbrk{\state^\prime}}, \\
    \nonumber &\hspace{-2em}\forall g \in \group, \state, \state^\prime \in \statespace, \beam \in \beamset,
\end{align}
\end{ttwocolumn}%
where equivalently to acting on $\state$ with $\groupaction_g$, we can act on the agents' local states and edge features separately with $\tilde{\groupaction}_g\!: \statespace_\agentindex \to \statespace_\agentindex$ and $\tilde{\groupactionedge}_g\!: \edgefeaturesset \to \edgefeaturesset$, to end up in the same global state:
\begin{equation}\label{eq:mm_paper_mmdp_groupaction_decomposition}
    \groupaction_g \sbrk{\state} = \Bigl(\permutation_g^\state \bigl[\{ \tilde{\groupaction}_g [\state_\agentindex]\}_{\agentindex \in \agentset} \bigr], \permutation_g^\edgefeatures \bigr[\{\tilde{\groupactionedge}_g [\edgefeatures_{\agentindex, \agentindex^\prime}]\}_{\rbrk{\agentindex, \agentindex^\prime} \in \commgraphedgeset} \bigr] \Bigr)
\end{equation}
where $\permutation_g^\state$ and $\permutation_g^\edgefeatures$ are state and edge permutations respectively.

Namely, a distributed \gls{mmdp} is symmetric if there exists state and action transformations which leave the reward and transition functions invariant (eqs. \eqref{eq:mm_paper_mmdp_symmetric_reward}, \eqref{eq:mm_paper_mmdp_symmetric_transition}).
As argued above, the group action on the global state can be decomposed into separate group actions on the local agent states and communication edges (eq. \eqref{eq:mm_paper_mmdp_groupaction_decomposition}).
Two state-action pairs $\rbrk{\state, \beam}$ and $\rbrk{\groupaction_g \sbrk{\state}, \groupactiontwo_g^ \state\sbrk{\beam}}$ satisfying eqs. \eqref{eq:mm_paper_mmdp_symmetric_reward} and \eqref{eq:mm_paper_mmdp_symmetric_transition} are called equivalent.
The importance of identifying a symmetric \gls{mmdp} lies in the fact that they admit symmetric optimal policies, i.e., whenever the state transforms, the policy must transform accordingly (see Fig.~\ref{fig:mm_paper_policy_permutation}).
By imbuing those inductive biases into our policy training methods, we can significantly simplify the search for an optimal policy.

Although our distributed \gls{mmdp} is symmetric, it is crucial to notice that the symmetries occur at the level of the state $\state=\sbrk{\userposition_\userindex}_{\userindex\in\userset}$, containing the ground-truth vehicle positions, which is not observable to the \gls{rsu} agents.
Instead, our agents only receive high-dimensional and incomplete observations $\rbrk{\envobservation_\agentindex}$ comprising images and \gls{csi}, which do not obey symmetric properties; meaning that when the vehicles' positions undergo a transformation, these observations do not transform by a tractable transformation~\cite{park22a}.
Hence we term our \gls{mmdp} as a distributed \gls{mmdp} with \emph{latent symmetries}, where agents must extract \emph{low-dimensional symmetric} features (vehicle positions $\userposition_\userindex$ in our case) from high-dimensional non-symmetric data $\observation_\agentindex$ (multimodal images and channels), to solve a downstream task in an efficient manner.

\begin{remark}
    Our work significantly differs from the previous literature on exploiting symmetries in \gls{marl}~\cite{yu2023esp},~\cite{pol2022multiagent},~\cite{zhou2024symmetry},~\cite{shi2025symmetry}.
    In previous works, it is assumed that the observable state of the environment is symmetric; for example~\cite{zhou2024symmetry} and~\cite{shi2025symmetry} assume that the positions of the users in the network are perfectly known by the \gls{bs} agents.
    In this work, while the symmetric nature of the environment is assumed as a prior (in terms of user positions), the agents do not directly access this state of the environment, and only observe unlabeled high-dimensional features (images and wireless channels).
\end{remark}

\begin{remark}
    Since our work is a first attempt to study the impact of symmetries in multimodal wireless networks, we limit our attention and use-case on the group $\group = \cyclicgroup_4$ of $90$\textdegree\ rotations.
    However, the framework we develop is versatile and general for other symmetry groups (like the Euclidean group $\mathrm{E}\rbrk{2}$ of translations, rotations and reflections), which we leave for future works.
\end{remark}

To proceed with solving our problem, we propose a scheme of two complementary steps, as shown in Fig.~\ref{fig:mm_paper_system_model}.
First in Section~\ref{section:mm_paper_ssl}, we devise a self-supervised learning framework for multimodal sensing where each \gls{rsu} agent estimates the positions of the terminals in its zone from high-dimensional observations.
Given this sensing data, in Section~\ref{section:mm_paper_marl} we design a distributed globally equivariant \gls{marl} policy network to solve our downstream rate maximization problem efficiently.

\section{Self-supervised Multimodal Sensing and Alignment}
\label{section:mm_paper_ssl}

\begin{oonecolumn}
\begin{figure}
    \centering
    \includegraphics[width=.6\linewidth]{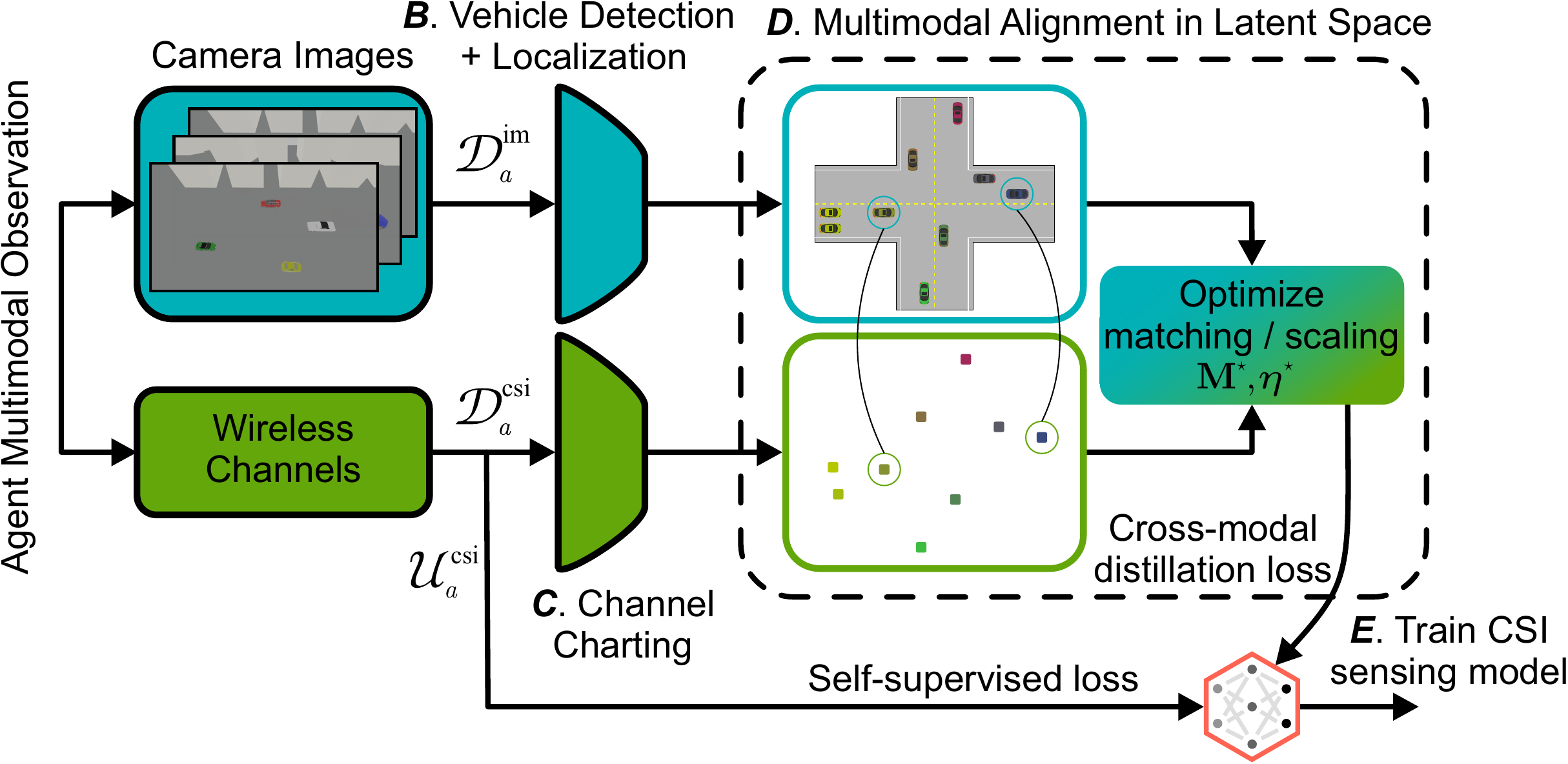}
    \caption{Proposed self-supervised multimodal sensing framework.}
    \label{fig:mm_paper_multimodal_framework}
\end{figure}    
\end{oonecolumn}

\begin{ttwocolumn}
\begin{figure}
    \centering
    \includegraphics[width=\linewidth]{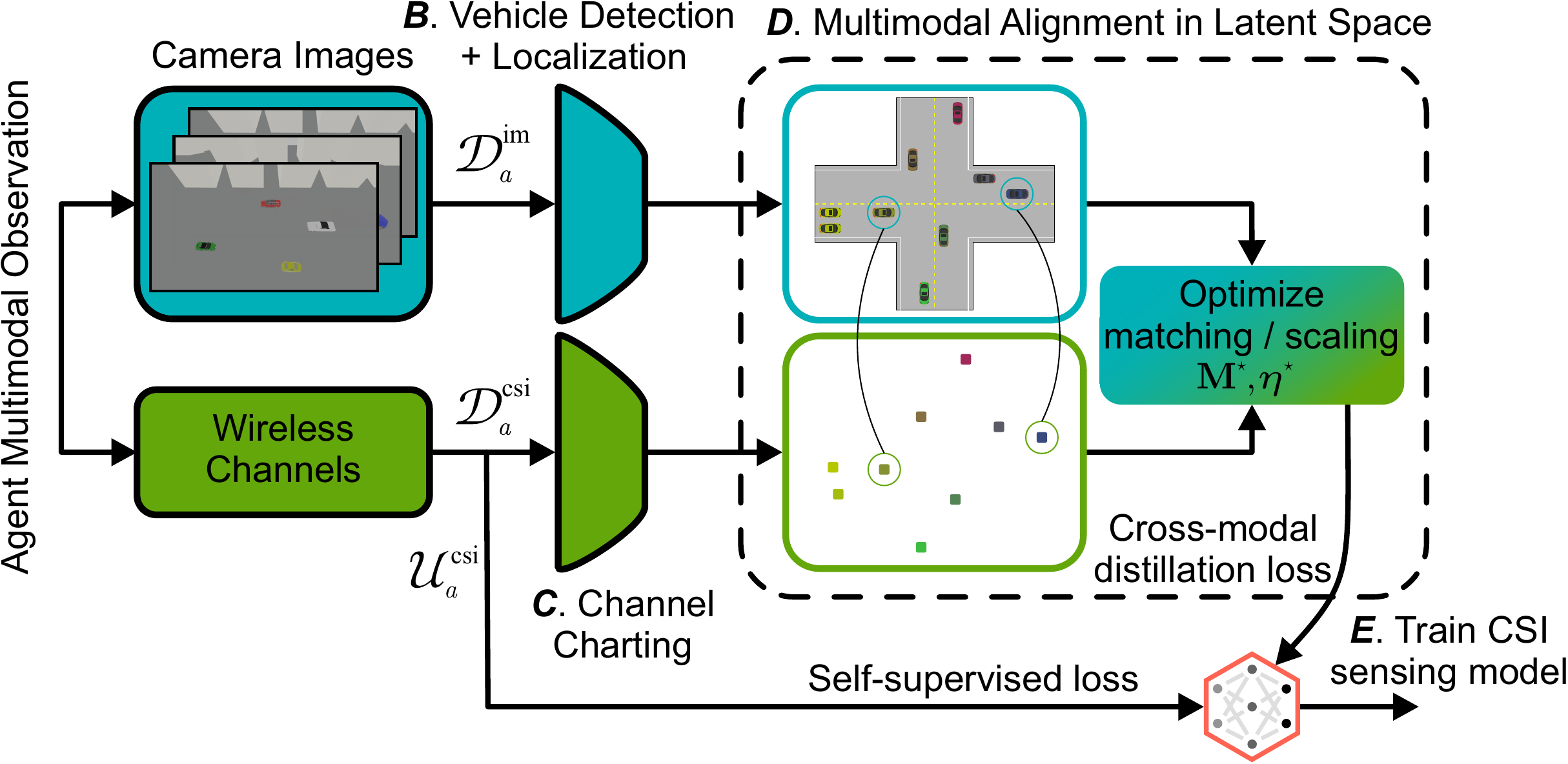}
    \caption{Proposed self-supervised multimodal sensing framework.}
    \label{fig:mm_paper_multimodal_framework}
\end{figure}    
\end{ttwocolumn}

\subsection{Main Idea and Solution Scheme}
In this section, we develop a novel framework for \emph{multimodal sensing}, shown in Fig.~\ref{fig:mm_paper_multimodal_framework}, allowing each \gls{rsu} agent $\agentindex \in \agentset$ to estimate the locations of the vehicles in its zone $\sbrk{\estimateduserposition^t_\userindex}_{\userindex \in \userset_\agentindex}$ using unlabeled and incomplete environmental observations $\envobservation_\agentindex^t$.
We recall that the multimodal data available to agent $\agentindex$ corresponds to a collection of images $\rbrk{\image_{\agentindex, \cameraindex}^t}_{\cameraindex \in \setcameras_\agentindex}$ from the \gls{bs} cameras showing a subset of users $\userset_\agentindex^{t,\text{im}}$, and a collection of \gls{csi} $\rbrk{\channel^{t}_{\agentindex, \userindex}}$ from a subset of users $\userset_\agentindex^{t,\text{csi}}$.
The data is \emph{unlabeled}, meaning that the agent has no information on the matching between the modalities, i.e., which estimated wireless channel corresponds to which user in the image.
Moreover, during each time slot $t$, the multimodal observation of an agent might be \emph{incomplete}, in the sense that not all vehicles appear in both modalities.

In order to estimate the positions $\sbrk{\estimateduserposition^t_\userindex}_{\userindex \in \userset_\agentindex}$ in such a setting, we propose an offline
training method%
\footnote{Although this routine is offline, each \gls{bs} agent runs our proposed sensing technique locally, using multimodal data from collected from its region only, and no communication occurs between the different agents (or with a central server) at this stage.}
where each agent forms two \emph{local} data sets: a data set of images $\dataset_\agentindex^{\text{im}} = \cbrk{\image_{\agentindex, \cameraindex}^t \mid \cameraindex \in \setcameras_\agentindex, t \leq \offlinetrainingperiod}$ and a channel data set $\dataset_\agentindex^{\text{csi}} = \cbrk{\channel^{t}_{\agentindex, \userindex} \mid \userindex \in \cup_{t \leq \offlinetrainingperiod} \userset_\agentindex^{t,\text{csi}}}$, where $\offlinetrainingperiod$ is the number of data collection slots.
Since the matching between the modalities is unknown, we seek to align the \emph{low dimensional features} of $\dataset_\agentindex^{\text{im}}$ and $\dataset_\agentindex^{\text{csi}}$ in a \emph{self-supervised} fashion.
Precisely, we use the image data to perform direct localization of (some of) the vehicles, while we embed the \gls{csi} data into representations conserving their relative location features.
Since the two representation sets share the same structure, we align their \emph{latent} features in a common space, corresponding to the vehicles' positions, by optimizing a \emph{latent} matching function over their intra-sample distances.
Finally, we train a parametrized function to perform sensing based on this aligned data, so that it can be used for inference or imputation during online deployment.
As such, each agent learns to perform multimodal sensing and alignment without any external supervision nor reward (such data labeling and annotation), and \emph{our algorithm is therefore a self-supervised learning method}.
As our technique is run locally for each \gls{bs} agent, we formalize our overall idea in the following subsections while fixing the agent index $\agentindex \in \agentset$.

\begin{oonecolumn}
\begin{figure*}
    \centering
    \subfloat[A sample camera image.]{
        \includegraphics[width=0.32\linewidth]{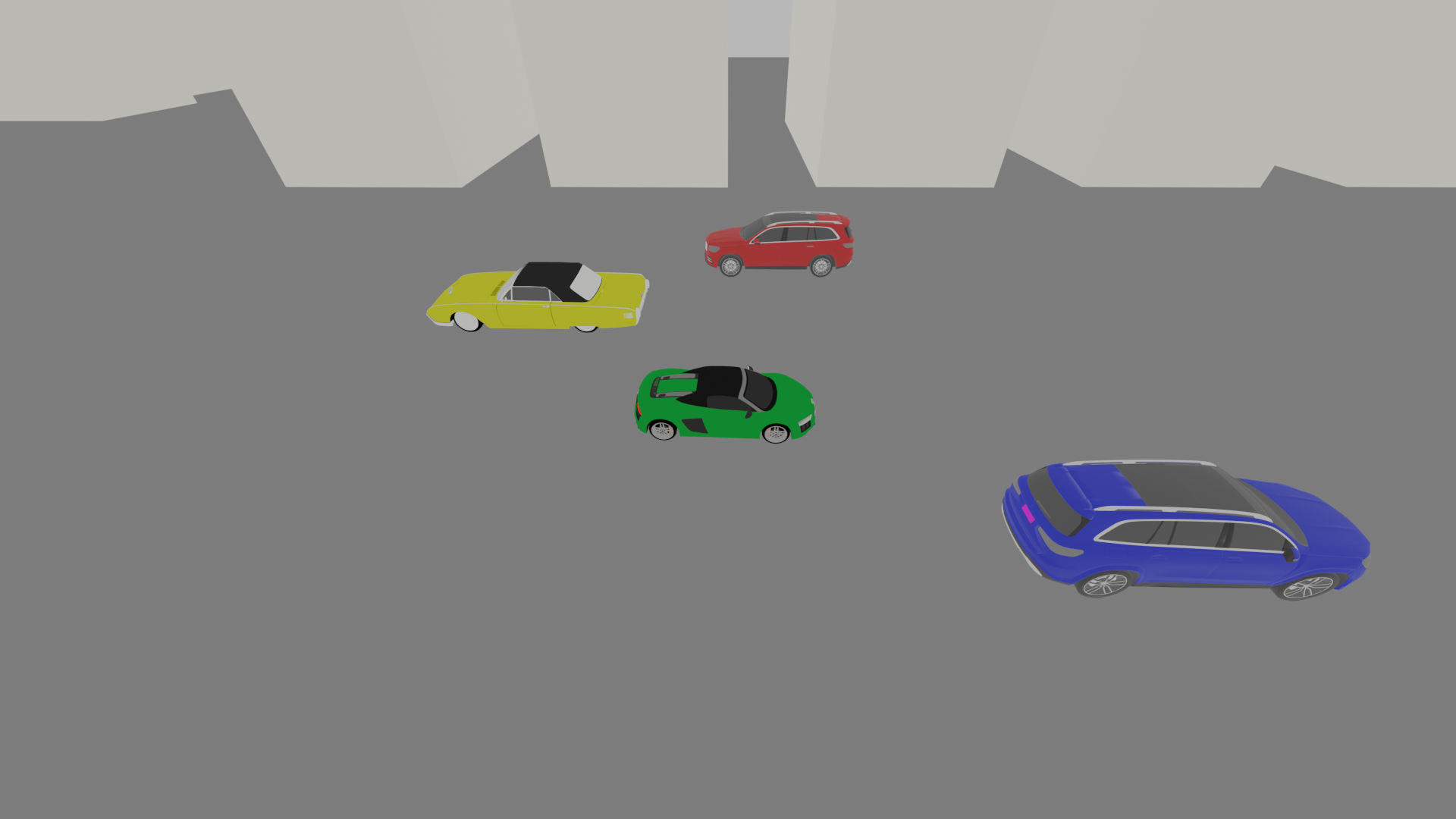}
        \label{fig:mm_paper_sample_cam_image}}
    \hfill
    \subfloat[Vehicle detection.]{
        \includegraphics[width=0.32\linewidth]{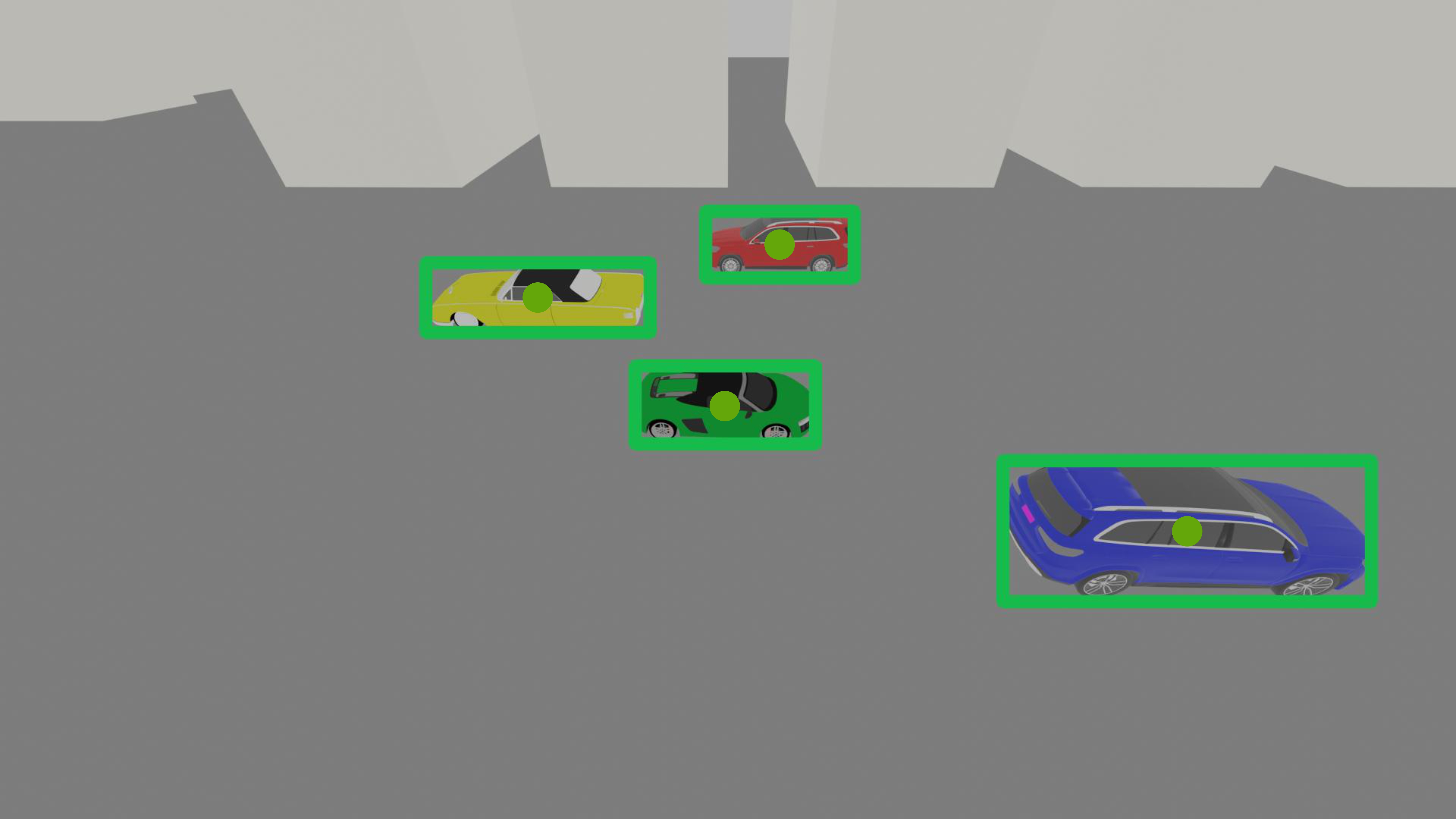}
        \label{fig:mm_paper_vehicle_detection}}
    \hfill
    \subfloat[Vehicle localization.]{
        \includegraphics[width=0.3\linewidth]{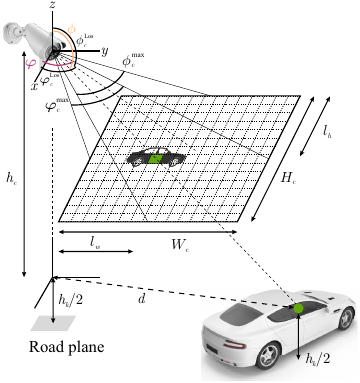}
        \label{fig:mm_paper_image_localization}}
    \caption{Proposed image-based localization.}
    \label{fig:mm_paper_image_processing}
\end{figure*}
\end{oonecolumn}
\begin{ttwocolumn}
\begin{figure*}
    \centering
    \subfloat[A sample camera image.]{
        \includegraphics[width=0.25\linewidth]{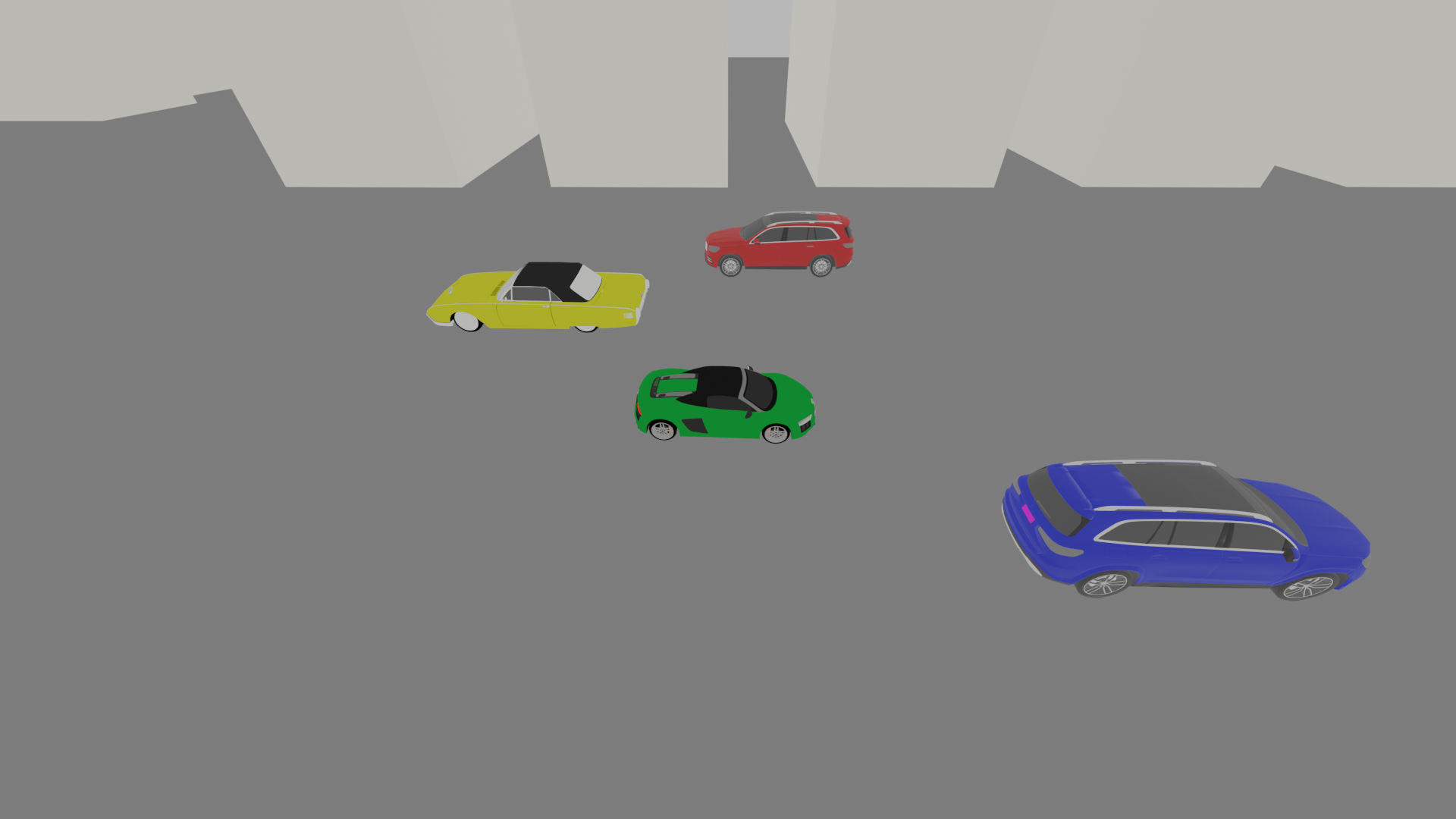}
        \label{fig:mm_paper_sample_cam_image}}
    \hfill
    \subfloat[Vehicle detection.]{
        \includegraphics[width=0.25\linewidth]{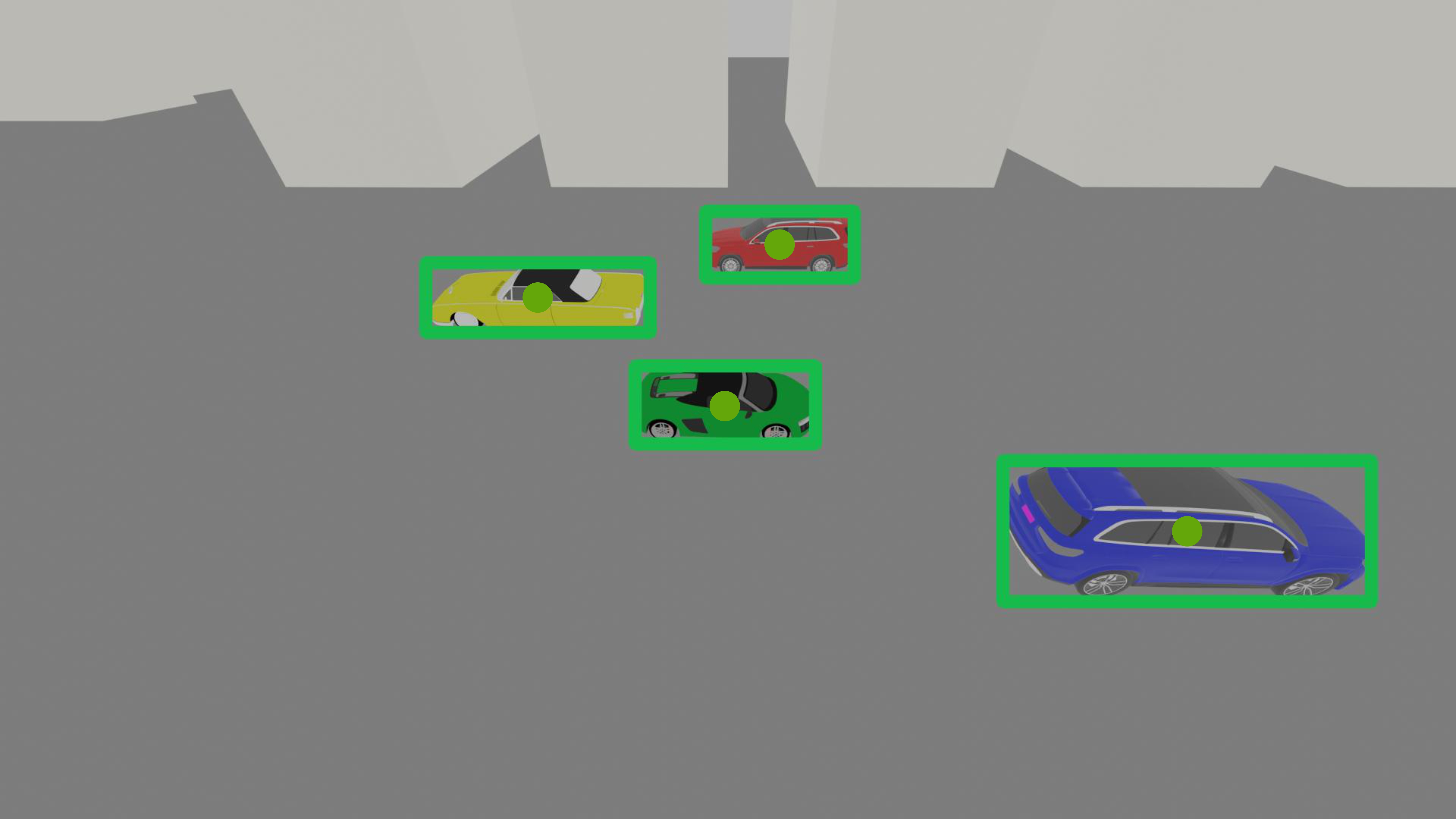}
        \label{fig:mm_paper_vehicle_detection}}
    \hfill
    \subfloat[Vehicle localization.]{{\setlength{\fboxrule}{0.01pt}\fcolorbox{purple}{white}{
        \includegraphics[width=0.25\linewidth]{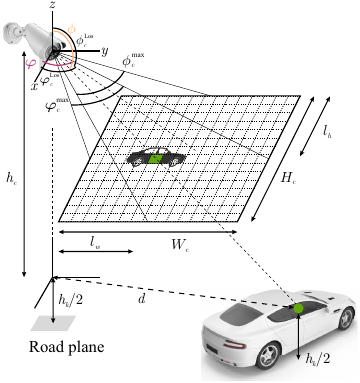}
        \label{fig:mm_paper_image_localization}}}}
    \caption{Proposed image-based localization.}
    \label{fig:mm_paper_image_processing}
\end{figure*}
\end{ttwocolumn}

\subsection{Image Processing}\label{subsection:mm_paper_image_processing}
We seek to extract the locations of the vehicles from the images in $\dataset_\agentindex^{\text{im}}$, which would amounts to forming another dataset of sensing features $\lowdimdataset_\agentindex^{\text{im}} = \cbrk{\estimateduserposition_\userindex^t \mid \userindex \in \cup_{t \leq \offlinetrainingperiod} \userset_\agentindex^{t,\text{im}}}$, where $\estimateduserposition_\userindex^t$ is an estimate of the location $\userposition_\userindex^t$ of terminal $\userindex$.
We start by invoking the following lemma.

\begin{oonecolumn}
\begin{lemma}\label{lemma:mm_paper_localization}
    The physical location $\userposition=\sbrk{\userpositionxy_x, \userpositionxy_y}$ relatively to camera $\cameraindex$ of pixel coordinates $\sbrk{\pixelcoordinate_w, \pixelcoordinate_h}$ can be estimated as:
    \begin{align}
        \label{eq:mm_paper_pixel_localizationx}
        \userpositionxy_x = \cameraheight_\cameraindex \tan \rbrk{\elevationangle^{\textnormal{Los}}_\cameraindex + \tan^{-1} \rbrk{\frac{2 \pixelcoordinate_h - \imageheight_\cameraindex}{\imageheight_\cameraindex} \tan \frac{\elevationangle^{\textnormal{max}}_\cameraindex}{2}}} \cos \rbrk{\azimuthangle^{\textnormal{Los}}_\cameraindex + \tan^{-1} \rbrk{\frac{2 \pixelcoordinate_w - \imagewidth_\cameraindex}{\imagewidth_\cameraindex} \tan \frac{\azimuthangle^{\textnormal{max}}_\cameraindex}{2}}}, \\
        \label{eq:mm_paper_pixel_localizationy}
        \userpositionxy_y = \cameraheight_\cameraindex \tan \rbrk{\elevationangle^{\textnormal{Los}}_\cameraindex + \tan^{-1} \rbrk{\frac{2 \pixelcoordinate_h - \imageheight_\cameraindex}{\imageheight_\cameraindex} \tan \frac{\elevationangle^{\textnormal{max}}_\cameraindex}{2}}} \sin \rbrk{\azimuthangle^{\textnormal{Los}}_\cameraindex + \tan^{-1} \rbrk{\frac{2 \pixelcoordinate_w - \imagewidth_\cameraindex}{\imagewidth_\cameraindex} \tan \frac{\azimuthangle^{\textnormal{max}}_\cameraindex}{2}}},
    \end{align}
    where $\rbrk{\cameraheight_\cameraindex, \azimuthangle^{\textnormal{Los}}_\cameraindex, \elevationangle^{\textnormal{Los}}_\cameraindex, \azimuthangle^{\textnormal{max}}_\cameraindex, \elevationangle^{\textnormal{max}}_\cameraindex, \imageheight_\cameraindex, \imagewidth_\cameraindex}$ are camera parameters defined previously.
\end{lemma}
\end{oonecolumn}
\begin{ttwocolumn}
\begin{lemma}\label{lemma:mm_paper_localization}
    The physical location $\userposition=\sbrk{\userpositionxy_x, \userpositionxy_y}$ relatively to camera $\cameraindex$ of pixel coordinates $\sbrk{\pixelcoordinate_w, \pixelcoordinate_h}$ can be estimated as:
    \begin{align}
        \begin{split}\label{eq:mm_paper_pixel_localizationx}
            \userpositionxy_x = &\cameraheight_\cameraindex \tan \rbrk{\elevationangle^{\text{Los}}_\cameraindex + \tan^{-1} \rbrk{\frac{2 \pixelcoordinate_h - \imageheight_\cameraindex}{\imageheight_\cameraindex} \tan \frac{\elevationangle^{\text{max}}_\cameraindex}{2}}} \\ &\cos \rbrk{\azimuthangle^{\text{Los}}_\cameraindex + \tan^{-1} \rbrk{\frac{2 \pixelcoordinate_w - \imagewidth_\cameraindex}{\imagewidth_\cameraindex} \tan \frac{\azimuthangle^{\text{max}}_\cameraindex}{2}}},
        \end{split}
        \\
        \begin{split}\label{eq:mm_paper_pixel_localizationy}
            \userpositionxy_y = &\cameraheight_\cameraindex \tan \rbrk{\elevationangle^{\text{Los}}_\cameraindex + \tan^{-1} \rbrk{\frac{2 \pixelcoordinate_h - \imageheight_\cameraindex}{\imageheight_\cameraindex} \tan \frac{\elevationangle^{\text{max}}_\cameraindex}{2}}} \\ &\sin \rbrk{\azimuthangle^{\text{Los}}_\cameraindex + \tan^{-1} \rbrk{\frac{2 \pixelcoordinate_w - \imagewidth_\cameraindex}{\imagewidth_\cameraindex} \tan \frac{\azimuthangle^{\text{max}}_\cameraindex}{2}}},
        \end{split}
    \end{align}
    where $\rbrk{\cameraheight_\cameraindex, \azimuthangle^{\textnormal{Los}}_\cameraindex, \elevationangle^{\textnormal{Los}}_\cameraindex, \azimuthangle^{\textnormal{max}}_\cameraindex, \elevationangle^{\textnormal{max}}_\cameraindex, \imageheight_\cameraindex, \imagewidth_\cameraindex}$ are camera parameters defined previously.
\end{lemma}
\end{ttwocolumn}%
\begin{proof}
    The proof is deferred to Appendix~\ref{proof:mm_paper_localization}.
\end{proof}

Lemma~\ref{lemma:mm_paper_localization} allows us to perform direct sensing given the pixel coordinate of a vehicle in an image $\image_{\agentindex, \cameraindex}^t$.
To detect the vehicles' coordinates in $\dataset_\agentindex^{\text{im}}$, we employ off-the-shelf \gls{ml}-based models, known for their superlative results in image segmentation tasks.
Particularly, we utilize the seventh version of the celebrated model You Only Look Once (YOLOv7)~\cite{wang2023yolov7} due to its fast and precise predictions, and lightweight hardware integration for industrial implementations.
We note that the \gls{rsu} does not train the model from scratch, but readily utilizes a pre-trained YOLOv7 model that outputs the bounding boxes of the vehicles from the camera images.

It is worth mentioning that in practice, the pixel we use to localize the vehicles is the center of the bounding box detected by YOLOv7, as depicted in Fig~\ref{fig:mm_paper_vehicle_detection} (enlarged green dots).
Accordingly, the physical height corresponding to this pixel is approximately half the height of the vehicle itself ($\frac{h_k}{2}$ where $h_k$ is the height of vehicle $k$ with respect to the road plane), as shown in Fig~\ref{fig:mm_paper_image_localization}.
Thus, when applying Lemma~\ref{lemma:mm_paper_localization} to localize the vehicles from images, the first factor $\cameraheight_\cameraindex$ in eqs.~\eqref{eq:mm_paper_pixel_localizationx} and~\eqref{eq:mm_paper_pixel_localizationy} corresponding to the camera's height must be adjusted to $\cameraheight_\cameraindex-\frac{h_k}{2}$.
However, since the vehicles' heights $h_k$ are unknown, we account for this parallax error by subtracting the average middle height of a vehicle from $\cameraheight_\cameraindex$, as shown in Fig.~\ref{fig:mm_paper_image_localization}, where we approximate $\frac{h_k}{2}\approx 0.8$m, considering vehicles' heights are highly concentrated around $1.6$m.
Accordingly, the error ensuing from our image localization is due to the average deviation of a vehicle's height from the average we assumed, and is therefore relatively negligible.

As shown in Fig.~\ref{fig:mm_paper_image_localization}, the agent detects the vehicles in its cameras' images and extracts their (low-dimensional) sensing data, collected in $\lowdimdataset_\agentindex^{\text{im}}$.
We let $\numsamples_\agentindex^{\text{im}} = \left\lvert \cup_{t \leq \offlinetrainingperiod} \userset_\agentindex^{t,\text{im}} \right\rvert = \left\lvert \lowdimdataset_\agentindex^{\text{im}} \right\rvert$ denote the number of vehicle location samples estimated by the agent from camera images, and $\matriximagesensing_\agentindex^\text{im} \in \mathbb{R}^{2 \times \numsamples_\agentindex^{\text{im}}}$ be a matrix whose columns are the estimated positions $\estimateduserposition_\userindex \in \lowdimdataset_\agentindex^{\text{im}}$.
The agent computes an intra-set distance matrix $\distancematrix_\agentindex^{\text{im}} \in \mathbb{R}^{\numsamples_\agentindex^{\text{im}} \times \numsamples_\agentindex^{\text{im}}}$ whose elements $\sbrk{\distancematrix_\agentindex^{\text{im}}}_{i,j} = \left\lvert \estimateduserposition_i - \estimateduserposition_j \right\rvert$ are the Euclidean distances between pairs of samples.

\subsection{CSI Processing}\label{subsection:mm_paper_csi_processing}
Similarly to image data, we seek to extract low-dimensional features of the agent's unlabeled \gls{csi} data set $\dataset_\agentindex^{\text{csi}}$, so as to align those features with those extracted from the images.
We let $\numsamples_\agentindex^{\text{csi}} = \left\lvert \cup_{t \leq \offlinetrainingperiod} \userset_\agentindex^{t,\text{si}} \right\rvert$ denote the number of unlabeled \gls{csi} samples estimated by the agent.
Practically, our objective is to build a channel distance matrix $\distancematrix_\agentindex^{\text{csi}} \in \mathbb{R}^{\numsamples_\agentindex^{\text{csi}} \times \numsamples_\agentindex^{\text{csi}}}$ whose elements are channel distances $\sbrk{\distancematrix_\agentindex^{\text{csi}}}_{i,j} = \channeldistance\rbrk{\channel_i, \channel_j}$.
Our aim is to align this distance matrix with the previously computed location distance matrix, so that the agent can determine the matching between the two modalities.
Therefore, $\distancematrix_\agentindex^{\text{csi}}$ must satisfy the following property: the distance between a pair of channels collected from two locations approximates the Euclidean distance of their locations (up to a scalar factor).

To construct our desired matrix, we draw inspiration from \emph{channel charting}, a self-supervised learning approach that seeks to embed high-dimensional \gls{csi} into a low-dimensional space conserving spatial neighborhoods~\cite{studer2018cc}.
In fact, typical channel charting techniques start by defining a channel dissimilarity, and then compress the \gls{csi} manifold by applying dimensionality reduction given the distance.
For our work, since we seek a globally robust distance that conserves the overall geometry of the system, we rely on the \gls{adp} distance proposed in~\cite{stephan2024angle}:
\begin{equation}\label{eq:mm_paper_adp_distance}
    \channeldistance_{\text{ADP}}\rbrk{\channel_i, \channel_j} = \sum_{t} 1 - \frac{\left\lvert \sum_{n=1}^\numantennas \tilde{\channel}^*_{i,n,t} \tilde{\channel}_{j,n,t} \right\rvert^2}{\left\lvert \sum_{n=1}^\numantennas \tilde{\channel}^*_{i,n,t} \right\rvert^2 \left\lvert \sum_{n=1}^\numantennas \tilde{\channel}^*_{j,n,t} \right\rvert^2},
\end{equation}
where $\tilde{\channel}_{i}$ is the inverse \gls{dft} of channel $\channel_i$ applied over the subcarrier axis.
In \eqref{eq:mm_paper_adp_distance}, $t$ indexes the channel impulse response taps and $n$ indexes the \gls{bs} antennas.
Such channel dissimilarities are reliable for local neighborhoods, meaning the distance between two channel samples that are spatially close is small; however it is not necessarily large for widely separated channels.
To obtain a globally representative channel dissimilarity approximating the actual terminal positions, we compute the \emph{geodesic} distance corresponding to the \gls{adp} dissimilarity.
Simply put, the geodesic distance between two samples is the sum of small dissimilarities between intermediate samples that are close (and for those nearby points, the \gls{adp} dissimilarity is accurate).
To obtain the geodesic distance, we start by computing the pairwise \gls{adp} metric for the data set $\dataset_\agentindex^{\text{csi}}$ and then find the $k$ nearest neighbors for every sample ($k=20$), and set the dissimilarity to other samples to an arbitrarily high constant.
Finally, the geodesic distance between two samples is the length of the shortest path between them, obtained using a shortest path algorithm (for instance Dijkstra's algorithm~\cite{dijkstra1959anote}).
The agent collects the geodesic distances in a matrix $\distancematrix_\agentindex^{\text{csi}}$.

\subsection{Self-supervised Multimodal Alignment}\label{subsection:mm_paper_multimodal_alignment}
To perform sensing, the agent needs to determine which estimated \gls{csi} corresponds to which vehicle in the image and equivalently its extracted location.
We propose to align the image and channel data sets, by matching their distances $\distancematrix_\agentindex^{\text{im}}$ and $\distancematrix_\agentindex^{\text{csi}}$~\cite{cui2014generalized}.
Without loss of generality, we assume $\numsamples_\agentindex^{\text{im}} \leq \numsamples_\agentindex^{\text{csi}}$, meaning the agent collects more \gls{csi} samples than images since \gls{csi} is easier to store for a \gls{bs}.
Formally, the agent aims to find the matching matrix $\matchingmatrix_\agentindex \in \matchingmatrixset_\agentindex = \cbrk{\matchingmatrix \in \cbrk{0,1}^{\numsamples_\agentindex^{\text{im}} \times \numsamples_\agentindex^{\text{csi}}} \mid \matchingmatrix \mathbf{1}_{\numsamples_\agentindex^{\text{csi}}} = \mathbf{1}_{\numsamples_\agentindex^{\text{im}}}, \matchingmatrix^\trans \mathbf{1}_{\numsamples_\agentindex^{\text{im}}} \leq \mathbf{1}_{\numsamples_\agentindex^{\text{csi}}}}$, where $\sbrk{\matchingmatrix_\agentindex}_{i,j} = 1$ means that the estimated location $\estimateduserposition_i \in \dataset_\agentindex^{\text{im}}$ corresponds to channel sample $\channel_j \in \dataset_\agentindex^{\text{csi}}$.
We formalize our multimodal alignment problem as follows:
\begin{equation}\label{eq:mm_paper_alignment_problem}
    \underset{\matchingscalar_\agentindex \geq 0, \; \matchingmatrix \in \matchingmatrixset_\agentindex}{\text{minimize}} \qquad
    \left\lVert \distancematrix_\agentindex^{\text{im}} - \matchingscalar_\agentindex \, \matchingmatrix_\agentindex \, \distancematrix_\agentindex^{\text{csi}} \, \matchingmatrix_\agentindex^\trans \right\rVert_F^2,
\end{equation}
where $\matchingscalar_\agentindex$ is an optimization parameter that rescales the distances (since the \gls{csi} distance is only proportional to the physical distance), and $\lVert\cdot\rVert_F$ is the Frobenius norm.
The main challenge in solving problem~\eqref{eq:mm_paper_alignment_problem} stems from the fact that the feasibility set $\matchingmatrixset_\agentindex$ is hard matching set which is neither closed nor convex.
To mitigate this issue, we relax the feasibility set to a soft matching $\matchingmatrixset_\agentindex^\prime = \Bigl\{\matchingmatrix \in \sbrk{0,1}^{\numsamples_\agentindex^{\text{im}} \times \numsamples_\agentindex^{\text{csi}}} \mid \matchingmatrix \mathbf{1}_{\numsamples_\agentindex^{\text{csi}}} = \mathbf{1}_{\numsamples_\agentindex^{\text{im}}}, \matchingmatrix^\trans \mathbf{1}_{\numsamples_\agentindex^{\text{im}}} \leq \mathbf{1}_{\numsamples_\agentindex^{\text{csi}}}\Bigr\}$.
Now, the elements of a matrix $\matchingmatrix_\agentindex \in \matchingmatrixset_\agentindex^\prime$ have a probabilistic interpretation where $\sbrk{\matchingmatrix_\agentindex}_{i,j}$ is the likelihood of image-based location $i$ corresponding to channel $j$~\cite{cao2020unsupervised},~\cite{cohen2023joint}.
With the relaxation of the feasibility set to $\matchingmatrixset_\agentindex^\prime$, we propose to solve~\eqref{eq:mm_paper_alignment_problem} using a primal-dual method.
Essentially, we minimize the Lagrangian associated to~\eqref{eq:mm_paper_alignment_problem} by alternating steps over the scaling parameter $\matchingscalar_\agentindex$ and the matrix $\matchingmatrix_\agentindex$.
We provide more details on our proposed solution to~\eqref{eq:mm_paper_alignment_problem} in Appendix~\ref{proof:mm_paper_details_matching_problem}.
The obtained solution variables $\rbrk{\matchingmatrix_\agentindex^\star, \matchingscalar_\agentindex^\star}$ correspond to the matching matrix that best aligns the vehicles' channels with their locations extracted from images by aligning their computed distances, and the scaling factor $\matchingscalar_\agentindex^\star$ that accounts for channel distance dis-proportionality.
Moreover, as shown in Appendix~\ref{proof:mm_paper_details_matching_problem}, the computational cost of our algorithm scales as $O\rbrk{\rbrk{\numsamples_\agentindex^{\text{csi}}}^3 + \rbrk{\numsamples_\agentindex^{\text{csi}}}^2 \numsamples_\agentindex^{\text{im}} + \rbrk{\numsamples_\agentindex^{\text{im}}}^2 \numsamples_\agentindex^{\text{csi}}}$, which is cubic in the number of available data samples.
Hence by moderately scaling the dataset sizes $\numsamples_\agentindex^{\text{csi}}$ and $\numsamples_\agentindex^{\text{im}}$ to improve the matching solution, the complexity of our solution remains practical, and is comparable to that of standard kernel methods and linear algebra routines.

\begin{ttwocolumn}
\begin{table*}[b]
\hrulefill
\begin{equation}\label{eq:mm_paper_cc_loss}
    \ell\rbrk{\chartingfunctionparams_\agentindex} = \sum_{\channel_i, \channel_j \in \unlabeledcsidataset_\agentindex^{\text{csi}}} 
    \underbracket[0.100ex][0.300ex]{
    \big\lvert \left\lvert \chartingfunction_{\chartingfunctionparams_\agentindex} \rbrk{\channel_i} - \chartingfunction_{\chartingfunctionparams_\agentindex} \rbrk{\channel_j} \right\rvert - \matchingscalar_\agentindex^{\star} \, \channeldistance_{\text{G-ADP}} \rbrk{\channel_i, \channel_j} \big\rvert^2}_{\text{self-supervised CSI loss}}
    + 
    \ccregularizer \sum_{\channel_i \in \dataset_\agentindex^{\text{csi}}} 
    \underbracket[0.100ex][0.300ex]{
    \left\lvert  \chartingfunction_{\chartingfunctionparams_\agentindex} \rbrk{\channel_i} - \matriximagesensing_\agentindex^\text{im} \matchingmatrix^\star_\agentindex \right\rvert^2}_{\text{cross-modal distillation loss}}.
\end{equation}
\end{table*}
\end{ttwocolumn}%

\begin{remark}
    Notice that the overall proposed multimodal alignment method is end-to-end self-supervised where the agent computes the distances between the modality samples $\distancematrix_\agentindex^{\text{im}}$ and $\distancematrix_\agentindex^{\text{csi}}$ and finds the correspondence between them with no external supervision.
    It is worth mentioning that our algorithm is versatile and can be extended to a semi-supervised matching.
    This corresponds to the case where the agent has partial information about the matching between some samples, for example whenever the \gls{rsu} estimates the channel of a vehicle that also transmits its position.
    For such instances, problem~\eqref{eq:mm_paper_alignment_problem} can be augmented by constraining the corresponding elements of $\matchingmatrix_\agentindex$ to $1$.
\end{remark}

\subsection{Multimodal Sensing}\label{subsection:mm_paper_multimodal_sensing}
So far, we showed how an agent collects offline multimodal image and \gls{csi} data and finds the alignment between their features.
However, during online deployment, the agent receives camera images $\rbrk{\image_{\agentindex, \cameraindex}^t}_{\cameraindex \in \setcameras_\agentindex}$ and wireless channels $\rbrk{\channel^{t}_{\agentindex, \userindex}}_{\userindex \in \userset_\agentindex^{t,\text{csi}}}$, and must find the sensing parameters $\sbrk{\estimateduserposition_\userindex}_{\userindex \in \userset_\agentindex^t}$.
For the images, the agent can follow the same offline procedure of subsection \ref{subsection:mm_paper_image_processing} to estimate the locations $\sbrk{\estimateduserposition_\userindex}$ for a subset of vehicles ${\userindex \in \userset_\agentindex^{t,\text{im}}}$.
The main challenge remains to infer the locations for the subset of vehicles ${\userindex \in \userset_\agentindex^{t,\text{csi}}}$ from wireless \gls{csi}.
To do so, we seek a function $\chartingfunction_\agentindex\!\!: \mathbb{C}^{\numantennas \times \numsubcarriers} \to \mathbb{R}^2$ that takes as input high-dimensional channels and outputs their corresponding locations.
Since this function must also generalize to unseen \gls{csi}, we propose to parametrize it by the weights of a neural network%
\footnote{Notice that the function is indexed by the agent, since each agent trains a separate function independently.}
 $\chartingfunctionparams_\agentindex$.

To train our localization function, we use a large dataset%
\footnote{In practice, only one data set $\unlabeledcsidataset_\agentindex^{\text{csi}}$ is collected by the agent, and then a small subset $\dataset_\agentindex^{\text{csi}} \subset \unlabeledcsidataset_\agentindex^{\text{csi}}$ is extracted (by uniform sampling) to perform the multimodal alignment, with $\left\lvert \unlabeledcsidataset_\agentindex^{\text{csi}} \right\rvert \gg \left\lvert \dataset_\agentindex^{\text{csi}} \right\rvert$.}
of unlabeled \gls{csi} $\unlabeledcsidataset_\agentindex^{\text{csi}}$ which is collected similarly to $\dataset_\agentindex^{\text{csi}}$.
We use the following loss function\begin{oonecolumn}:\end{oonecolumn}\begin{ttwocolumn} \eqref{eq:mm_paper_cc_loss}, shown at the bottom of the page. \end{ttwocolumn}%
\begin{oonecolumn}
\begin{equation}\label{eq:mm_paper_cc_loss}
    \ell\rbrk{\chartingfunctionparams_\agentindex} = \sum_{\channel_i, \channel_j \in \unlabeledcsidataset_\agentindex^{\text{csi}}} 
    \underbracket[0.100ex][0.300ex]{
    \big\lvert \left\lvert \chartingfunction_{\chartingfunctionparams_\agentindex} \rbrk{\channel_i} - \chartingfunction_{\chartingfunctionparams_\agentindex} \rbrk{\channel_j} \right\rvert - \matchingscalar_\agentindex^{\star} \, \channeldistance_{\text{G-ADP}} \rbrk{\channel_i, \channel_j} \big\rvert^2}_{\text{self-supervised CSI loss}}
    + 
    \ccregularizer \sum_{\channel_i \in \dataset_\agentindex^{\text{csi}}} 
    \underbracket[0.100ex][0.300ex]{
    \left\lvert  \chartingfunction_{\chartingfunctionparams_\agentindex} \rbrk{\channel_i} - \matriximagesensing_\agentindex^\text{im} \matchingmatrix^\star_\agentindex \right\rvert^2}_{\text{cross-modal distillation loss}}.
\end{equation}
\end{oonecolumn}%
Our loss function is a weighted sum of two losses.
The first loss is a \emph{self-supervised} loss which trains the network as a channel charting function, i.e., to embed two channels such that their embeddings distance equals a modified channel dissimilarity $\matchingscalar_\agentindex^{\star} \, \channeldistance_{\text{G-ADP}}$.
Recall that the scaling parameter $\matchingscalar_\agentindex$ is optimized such that the scaled channel distance matches with the physical distance estimated from images.
Thus, we exploit our multimodal alignment procedure to re-scale the geodesic channel distance $\channeldistance_{\text{G-ADP}}$, such that $\matchingscalar_\agentindex^{\star} \, \channeldistance_{\text{G-ADP}}$ well approximates the spatial distance.
Training our network with this loss only produces channel embeddings that could be arbitrarily rotated or translated versions of the desired user locations, since this loss only enforces that pairwise distances be comparable to physical distances.
Hence, our \emph{cross-modal distillation} loss grounds the channel embeddings by imposing the channels from the subset $\dataset_\agentindex^{\text{csi}}$ to have their corresponding aligned locations from $\dataset_\agentindex^{\text{im}}$, therefore obtaining a globally robust \gls{csi} sensing function $\chartingfunction_\agentindex$.
The parameter $\ccregularizer$ is trade-off parameter between the two terms, which we set to $5$.

\begin{remark}
    Our novel loss function \eqref{eq:mm_paper_cc_loss} can be interpreted as a typical channel charting loss, which is learned in a self-supervised fashion by a Siamese neural network, with two major differences stemming from our multimodal approach.
    First, we modify the geodesic \gls{adp} distance by scaling it with $\matchingscalar_\agentindex^{\star}$ which is learned by aligning image and channel features, hence obtaining a novel globally valid channel distance.
    Second, our cross-modal regularization term enforces the channel chart to be representative of the users' locations by exploiting localization data processed from the image modality, whereas previous approaches either rely on a small labeled \gls{csi} data set~\cite{stephan2024angle}, or a crafted \gls{csi} regularizer~\cite{taner2025channel}.
    It is also worth mentioning that compressing \gls{csi} via channel charting yields solid localization results even under significant non-line-of-sight and dynamic environments, as attempted on measured data in~\cite{stephan2024angle} and simulated channels in~\cite{taner2025channel}.
    As such, the obtained \gls{csi} embeddings remain representative of spatial user locations and can therefore be aligned with their corresponding counterparts through our proposed loss function.
\end{remark}

\begin{algorithm}[t]
\DontPrintSemicolon
\caption{Self-supervised Multimodal Sensing and Alignment}\label{alg:mm_paper_mutimodal}
\SetKwBlock{Begin}{Offline Training}{end function}
\Begin(Input unlabeled data sets $\dataset_\agentindex^{\text{im}}, \dataset_\agentindex^{\text{csi}}, \unlabeledcsidataset_\agentindex^{\text{csi}}$){
Extract sensing features $\lowdimdataset_\agentindex^{\text{im}}$ from $\dataset_\agentindex^{\text{im}}$ (section~\ref{subsection:mm_paper_image_processing})\;
Compute distance matrix $\distancematrix_\agentindex^{\text{im}}$\;
Compute pairwise channel dissimilarity \eqref{eq:mm_paper_adp_distance} for $\dataset_\agentindex^{\text{csi}}$ (section~\ref{subsection:mm_paper_csi_processing})\;
Compute geodesic distance matrix $\distancematrix_\agentindex^{\text{csi}}$\;
Solve \eqref{eq:mm_paper_alignment_problem} using primal-dual method and obtain $\rbrk{\matchingmatrix_\agentindex^\star, \matchingscalar_\agentindex^\star}$ (section~\ref{subsection:mm_paper_multimodal_alignment})\;
Fit model $\chartingfunction_\agentindex$ using data set $\unlabeledcsidataset_\agentindex^{\text{csi}}$ via gradient steps on $\ell\rbrk{\chartingfunctionparams_\agentindex}$ (section~\ref{subsection:mm_paper_multimodal_sensing} eq. \eqref{eq:mm_paper_cc_loss})\;
}
\SetKwBlock{Begin}{Online Inference}{end function}
\Begin(Input observation $\envobservation_\agentindex^t$){
Estimate locations $\estimateduserposition_\agentindex^{t, \text{im}} =  \sbrk{\estimateduserposition^t_\userindex}_{\userindex \in \userset_\agentindex^{t,\text{im}}}$ from images $\rbrk{\image_{\agentindex, \cameraindex}^t}_{\cameraindex \in \setcameras_\agentindex}$ (section~\ref{subsection:mm_paper_image_processing})\;
Estimate locations $\estimateduserposition^{t, \text{csi}}_\userindex = \chartingfunction_\agentindex \rbrk{\channel^{t}_{\agentindex, \userindex}}$ from \gls{csi}\;
Remove duplicates from $\sbrk{\estimateduserposition^{t, \text{im}}_\userindex, \estimateduserposition^{t, \text{csi}}_\userindex}$\;
Output estimated state $\estimatedstate_\agentindex^t = \sbrk{\estimateduserposition^t_\userindex}_{\userindex \in \userset_\agentindex^t}$
}
\end{algorithm}

\subsection{Cross-modal Imputation}
The agent can now utilize its trained function $\chartingfunction_\agentindex$ for online inference as follows.
At time slot $t$, the agent observes $\envobservation_\agentindex^t = \sbrk{\rbrk{\image_{\agentindex, \cameraindex}^t}_{\cameraindex \in \setcameras_\agentindex}, \rbrk{\channel^{t}_{\agentindex, \userindex}}_{\userindex \in \userset_\agentindex^{t,\text{csi}}}}$:
\begin{itemize}[nolistsep, leftmargin=*]
    \item The locations of the vehicles $\estimateduserposition_\agentindex^{t, \text{im}} =  \sbrk{\estimateduserposition^t_\userindex}_{\userindex \in \userset_\agentindex^{t,\text{im}}}$ in the images $\rbrk{\image_{\agentindex, \cameraindex}^t}_{\cameraindex \in \setcameras_\agentindex}$ are obtained following the procedure of subsection \ref{subsection:mm_paper_image_processing}.
    \item The locations of the vehicles $\estimateduserposition_\agentindex^{t, \text{csi}} = \sbrk{\estimateduserposition^t_\userindex}_{\userindex \in \userset_\agentindex^{t,\text{csi}}}$ with estimated \gls{csi} are obtained as $\estimateduserposition^t_\userindex = \chartingfunction_\agentindex \rbrk{\channel^{t}_{\agentindex, \userindex}}$.
\end{itemize}
Finally, the full set of vehicle positions $\estimatedstate_\agentindex^t = \sbrk{\estimateduserposition^t_\userindex}_{\userindex \in \userset_\agentindex^t}$ is obtained by removing duplicates from the concatenation of the above position subsets.
We identify duplicates (this is the case of a terminal appearing in both modalities $\userindex \in \userset_\agentindex^{t,\text{im}} \cap \userset_\agentindex^{t,\text{csi}}$) whenever the difference between the two estimated positions is less than a minimal distance $\sensingduplicateparameter$ (for example, the average length of a vehicle).
We summarize the proposed online training and offline inference routines in Algorithm~\ref{alg:mm_paper_mutimodal}.
We note that during the initial training part, the \gls{rsu} only trains the \gls{csi} model that is used to estimate the users' locations from wireless data, while the detection model that is used to localize the vehicles in the images (offline and online) is a pre-trained and frozen YOLOv7.

Notice that our proposed technique functions as a cross-modal imputation, where both modalities complement each other so that each agent can fully recover the desired sensing parameters (crucial to identify symmetries in the environment).
The subsequent section exploits this sensing information extracted (locally per agent) form the multimodal data to train a distributed symmetry-aware \gls{marl} policy solving problem~\eqref{eq:mm_paper_problem}.
The following remarks underscore the versatility of our approach.

\begin{remark}
    It is worth mentioning that our framework is non auto-regressive in the sense that for a missing user location, instead of relying on its previous progression pattern to impute its current value, the agent infers its current value from its aligned counterpart stemming from another modality.
\end{remark}

\begin{remark}
    Our framework can be easily extended with a decoding ability, allowing for instance, the estimation of wireless channels for users $\userindex \in \userset_\agentindex^{t,\text{im}} \setminus \userset_\agentindex^{t,\text{csi}}$.
    To do so, one can follow a similar approach to subsection \ref{subsection:mm_paper_multimodal_sensing}, where a channel decoder, taking as input vehicle positions and generating their wireless \gls{csi}, can be trained by again exploiting the aligned modalities.
    To focus on our substantial contributions, we leave such extensions for future works.
\end{remark}
\section{Distributed Rate Maximization with MARL under Latent Symmetries}
\label{section:mm_paper_marl}
%
In the previous section, we provided a framework where each \gls{bs} agent $\agentindex \in \agentset$ utilizes its high-dimensional multimodal observation $\envobservation_\agentindex^t$ to estimate the locations of the users in its region $\estimatedstate_\agentindex^t =  \sbrk{\estimateduserposition^t_\userindex}_{\userindex \in \userset_\agentindex^t}$.
As discussed previously, unlike the observations, the locations of the users exhibit symmetry properties which can be exploited to facilitate our distributed \gls{marl} training.
Essentially, we showed that the optimal \gls{mmdp} policy maximizing the bitrates is \emph{equivariant} under group rotations of the users' positions in the environment (see Fig.~\ref{fig:mm_paper_policy_permutation}).
In this section, we develop a neural network architecture to train our policy that obeys our desired \emph{equivariance} property, while allowing distributed execution.
Subsequently, we start by showing how to build a learnable neural network layer that satisfies equivariance.
We then use such equivariant layers to construct our policy network.

\subsection{Equivariant Layers}
There are many techniques proposed in the literature to design equivariant layers.
In this work, we follow the \emph{Symmetrizer} approach~\cite{van2020mdp}.
Consider a single linear layer%
\footnote{Here we stack the layer's biases $\neuralnetbias$ into the weights $\neuralnetweight \mapsto \sbrk{\neuralnetweight, \neuralnetbias}$ and extend the input $\neuralnetinput \mapsto \sbrk{\neuralnetinput, \mathbf{1}}$.}
with weights $\neuralnetweight \in \mathbb{R}^{D_{\text{out}} \times D_{\text{in}}}$ that maps inputs $\neuralnetinput \in \mathbb{R}^{D_{\text{in}}}$ to $\neuralnetweight \neuralnetinput$.
Given a pair of input-output linear group transformations $\rbrk{\mathbf{\groupaction}_g, \mathbf{\groupactiontwo}_g}_{g \in \group}$, we wish to find weights satisfying the equivariance property: 
\begin{equation}\label{eq:mm_paper_equivariance_symmetrizer}
    \neuralnetweightset = \cbrk{\neuralnetweight \mid \mathbf{\groupactiontwo}_g \neuralnetweight \neuralnetinput = \neuralnetweight \mathbf{\groupaction}_g \neuralnetinput, \; \forall g \in \group, \neuralnetinput \in \mathbb{R}^{D_{\text{in}}}}.
\end{equation}
The learnable weights $\neuralnetweight$ are updated during training, however we want to constrain them within $\neuralnetweightset$.
Hence, we propose to parametrize them as follows: $\neuralnetweight = \sum_{i=1}^{\text{rank}\rbrk{\neuralnetweightset}} \neuralnetbasisfactor_i \neuralnetbasis_i$ where $\cbrk{\neuralnetbasis_i}$ is a basis of $\neuralnetweightset = \text{span}\rbrk{\neuralnetbasis_i}$.
To find a basis for the equivariant subspace,~\cite{van2020mdp} introduces the \emph{Symmetrizer} functional: $\symmetrizer\rbrk{\neuralnetweight} = \frac{1}{\lvert \group \rvert} \sum_{g \in \group} \mathbf{\groupactiontwo}_g^{-1} \neuralnetweight \mathbf{\groupaction}_g$ which is proven to transform linear maps $\neuralnetweight$ to equivariant linear maps $\symmetrizer\rbrk{\neuralnetweight} \in \neuralnetweightset$, and then a singular value decomposition (SVD) is used to obtain the basis.
As such, before training, the weights are written as a linear combination of a basis of $\neuralnetweightset$, and the learnable parameters $\rbrk{\neuralnetbasisfactor_i}$ are updated during training using gradient-based optimization.

Given that we can now design an equivariant layer, we are interested in constructing a deep equivariant neural network.

\begin{oonecolumn}
\begin{figure}
    \centering
    \includegraphics[width=0.4\linewidth]{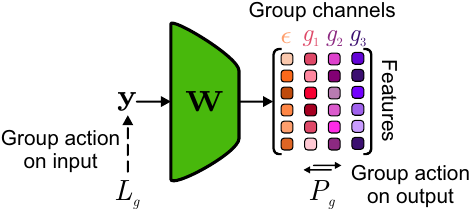}
    \caption{Equivariant layer: when the input is transformed by a group action, the output permutes over the group channels.}
    \label{fig:mm_paper_equivariant_layer}
\end{figure}
\end{oonecolumn}
\begin{ttwocolumn}
\begin{figure}
    \centering
    \includegraphics[width=0.7\linewidth]{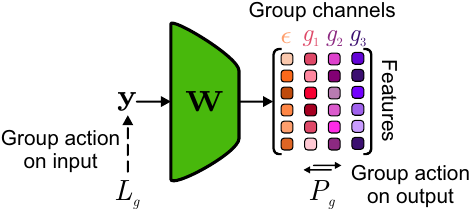}
    \caption{Equivariant layer: when the input is transformed by a group action, the output permutes over the group channels.}
    \label{fig:mm_paper_equivariant_layer}
\end{figure}
\end{ttwocolumn}

\begin{lemma}\label{lemma:mm_paper_composition_equivariance}
    Given a group $\group$, if $f_1$ is an $\rbrk{\groupaction_g, \permutation_g}$ equivariant function and $f_2$ is an $\rbrk{\permutation_g, \groupactiontwo_g}$ equivariant function, then the composition $f_2 \circ f_1$ is $\rbrk{\groupaction_g, \groupactiontwo_g}$ equivariant.
\end{lemma}

\begin{lemma}\label{lemma:mm_paper_nonlinearity_equivariance}
    Given a finite group $\group$, point-wise nonlinearities applied to an $\rbrk{\groupaction_g, \groupactiontwo_g}$ equivariant function preserve its equivariance.
\end{lemma}
\begin{proof}
    The proofs are deferred to Appendix~\ref{proof:mm_paper_composition_equivariance} and~\ref{proof:mm_paper_nonlinearity_equivariance}.
\end{proof}

Jointly, lemmas~\ref{lemma:mm_paper_composition_equivariance} and~\ref{lemma:mm_paper_nonlinearity_equivariance} imply that, similarly to classical neural networks, an equivariant network can be built by stacking equivariant layers followed by point-wise nonlinear activations (such as ReLU)~\cite{cohen2016group}.
However, we have to carefully design intermediate layers to have matching group representations.
For instance, to obtain an end-to-end $\rbrk{\groupaction_g, \groupactiontwo_g}$ equivariant two layer network, the output transformation of the first layer must be shared with the input transformation of the second layer (see Lemma~\ref{lemma:mm_paper_composition_equivariance}).

In the following, to build our equivariant policy network, we use permutations as group representations for intermediate (and output) layers.
As we show in Fig.~\ref{fig:mm_paper_equivariant_layer}, for a group $\group$, we let the $\lvert \group \rvert \times D_{\text{out}}$ output of the first layer permute along the group channels whenever the layer's input permutes by $\groupaction_g$ (a rotation in our case).
The intermediate layers are then designed such that their outputs similarly permutes whenever their inputs permutes along the group channels, and the final layer's output transformation is set to $\groupactiontwo_g$ (which is itself a policy permutation in our case).
With this in mind, we now explain the architecture of our policy networks.

\subsection{Equivariant MARL Policy Network}
We start by parameterizing each agent as a an actor-critic network as follows:
\begin{equation*}
    \label{eq:mm_paper_actor_critic} 
    \text{Actor:} \; \beam_\agentindex^t \sim \policy_{\agentindex, \policyparams} \rbrk{\estimatedstate_\agentindex^t}, \quad
    \text{Critic:} \; \valuefunction_{\agentindex, \criticparams} \rbrk{\estimatedstate_\agentindex^t} \approx \mathbb{E}_{\policy} \sbrk{\valuefunction^t},
\end{equation*}
where the input of both networks is the estimated state $\estimatedstate_\agentindex^t =  \sbrk{\estimateduserposition^t_\userindex}_{\userindex \in \userset_\agentindex^t}$.
While the critic estimates the returns achieved by the actors, the actors are trained to maximize the critic's output.
The actor and critic networks are respectively parametrized by $\policyparams$ and $\criticparams$, which requires communication between agents, encouraging coordination under decentralized execution.
In the following, we will detail our agents' network architecture, while the next subsection describes their training procedure.
We omit the parameterizations $\policyparams$ and $\criticparams$ whenever ambiguity is unlikely.

Our global equivariance constraint is: $\groupactiontwo_g^{\state} \sbrk{\policy\rbrk{\state}} = \policy\rbrk{\groupaction_g \sbrk{\state}}$, i.e., whenever the global state undergoes a rotation, the global policy must permute between the agents.
Since we want distributed execution with communication, it is natural to parametrize the global policy network $\policy$ as a \gls{gnn} over the \gls{mmdp} communication graph $\commgraph$, where each agent's policy $\policy_\agentindex$ is computed per node~\cite{battaglia2018relational},~\cite{satorras2021neural}.
Each agent's network consists of the following:
\begin{itemize}[nolistsep, leftmargin=*]
    \item An equivariant local observation encoder $\obsencoder_\agentindex\!\!: \statespace_\agentindex \to \mathbb{R}^{\lvert \group \rvert \times \encodingdim}$ where $\encodingdim$ is the encoding dimension,
    \item An equivariant message passing function $\messages_\agentindex\!\!: \commgraphedgeset \times \mathbb{R}^{\lvert \group \rvert \times \encodingdim} \to \mathbb{R}^{\lvert \group \rvert \times \messagesize}$ where $\messagesize$ is the message size,
    \item An equivariant local update function $\update_\agentindex\!\!: \mathbb{R}^{\lvert \group \rvert \times \encodingdim} \times \mathbb{R}^{\lvert \group \rvert \times \messagesize} \to \mathbb{R}^{\lvert \group \rvert \times \encodingdim}$ where $\messagesize$ is the message size,
    \item An equivariant policy head $\localpolicy_\agentindex\!\!: \mathbb{R}^{\lvert \group \rvert \times \encodingdim} \to \sbrk{0, 1}^{\lvert \beamset_\agentindex \rvert}$,
    \item An invariant value head $\localvalue_\agentindex\!\!: \mathbb{R}^{\lvert \group \rvert \times \encodingdim} \to \mathbb{R}$.
\end{itemize}
The agent's network is built using an encoding function $\obsencoder_\agentindex$, followed by $\numlayers$ message passing layers using the messaging $\messages_\agentindex$ and update $\update_\agentindex$ functions.
Finally, on top of this backbone network, the policy $\localpolicy_\agentindex$ and value $\localvalue_\agentindex$ heads receive the final local encoding to output the agent's predicted policy and value.
We now detail the architecture shown in Fig.~\ref{fig:mm_paper_gnn}.

\subsubsection{Encoding}
The encoder must satisfy the equivariance constraint $\forall g \in \group$:
\begin{equation}\label{eq:mm_paper_encoder_equivariance}
    \permutation_g \sbrk{\obsencoder_\agentindex \rbrk{\estimatedstate_\agentindex}} = \obsencoder_\agentindex\rbrk{\groupaction_g \sbrk{\estimatedstate_\agentindex}} = \obsencoder_\agentindex\rbrk{\rotationmatrix_g \sbrk{\estimateduserposition_1}, \dots, \rotationmatrix_g \sbrk{\estimateduserposition_{\lvert \userset_\agentindex \rvert}}},
\end{equation}
where $G=\cyclicgroup_4$ is the group of $90$\textdegree\ rotations, the output transformation $\permutation_g$ is a permutation along the group channels, and the input transformation $\groupaction_g$ corresponds to a rotation of each location in the estimated state $\estimatedstate_\agentindex =  \sbrk{\estimateduserposition_\userindex}_{\userindex \in \userset_\agentindex}$.
To implement such a transformation, we use a direct sum representation:
\begin{oonecolumn}
\begin{equation}\label{eq:mm_paper_rotation_directsum}
    \groupaction_g = \oplus_{\userindex \in \userset_\agentindex} \rotationmatrix_g = 
    \begingroup
    \setlength\arraycolsep{2pt}
    \begin{bmatrix}
        \rotationmatrix_g & 0 & \dots & 0 \\[-10pt]
        0 & \rotationmatrix_g &  & \vdots \\[-10pt]
        \vdots &  & \ddots & 0 \\[-10pt]
        0 & \dots & 0 & \rotationmatrix_g
    \end{bmatrix}
    \endgroup
\end{equation}
\end{oonecolumn}%
\begin{ttwocolumn}
\begin{equation}\label{eq:mm_paper_rotation_directsum}
    \groupaction_g = \oplus_{\userindex \in \userset_\agentindex} \rotationmatrix_g = 
    \begingroup
    \setlength\arraycolsep{2pt}
    \begin{bmatrix}
        \rotationmatrix_g & 0 & \dots & 0 \\[-4pt]
        0 & \rotationmatrix_g &  & \vdots \\[-4pt]
        \vdots &  & \ddots & 0 \\[-4pt]
        0 & \dots & 0 & \rotationmatrix_g
    \end{bmatrix}
    \endgroup
\end{equation}
\end{ttwocolumn}%
where $\rotationmatrix_g$ is a rotation matrix (eq. ~\eqref{eq:mm_paper_rotation_matrix}) of the corresponding group element.
Its output, denoted $\encoding_\agentindex^0 = \obsencoder_\agentindex \rbrk{\estimatedstate_\agentindex}$, carries the rotation of the state by the ordering of its channels, and the state features are the elements in those channels.

For implementation purposes, the intermediate permutation matrices are (shorthand notations): $\permutation_{\groupidentityelement}=\allowbreak\sbrk{0,1,2,3}$, \mbox{$\permutation_{g_1}=\sbrk{3,0,1,2}$}, \mbox{$\permutation_{g_2}=\sbrk{2,3,0,1}$}, \mbox{$\permutation_{g_3}=\sbrk{1,2,3,0}$}.

\begin{oonecolumn}
\begin{figure*}
    \centering
    \includegraphics[width=0.95\linewidth]{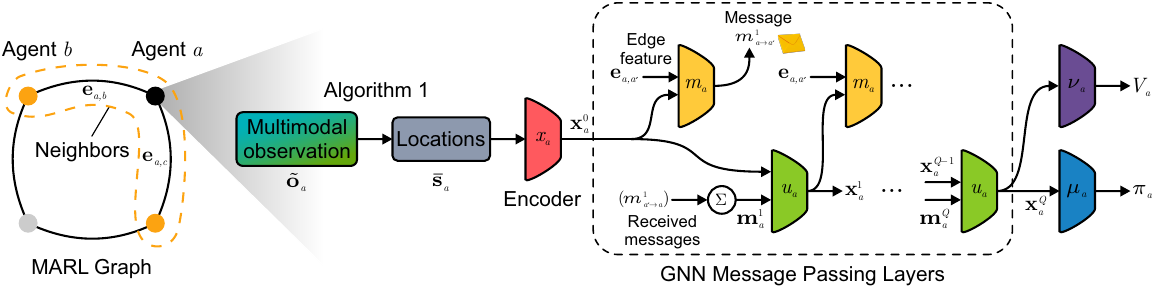}
    \caption{Proposed \gls{gnn} for \gls{marl} equivariant policy training. Each agent collects multimodal observations which are processed as detailed in section~\ref{section:mm_paper_ssl} and algorithm~\ref{alg:mm_paper_mutimodal} to extract its users' locations which serve as the \gls{gnn} input. Each agent encodes its state estimate, which is used along an emergent signaling scheme over the \gls{gnn}'s message passing layers to update its state embedding, and finally compute its policy. All networks are implemented as equivariant layers followed by point-wise nonlinearities.}
    \label{fig:mm_paper_gnn}
\end{figure*}
\end{oonecolumn}
\begin{ttwocolumn}
\begin{figure*}
    \centering
    \includegraphics[width=.85\linewidth]{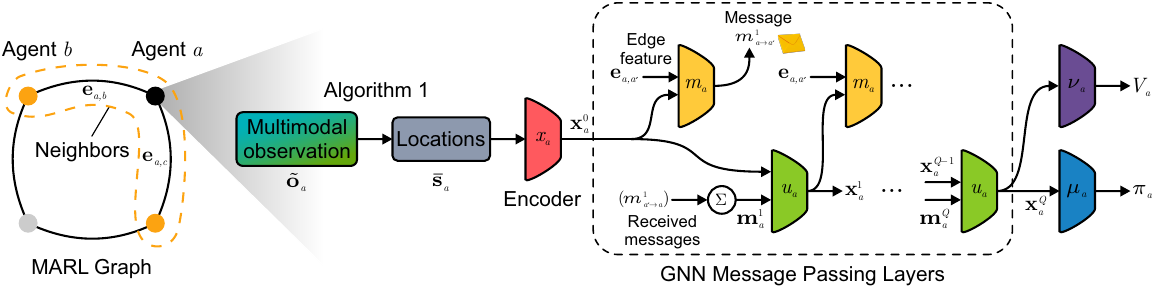}
    \caption{Proposed \gls{gnn} for \gls{marl} equivariant policy training. Each agent collects multimodal observations which are processed as detailed in section~\ref{section:mm_paper_ssl} and algorithm~\ref{alg:mm_paper_mutimodal} to extract its users' locations which serve as the \gls{gnn} input. Each agent encodes its state estimate, which is used along with received messages over the \gls{gnn}'s message passing layers to update its state embedding, and finally compute its policy. All networks are implemented as equivariant layers followed by point-wise nonlinearities.}
    \label{fig:mm_paper_gnn}
\end{figure*}
\end{ttwocolumn}

\subsubsection{Message Passing}
Given the agents' initial encoding $\encoding_\agentindex^0$, neighboring agents exchange messages for $\numlayers$ rounds, given by $\messages^\layerindex_{\agentindex \to \agentindex^\prime} = \messages_\agentindex \rbrk{\encoding_\agentindex^{\layerindex}, \edgefeatures_{\agentindex^\prime, \agentindex}}$, where $\layerindex$ indexes the layer, $\edgefeatures_{\agentindex^\prime, \agentindex} = \agentposition_{\agentindex^\prime} - \agentposition_\agentindex$ is the edge features corresponding to the difference between \gls{rsu} locations, and $\encoding_\agentindex^{\layerindex}$ is the agent's encoding at round $\layerindex$ (initialized by the encoder).
The messages can be interpreted as an emergent protocol used by the agents to coordinate their policies under partial observability~\cite{chafii2023emergent},~\cite{lazaridou2020emergent}.
The message function must satisfy the following equivariance $\forall g \in \group$:
\begin{equation}\label{eq:mm_paper_message_equivariance}
    \permutation_g \sbrk{\messages_\agentindex \rbrk{\encoding_\agentindex^{\layerindex}, \edgefeatures_{\agentindex^\prime, \agentindex}}} = \messages_\agentindex \rbrk{\permutation_g \sbrk{\encoding_\agentindex^{\layerindex}}, \rotationmatrix_g \sbrk{\edgefeatures_{\agentindex^\prime, \agentindex}}}.
\end{equation}
In other words, whenever the local state rotates which is detected by a permutation $\permutation_g$ of the agent's encoding $\encoding_\agentindex^{\layerindex}$, and simultaneously the agent location differences $\edgefeatures_{\agentindex^\prime, \agentindex}$ rotate according to $\rotationmatrix_g$ (as explained in Fig.~\ref{fig:mm_paper_policycomm_permutation}), then the exchanged messages are permuted by $\permutation_g$.
We build an equivariant layer message passing layer with the direct sum representation $\permutation_g \oplus \rotationmatrix_g$, similarly to \eqref{eq:mm_paper_rotation_directsum}.
After message passing round $\layerindex$, each agent $\agentindex$ must aggregate its received messages $\rbrk{\messages^\layerindex_{\agentindex^\prime \to \agentindex}}_{\rbrk{\agentindex^\prime, \agentindex} \in \commgraphedgeset}$ in such a way to preserve their equivariance.

\begin{lemma}\label{lemma:mm_paper_linearcombination_equivariance}
    A linear combination $\sum_{\rbrk{\agentindex^\prime, \agentindex} \in \commgraphedgeset} \messages^\layerindex_{\agentindex^\prime \to \agentindex}$ of equivariant functions is equivariant.
\end{lemma}
\begin{proof}
    The proof is deferred to Appendix~\ref{proof:mm_paper_linearcombination_equivariance} 
\end{proof}

Lemma~\ref{lemma:mm_paper_linearcombination_equivariance} provides a simple yet effective scheme to aggregate each agent's received messages by taking their sum $\aggregatedmessages_\agentindex^\layerindex =  \sum_{\rbrk{\agentindex^\prime, \agentindex} \in \commgraphedgeset} \messages^\layerindex_{\agentindex^\prime \to \agentindex}$.

\subsubsection{Local Update}
After each message passing layer $q$, each agent updates its encoding $\encoding_\agentindex^{\layerindex + 1} = \update_\agentindex \rbrk{\encoding_\agentindex^\layerindex, \aggregatedmessages_\agentindex^\layerindex}$ given its previous encoding and collected messages.
The node update function $\update_\agentindex$ must be permutation equivariant to permutation along the agent's encoding and messages, i.e., $\forall g \in \group$:
\begin{equation}\label{eq:mm_paper_update_equivariance}
    \permutation_g \sbrk{\update_\agentindex \rbrk{\encoding_\agentindex^\layerindex, \aggregatedmessages_\agentindex^\layerindex}} = \update_\agentindex \rbrk{\permutation_g \sbrk{\encoding_\agentindex^\layerindex}, \permutation_g \sbrk{\aggregatedmessages_\agentindex^\layerindex}}.
\end{equation}
Notice that the same permutation group action $\permutation_g$ is used for all layers, guaranteeing the network's end-to-end equivariance.
Given each agent's updated features at round $\layerindex$, another message passing round occurs and the features are updated again by local node updates, with the whole process spanning $\numlayers$ rounds.
Finally, each agent's local embeddings are $\encoding_\agentindex^\numlayers$.

\subsubsection{Policy Head}
The agent's policy head computes its action distribution as $\policy_\agentindex \rbrk{\estimatedstate_\agentindex^t} = \localpolicy_\agentindex \rbrk{\encoding_\agentindex^\numlayers}$, and is an equivariant layer satisfying $\forall g \in \group$:
\begin{equation}\label{eq:mm_paper_policy_equivariance}
    \groupactiontwo_g \sbrk{\localpolicy_\agentindex\rbrk{\encoding_\agentindex^\numlayers}} = \localpolicy_\agentindex\rbrk{\permutation_g \sbrk{ \encoding_\agentindex^\numlayers}},
\end{equation}
where $\groupactiontwo_g$ is a permutation along each \gls{bs} codebook.

To obtain the desired policy permutation shown in Fig.~\ref{fig:mm_paper_policy_permutation}, the output permutation matrices are (shorthand notations): $\groupactiontwo_{\groupidentityelement}=\sbrk{0,1,\dots,\lvert\beamset_\agentindex\rvert}$, \mbox{$\groupactiontwo_{g_1}=\sbrk{\lvert\beamset_\agentindex\rvert,\dots,1,0}$}, \mbox{$\groupactiontwo_{g_2}=\sbrk{0,1,\dots,\lvert\beamset_\agentindex\rvert}$}, \mbox{$\groupactiontwo_{g_3}=\sbrk{\lvert\beamset_\agentindex\rvert,\dots,1,0}$}.

\begin{algorithm}[t]
\DontPrintSemicolon
\caption{Proposed GNN Forward Pass}\label{alg:mm_paper_gnn}
\hspace{-.7em}\textbf{Input} Observations $\rbrk{\envobservation_\agentindex}_{\agentindex\in\agentset}$\;
\For{\textnormal{agent} $\agentindex \in \agentset$}{
Estimate $\estimatedstate_\agentindex = \sbrk{\estimateduserposition_\userindex}_{\userindex \in \userset_\agentindex}$ using Algorithm~\ref{alg:mm_paper_mutimodal}\;
Compute encoding $\encoding_\agentindex^0 = \obsencoder_\agentindex \rbrk{\estimatedstate_\agentindex}$}
\For{\textnormal{round} $\layerindex = 1, \dots, \numlayers$}{
\For{\textnormal{edge} $\rbrk{\agentindex, \agentindex^\prime} \in \commgraphedgeset$}{
Compute message $\messages^\layerindex_{\agentindex \to \agentindex^\prime} = \messages_\agentindex \rbrk{\encoding_\agentindex^{\layerindex - 1}, \edgefeatures_{\agentindex^\prime, \agentindex}}$
}
\For{\textnormal{agent} $\agentindex \in \agentset$}{
Aggregate messages $\aggregatedmessages_\agentindex^\layerindex =  \sum_{\rbrk{\agentindex^\prime, \agentindex} \in \commgraphedgeset} \messages^\layerindex_{\agentindex^\prime \to \agentindex}$\;
Update encoding $\encoding_\agentindex^{\layerindex + 1} = \update_\agentindex \rbrk{\encoding_\agentindex^\layerindex, \aggregatedmessages_\agentindex^\layerindex}$
}
}
\For{\textnormal{agent} $\agentindex \in \agentset$}{
Compute policy $\policy_\agentindex \rbrk{\estimatedstate_\agentindex} = \localpolicy_\agentindex \rbrk{\encoding_\agentindex^\numlayers}$
}
\end{algorithm}

\subsubsection{Value Head}
The agent's value head computes the estimated value function as $\valuefunction_\agentindex \rbrk{\estimatedstate_\agentindex^t} = \localvalue_\agentindex \rbrk{\encoding_\agentindex^\numlayers}$, and is an invariant layer satisfying $\forall g \in \group$:
%
\begin{equation}\label{eq:mm_paper_value_invariance}
    \localvalue_\agentindex\rbrk{\encoding_\agentindex^\numlayers} = \localvalue_\agentindex\rbrk{\permutation_g \sbrk{ \encoding_\agentindex^\numlayers}}.
\end{equation}
%
which constrains the value function to be invariant to global state transformations.
The value head is a typical equivariant layer where the output transformation representations are identity matrices, implemented similarly to eq.~\eqref{eq:mm_paper_policy_equivariance} with $\groupactiontwo_g = I, \;\;\forall g \in \group$.
\begin{ttwocolumn}
\begin{table*}[b]
\hrulefill
\begin{equation}\label{eq:mm_paper_ppo_actor_loss}
    \ell\rbrk{\policyparams} = \mathbb{E}_{\policy_\agentindex^{\text{old}}} \sbrk{\min\rbrk{\frac{\policy_{\agentindex, \policyparams}\rbrk{\estimatedstate_\agentindex}}{\policy_{\agentindex, \policyparams}^{\text{old}}\rbrk{\estimatedstate_\agentindex}} \hat{\advantage}_\agentindex\rbrk{\estimatedstate_\agentindex},
    \text{clip}\rbrk{\frac{\policy_{\agentindex, \policyparams} \rbrk{\estimatedstate_\agentindex}}{\policy_{\agentindex, \policyparams}^{\text{old}}\rbrk{\estimatedstate_\agentindex}}, 1-\ppoclip, 1+\ppoclip} \hat{\advantage}_\agentindex \rbrk{\estimatedstate_\agentindex}} - \ppoentropyfactor\, \text{H} \rbrk{\policy_{\agentindex, \policyparams}^{\text{old}} \rbrk{\estimatedstate_\agentindex}}}
\end{equation}
\end{table*}
\end{ttwocolumn}%

Note that the agent's policy $\policy_\agentindex \rbrk{\estimatedstate_\agentindex^t}$ and value $\valuefunction_\agentindex \rbrk{\estimatedstate_\agentindex^t}$ networks comprise all the \gls{gnn} message passing layers as a common backbone, followed by the policy and value heads respectively.
Algorithm~\ref{alg:mm_paper_gnn} outlines a forward pass in our proposed \gls{gnn} network.
The following result guarantees the equivariance of our policy and value networks.

\begin{proposition}\label{proposition:mm_paper_policy_equivariance}
    The global policy $\rbrk{\policy_\agentindex}_{\agentindex \in \agentset}$ and value function $\valuefunction_\agentindex$ computed by our network are respectively equivariant and invariant under global state rotations: $\forall g \in \group, \, \groupactiontwo_g^{\state} \sbrk{\policy \rbrk{\state}} = \policy\rbrk{\groupaction_g \sbrk{\state}}$ and $\valuefunction_\agentindex \rbrk{\state} = \valuefunction_\agentindex \rbrk{\groupaction_g \sbrk{\state}}$.
\end{proposition}
\begin{proof}
    The proof is deferred to Appendix~\ref{proof:mm_paper_policy_equivariance}.
\end{proof}

Given the constructed equivariant network, we now provide the \gls{marl} training procedure in the following subsection, while the remark below discusses the inference latency of the overall proposed models.

\begin{remark}
    By inspecting Fig.~\ref{fig:mm_paper_gnn}, we note that at the level of each \gls{rsu}, the multimodal observations are processed by the image and \gls{csi} models to extract the vehicles' locations which feed the \gls{gnn} policy model. For deployment considerations, the inference latency of this overall operation can be minimized as follows. The image processing latency can be efficiently reduced using modern post-training techniques on YOLOv7 such as pruning, quantization, and distillation~\cite{liu2024lightweight},~\cite{frantar2022optimal}, and utilizing effective deep learning hardware~\cite{shuvo2023efficient}, noting that such lightweight detection models are already integrated in many industrial cameras. Similar model compression methods can be used for the \gls{csi} neural network that has a less complex architecture compared that of the image network. Regarding the \gls{gnn} layers, the parametrization of the encoding, messaging and update functions is done using shallow networks with one or two layers of small size (see Section~\ref{section:mm_paper_simulation}-\ref{subsection:mm_paper_simulation_setting}); hence, inducing a tolerable inference latency. We emphasize that while messaging passing between \glspl{rsu} incurs a delay to the inference process, our distributed optimization method overcomes the intolerable delay of centralized optimization methods where all nodes must communicate their multimodal data to a centralized server for processing and decision feedback.
\end{remark}

\subsection{Proximal Policy Optimization}
We train our \gls{marl} using \gls{ppo}~\cite{schulman2017proximal}.
The actor's loss is computed as\begin{oonecolumn}:\end{oonecolumn}\begin{ttwocolumn} shown on the bottom of the page in \eqref{eq:mm_paper_ppo_actor_loss},\end{ttwocolumn}
\begin{oonecolumn}
\begin{equation}\label{eq:mm_paper_ppo_actor_loss}
    \ell\rbrk{\policyparams} = \mathbb{E}_{\policy_\agentindex^{\text{old}}} \sbrk{\min\rbrk{
    \frac{\policy_{\agentindex, \policyparams}\rbrk{\estimatedstate_\agentindex}}{\policy_{\agentindex, \policyparams}^{\text{old}}\rbrk{\estimatedstate_\agentindex}} \hat{\advantage}_\agentindex\rbrk{\estimatedstate_\agentindex},
    \text{clip}\rbrk{\frac{\policy_{\agentindex, \policyparams} \rbrk{\estimatedstate_\agentindex}}{\policy_{\agentindex, \policyparams}^{\text{old}}\rbrk{\estimatedstate_\agentindex}}, 1-\ppoclip, 1+\ppoclip} \hat{\advantage}_\agentindex \rbrk{\estimatedstate_\agentindex}
    }
    - \ppoentropyfactor\, \text{H} \rbrk{\policy_{\agentindex, \policyparams}^{\text{old}} \rbrk{\estimatedstate_\agentindex}}
    }
\end{equation}
\end{oonecolumn}
where the advantage for $T$ steps is estimated using \gls{gae}:
\begin{equation}\label{eq:mm_paper_ppo_gae}
    \hat{\advantage}_\agentindex \rbrk{\estimatedstate_\agentindex^t} = \advantagefactor_t + \discountfactor \advantagediscount \advantagefactor_{t+1} + \dots + \rbrk{\discountfactor \advantagediscount}^{T-t+1} \advantagefactor_{T-1},
\end{equation}
with $\advantagefactor_t = \reward_t + \discountfactor \valuefunction\rbrk{\estimatedstate_\agentindex^{t+1}} - \valuefunction\rbrk{\estimatedstate_\agentindex^t}$ and $\advantagediscount=0.95$ is a discounting hyperparameter.
In \eqref{eq:mm_paper_ppo_actor_loss}, the clipping factor $\ppoclip=0.2$ stabilizes policy updates, while the actor's entropy is regularized with $\ppoentropyfactor=0.01$ to encourage exploration.
This loss updates the actor's policy to increase the probability of actions with positive advantage.
The critic is trained to regress the target value function with the following loss:
\begin{equation}\label{eq:mm_paper_ppo_critic_loss}
    \ell\rbrk{\criticparams} = \mathbb{E}_{\policy} \sbrk{\frac{1}{2} \rbrk{
    \valuefunction_{\agentindex, \criticparams} \rbrk{\estimatedstate_\agentindex^t} - \hat{\reward_t}}^2
    }
\end{equation}
where $\hat{\reward_t}$ is the sum of future discounted rewards.

To train the networks, the actors follow a fixed policy $\policy_{\agentindex, \policyparams}^{\text{old}}$ to collect environmental data for a certain number of epochs.
This collected data is used to update the actor and critic networks via gradient descent on \eqref{eq:mm_paper_ppo_actor_loss} and \eqref{eq:mm_paper_ppo_critic_loss} to improve the policies, and the procedure is repeated until convergence.

\subsection{Theoretical Analysis under Partial Symmetry}
Throughout our discussion in this section, we have assumed that the latent symmetry obeyed by our environment is \emph{exact}, i.e., when the unknown vehicle locations undergo an exact rotation, the global policy permutes between the agents.
However, in practice, this assumption of perfect rotation symmetry is only partially satisfied.
For instance, practical road infrastructure will only adhere to an approximate rotation symmetry, introducing an asymmetry error that is due to asymmetric road topologies.
In such cases, enforcing perfect policy equivariance may degrade the performance, since the optimal policy is no longer necessarily equivariant.
The following result bounds the performance loss due to utilizing our proposed equivariant policy under imperfect symmetries caused by errors breaking idealized group symmetries.
\begin{proposition}\label{proposition:mm_paper_partial_symmetry}
    Assume there exists $\partialsymmetryR, \partialsymmetryT \geq 0$, such that the reward and transition operator of the \gls{mmdp} satisfy, $\forall g \in \group, \state, \state^\prime \in \statespace, \beam \in \beamset$:
    \begin{align}
        \label{eq:mm_paper_mmdp_partial_symmetric_reward}
        &\left\lvert \reward\rbrk{\state, \beam} - \reward\rbrk{\groupaction_g \sbrk{\state}, \groupactiontwo_g^ \state\sbrk{\beam}} \right\rvert \leq \partialsymmetryR, \\
        \label{eq:mm_paper_mmdp_partial_symmetric_transition}
        &\underset{h \in \mathcal{H}}{\sup} \Bigl| \mathbb{E}_{\state^\prime \sim \transitionfunction \rbrk{\state^\prime \mid \beam, \state}}\sbrk{h \rbrk{\state^\prime}} - \mathbb{E}_{\state^\prime \sim \transitionfunction\rbrk{\groupaction_g\sbrk{\state^\prime} \mid  \groupactiontwo_g^\state\sbrk{\beam}, \groupaction_g\sbrk{\state}}}\sbrk{h \rbrk{\state^\prime}} \Bigr| \leq \partialsymmetryT,
    \end{align}
    where $\mathcal{H}$ is the class of bounded mappings $\statespace\to\mathbb{R}$.
    Let $\actionvaluefunction^\star$ denote the optimal action value function of the \gls{mmdp}.
    Then, $\forall g \in \group, \state \in \statespace, \beam \in \beamset, \left\lvert \actionvaluefunction^\star \rbrk{\state,\beam} -\actionvaluefunction^\star \rbrk{\groupaction_g\sbrk{\state}, \groupactiontwo_g^ \state\sbrk{\beam}} \right\rvert \leq \rbrk{\partialsymmetryR + \discountfactor \partialsymmetryT}\rbrk{1-\discountfactor}^{-1}$.
\end{proposition}
\begin{proof}
    The proof is deferred to Appendix~\ref{proof:mm_paper_partial_symmetry}.
\end{proof}
Proposition~\ref{proposition:mm_paper_partial_symmetry} is explained as follows.
First, we start by relaxing the exact reward and transition invariance from eqs.~\eqref{eq:mm_paper_mmdp_symmetric_reward} and~\eqref{eq:mm_paper_mmdp_symmetric_transition} to the partial invariance forms shown in eqs.~\eqref{eq:mm_paper_mmdp_partial_symmetric_reward} and~\eqref{eq:mm_paper_mmdp_partial_symmetric_transition}.
Under this partial symmetry, constraining the policy network to be exactly equivariant yields a performance that is within the given bound compared to an optimal policy.
Notice that by setting $\partialsymmetryR=\partialsymmetryT=0$, we recover a distributed \gls{mmdp} with exact symmetries, that satisfies an equivariant optimal policy.
As such, while the above result holds analytical significance when the environment has approximate symmetries in the sense of~\eqref{eq:mm_paper_mmdp_partial_symmetric_reward} and~\eqref{eq:mm_paper_mmdp_partial_symmetric_transition}, we empirically test our proposed policy in such settings in the following section.
\section{Simulation Results}
\label{section:mm_paper_simulation}

\subsection{Setting}\label{subsection:mm_paper_simulation_setting}
We conduct simulations using data generated by the graphics software Blender, and the ray tracing tool Sionna~\cite{hoydis2022sionna}.
We design a scene in Blender as shown in Fig.~\ref{fig:mm_paper_simulation_setup}, where each \gls{bs} is equipped with 4 cameras positioned on top of the buildings at 15 m, shown as the blue squares, each providing an 80\textdegree\ field of view.
Communication is operated at a bandwidth of 200 MHz centered at a frequency of 28.6 GHz.

\begin{oonecolumn}
\begin{figure}
    \centering
    \includegraphics[width=0.5\linewidth]{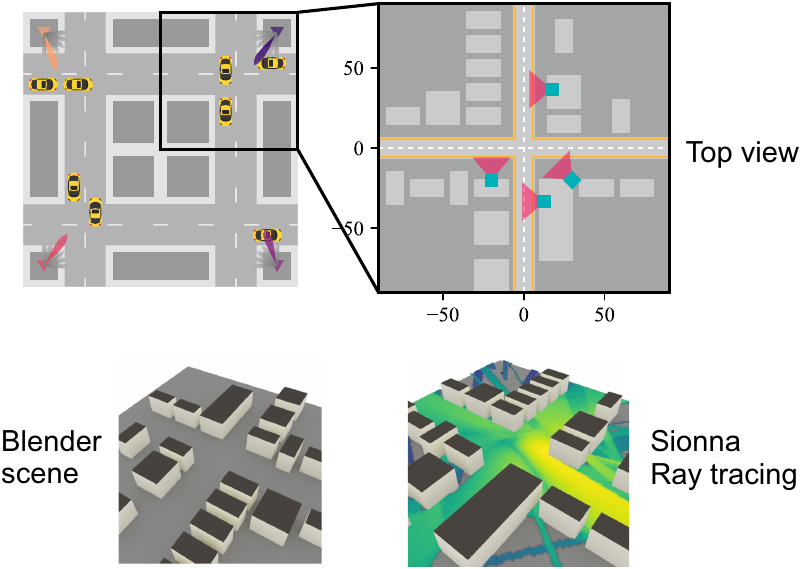}
    \caption{Simulation setting showing the top view of the considered V2I network. A Blender scene is processed by Sionna to render wireless CSI using ray tracing.}
    \label{fig:mm_paper_simulation_setup}
\end{figure}
\end{oonecolumn}
\begin{ttwocolumn}
\begin{figure}
    \centering
    \includegraphics[width=0.9\linewidth]{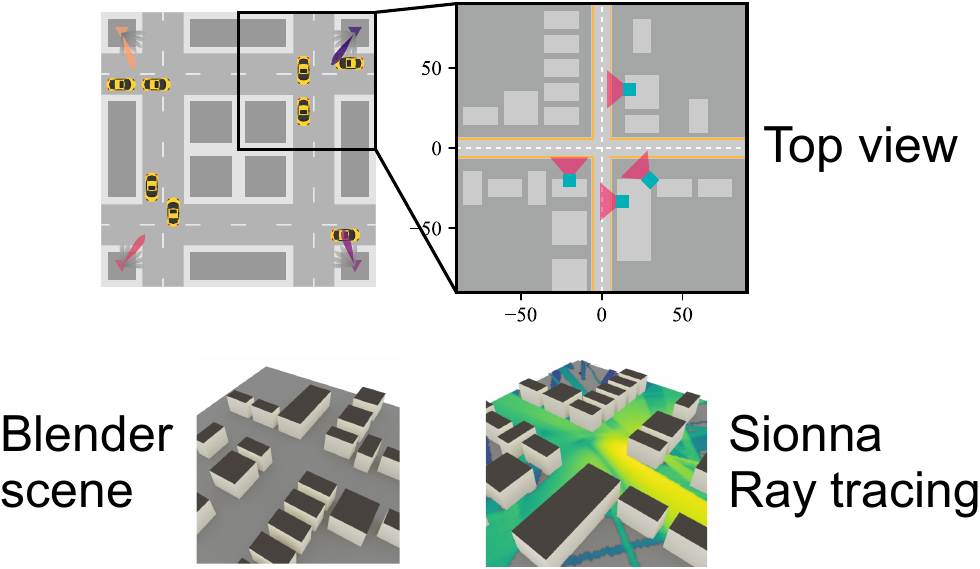}
    \caption{Simulation setting showing the top view of the considered V2I network. A Blender scene is processed by Sionna to render wireless CSI using ray tracing.}
    \label{fig:mm_paper_simulation_setup}
\end{figure}
\end{ttwocolumn}

To test our multimodal sensing framework, we collect 3,000 images from the scene (depicting around 5,000 vehicles as processed by YOLOv7) and 35,000 wireless channels.
The charting function $\chartingfunction_{\chartingfunctionparams_\agentindex}$ is parametrized as a \gls{mlp} with 5 hidden layers of (1024, 512, 256, 128, 64) neurons followed by ReLU activations, and trained with a learning rate of 0.01.
For the \gls{marl}, each encoder is made of two equivariant layers comprising 64 and 32 neurons, while the message passing and update functions are equivariant layers with 64 neurons, all followed by ReLU activations.
The policy head is an equivariant layer that receives the final embedding and outputs a distribution over the codebook, while the value head is another equivariant layer with 64 neurons that outputs the estimated value.
We train the \gls{ppo} agents with a learning rate of 0.0001.

\subsection{Discussion and Key Insights}
%
\subsubsection{Impact of Multimodal Sensing}
We start by studying the performance of our \textbf{Proposed} self-supervised multimodal sensing method.
Since the same framework is used by all agents, we show results for one agent where the ground-truth vehicle locations are shown in Fig.~\ref{fig:mm_paper_locations_gt}.
We compare with three baselines:
\begin{itemize}[nolistsep, leftmargin=*]
    \item \textbf{Baseline}: a typical (\gls{csi} only) channel charting scheme followed by an optimal affine transform to obtain the users' locations from \gls{csi} embeddings. This is equivalent to setting $\matchingscalar = 1$ and $\ccregularizer = 0$ in \eqref{eq:mm_paper_cc_loss}.
    \item \textbf{Supervised}: a fully supervised fingerprinting baseline where we assume the ground-truth matching locations of wireless channels are known by the agent.
    \item \textbf{Proposed (Partial)}: an ablation benchmark where the agent only accesses images from limited sections of the environments. We implement this approach by removing the left-most camera in Fig.\ref{fig:mm_paper_simulation_setup}, which withholds the agent from observing vehicle locations in the yellow parts of Fig.~\ref{fig:mm_paper_locations_gt}. The idea of this baseline is to examine the out-of-distribution generalizability of our method, testing whether the agent can correctly align the partial matching sections of the two modalities, and recover the locations of the unmatched \gls{csi} samples with no supervision.
\end{itemize}

\begin{table*}
    \caption{Latent space and localization metrics for different methods}
    \centering
    \begin{tabular}{lcccccc}
    \toprule
    & & \multicolumn{3}{c}{Latent Space Metrics} & \multicolumn{2}{c}{Localization error [m]} \\\cmidrule(lr){3-5} \cmidrule(lr){6-7}
    Method & Figure & CT $\rbrk{\uparrow}$ & TW $\rbrk{\uparrow}$ & KS $\rbrk{\downarrow}$ & Mean $\rbrk{\downarrow}$ & 95\textsuperscript{th} percentile $\rbrk{\downarrow}$ \\
    \midrule
    Supervised & --- & 0.999935 & 0.999939 & 0.013794 & 0.368543 & 0.978475 \\
    Baseline & Fig.~\ref{fig:mm_paper_baseline_chart} & 0.993567 & 0.995233 & 0.118514 & 3.916458 & 9.258683 \\
    Proposed & Fig.~\ref{fig:mm_paper_proposed_chart} & 0.999407 & 0.999515 & 0.032431 & 1.442052 & 3.366604 \\
    Proposed (Partial) & Fig.~\ref{fig:mm_paper_proposed_chart_trunc} & 0.996655 & 0.996944 & 0.073794 & 2.092158 & 5.983181 \\
    \bottomrule
    \end{tabular}
    \label{tab:mm_paper_cc_metrics}
\end{table*}

\begin{figure*}[!t]
    \centering
    \subfloat[Ground-truth locations.]{
        \includegraphics{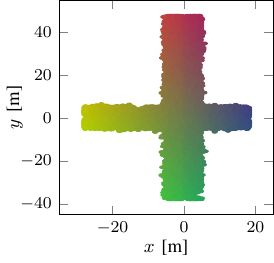}
        \label{fig:mm_paper_locations_gt}}
    \hfill
    \subfloat[Baseline.]{
        \includegraphics{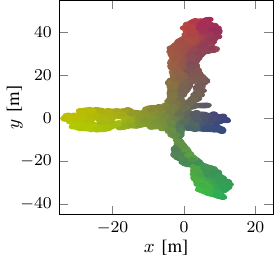}
        \label{fig:mm_paper_baseline_chart}}
    \hfill
    \subfloat[Proposed.]{
        \includegraphics{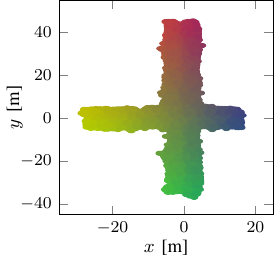}
        \label{fig:mm_paper_proposed_chart}}
    \caption{Comparison between different sensing frameworks.}
    \label{fig:mm_paper_charts}
\end{figure*}
\begin{figure}[!t]
    \centering
    \includegraphics{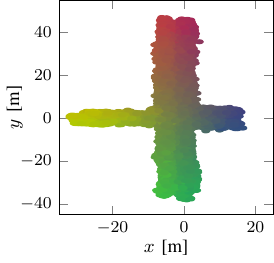}
    \caption{Localization performance of the proposed method with partial alignment.}
    \label{fig:mm_paper_proposed_chart_trunc}
\end{figure}

Figs.~\ref{fig:mm_paper_baseline_chart} and~\ref{fig:mm_paper_proposed_chart} show the obtained locations from \gls{csi} samples, corresponding to ground-truth locations gradient colored in Fig.~\ref{fig:mm_paper_locations_gt}, for the baseline and our proposed approach.
Clearly, our approach significantly outperforms the baseline, signifying that the agent correctly matches the modalities without supervision allowing for precise localization.
To quantify these results, Table~\ref{tab:mm_paper_cc_metrics} presents the continuity (CT), trustworthiness (TW) and Kruskal stress (KS) metrics, typically used the channel charting literature~\cite{stephan2024angle}, while also showing the mean and 95\textsuperscript{th} percentile localization error.
First, we notice that all approaches perform well in terms of obtained latent space quality, noting that our proposed approach achieves lower KS than the baseline underscoring that the obtained embeddings mirror ground-truth locations more accurately (since KS measures the discrepancy between pairwise distances in original and latent spaces, determining the global structure preservation of the embeddings).
Furthermore, our proposed method attains an average localization error of 1.44 m, 64\% less than the baseline of 3.91 m, while the supervised approach guarantees a 0.36 m mean error.

\subsubsection{Impact of Multimodal Alignment}
To gain more insight on the proposed multimodal alignment, Fig.~\ref{fig:mm_paper_proposed_chart_trunc} illustrates the obtained locations from our model trained under partial alignment.
We observe that our approach is almost unaffected by such disruptions, and generalizes well even with partial sections of the environment completely missing from the image modality.
In fact, this partial alignment approach achieves a 95\textsuperscript{th} percentile error of 5.98 m, 36\% better than the baseline.

Fig.~\ref{fig:mm_paper_loc_cdf} plots the cumulative distribution function (CDF) of the localization error for the different proposed and baseline methods.
We again remark that our approach, even under partial alignment, realizes substantially better sensing results than the baseline.
We also notice the impact of the missing modality on the localization error with the gap between the proposed (green) and the partial (blue) curves.
This is due to the misaligned \gls{csi} samples which the agent incorrectly localizes, underscoring the importance of the cross-modal distillation term in our loss function \eqref{eq:mm_paper_cc_loss} to accurately ground the \gls{csi} embeddings with their aligned images.

\begin{oonecolumn}
\begin{figure}[!t]
    \begin{minipage}[t]{.45\linewidth}
        \centering
        \includegraphics{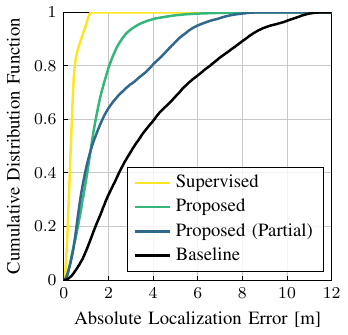}
        \caption{Empirical CDF of localization error for different methods.}
        \label{fig:mm_paper_loc_cdf}
    \end{minipage}
    \hfill
    \begin{minipage}[t]{.45\linewidth}
        \centering
        \includegraphics{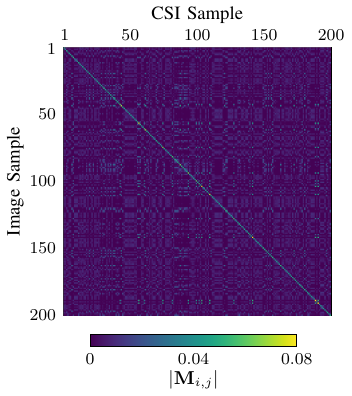}
        \caption{First rows and columns of the obtained matching matrix.}
        \label{fig:mm_paper_matching}
    \end{minipage}
\end{figure}
\end{oonecolumn}

\begin{ttwocolumn}
\begin{figure*}[!t]
    \begin{minipage}[t]{.5\linewidth}
        \centering
        \includegraphics{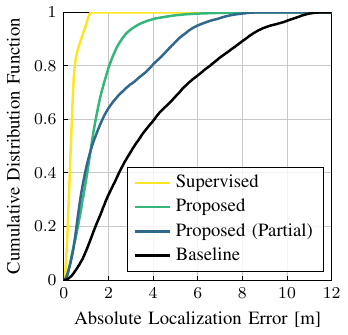}
        \caption{Empirical CDF of localization error for different methods.}
        \label{fig:mm_paper_loc_cdf}
    \end{minipage}
    \hfill
    \begin{minipage}[t]{.5\linewidth}
        \centering
        \includegraphics{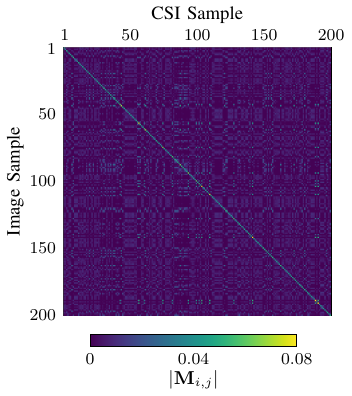}
        \caption{First rows and columns of the obtained matching matrix.}
        \label{fig:mm_paper_matching}
    \end{minipage}
\end{figure*}
\end{ttwocolumn}

\begin{ttwocolumn}
\begin{figure}[!t]
\centering
{\setlength{\fboxrule}{1pt}\setlength\fboxsep{-0.5pt}\fcolorbox{purple}{white}{
\subfloat[$\left\lvert\dataset_\agentindex^{\text{csi}}\right\rvert = 3{,}000$ samples.]{
    \includegraphics{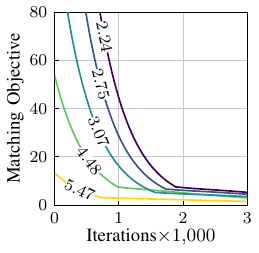}
    \label{fig:mm_paper_alignment_convergence0}}
\hspace{-1em}
\subfloat[$\left\lvert\dataset_\agentindex^{\text{csi}}\right\rvert = 5{,}000$ samples.]{
    \includegraphics{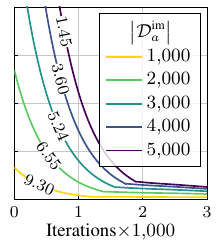}
    \label{fig:mm_paper_alignment_convergence1}}
}}
\caption{\revise{Convergence of our proposed multimodal alignment procedure for different data set sizes. Each curve is also labeled by its corresponding mean localization error in meters.}}
\label{fig:mm_paper_alignment_convergence}
\end{figure}
\end{ttwocolumn}

\begin{oonecolumn}
\begin{figure}[!t]
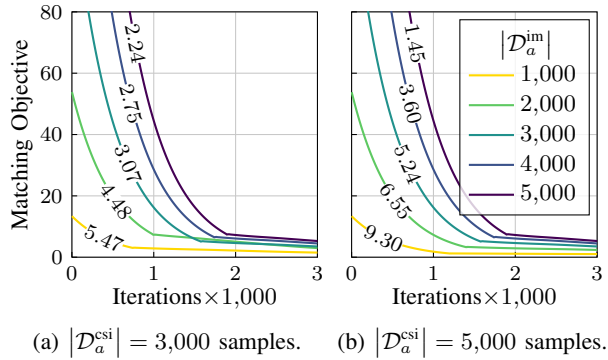

\centering
\subfloat[$\left\lvert\dataset_\agentindex^{\text{csi}}\right\rvert = 3{,}000$ samples.]{
    \includegraphics{revised_figures/optconv0.pdf}
    \label{fig:mm_paper_alignment_convergence0}}
\hspace{-1em}
\subfloat[$\left\lvert\dataset_\agentindex^{\text{csi}}\right\rvert = 5{,}000$ samples.]{
    \includegraphics{revised_figures/optconv1.pdf}
    \label{fig:mm_paper_alignment_convergence1}}
\caption{Convergence of our proposed multimodal alignment procedure for different data set sizes. Each curve is also labeled by its corresponding mean localization error in meters.}
\label{fig:mm_paper_alignment_convergence}
\end{figure}
\end{oonecolumn}

Moreover, Fig.~\ref{fig:mm_paper_matching} depicts the first 200 rows and columns of our obtained multimodal matching matrix.
For this experiment, the samples are arranged such that the first image location corresponds to the first \gls{csi} sample, etc.
We notice that our obtained matching matrix is strongly diagonal, eliciting a correct matching between wireless \gls{csi} and images, and confirming all previous localization results.

Fig.~\ref{fig:mm_paper_alignment_convergence} presents the convergence of our solution to the multimodal alignment problem (13), for different \gls{csi} and image data set sizes.
Furthermore, we report the average localization error for each tuple $\rbrk{\numsamples_\agentindex^{\text{csi}},\numsamples_\agentindex^{\text{im}}}$ obtained by training the multimodal sensing model with the corresponding scaling factor and matching matrix $\rbrk{\matchingscalar_\agentindex, \matchingmatrix_\agentindex}$.
First, as shown in Fig.~\ref{fig:mm_paper_alignment_convergence0} for the case of $\numsamples_\agentindex^{\text{csi}}=3{,}000$, we notice that with $1{,}000$ image samples, our algorithm converges quickly after around $500$ iterations, while achieving a localization error of $5.4$m.
This convergence is delayed by four times when $\numsamples_\agentindex^{\text{im}}$ increases to $5{,}000$ that yields a localization error of $2.2$m, less than half of the former after $2000$ iterations.
When we increase $\numsamples_\agentindex^{\text{csi}}$ to $5{,}000$ samples in Fig.~\ref{fig:mm_paper_alignment_convergence1}, we remark that the algorithm requires around $1.5$ times more steps to converge due to the increased dimensionality of the problem, for 1-3{,}000 images, however achieves comparably higher localization errors.
Interestingly, the necessary number of image samples to achieve a competitive localization performance must increase similarly to \gls{csi}.
For instance, our algorithm marks a $2.2$m error when $\numsamples_\agentindex^{\text{csi}}=3{,}000$ and $\numsamples_\agentindex^{\text{im}}=5{,}000$, in contrast to a $5.2$m error when the sample sizes are flipped.
This implies that images are more important to our localization algorithm than \gls{csi}, since they provide the grounding labels that allow the extraction of the user locations from \gls{csi} embeddings.
The best overall performance is achieved with $5{,}000$ images and channel samples, yielding a mean localization error of $1.4$m.

Besides, we investigate the sensitivity of our multimodal alignment to timing offsets between the camera and \gls{csi} estimate hardware.
As such, we delay the image sampling frequency at a fixed offset from that of the wireless channel, and execute our multimodal sensing algorithm such that \gls{csi} is now aligned with delayed images.
We report the empirical localization error distribution in Fig~\ref{fig:mm_paper_timing_offset}, while illustrating the corresponding five number summary.
We remark that a low offset of $10$ms does not affect our algorithm, where the localization error is concentrated around $1$-$2$m, achieving a median error of $1.5$m.
However, this distribution turns into a bimodal distribution at a $20$ms offset, where the errors are almost equally spread around the modes at $1$ and $3.5$m, with an interquartile range of $2.5$m, denoting a significant impact of offset between the modalities.
Further, a $30$ ms offset shifts almost all the localization errors to become concentrated around $3.7$m with a narrower interquartile range of $0.8$m.
This indicates that further work on multimodal sensing ought to explicitly account for timing offsets between the different modality hardware within the alignment procedure.

\begin{ttwocolumn}
\begin{figure*}[!t]
    \begin{minipage}[b]{.5\linewidth}
            \centering{\setlength{\fboxrule}{1pt}\setlength{\fboxsep}{5pt}\fcolorbox{purple}{white}{
        \includegraphics[width=0.75\linewidth]{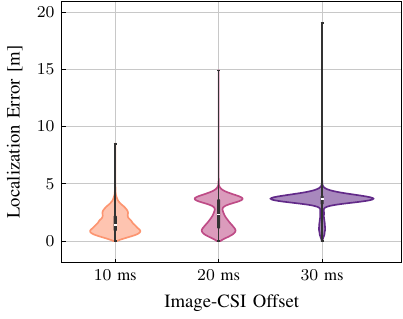}}}
        \caption{\revise{Violin plot showing the distribution of localization errors under different timing offsets between the camera and CSI acquisition modules.}}
        \label{fig:mm_paper_timing_offset}
    \end{minipage}
    \hfill
    \begin{minipage}[b]{.5\linewidth}
        \centering
        \includegraphics[width=0.9\linewidth]{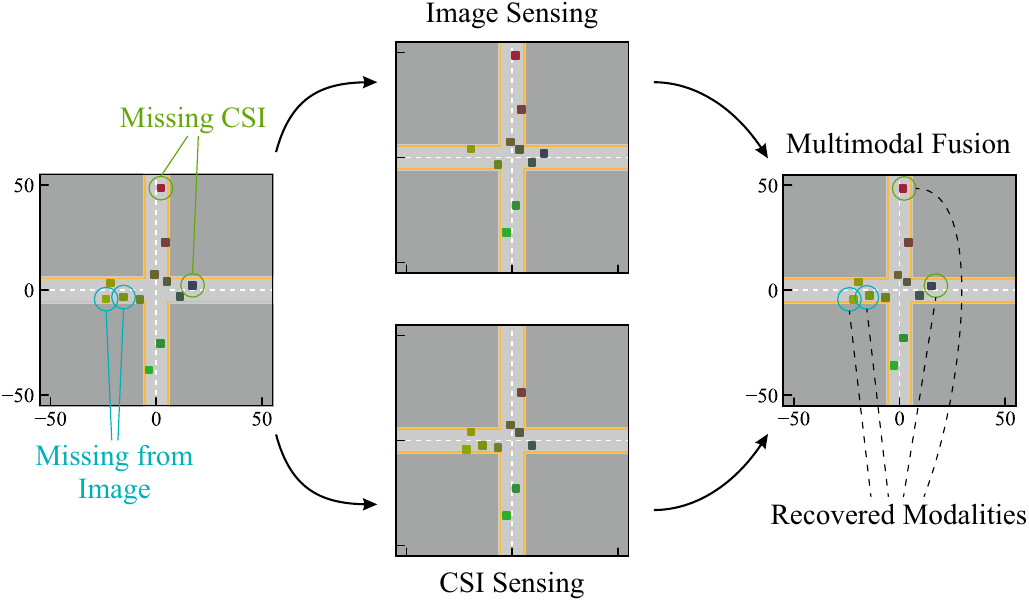}
        \caption{Proposed crossmodal imputation on an environment with missing data from each modality.}
        \label{fig:mm_paper_proposed_imputation}
    \end{minipage}
\end{figure*}
\end{ttwocolumn}

\begin{ttwocolumn}
\begin{figure*}[!t]
    \begin{minipage}[t]{.5\linewidth}
        \centering
        \includegraphics{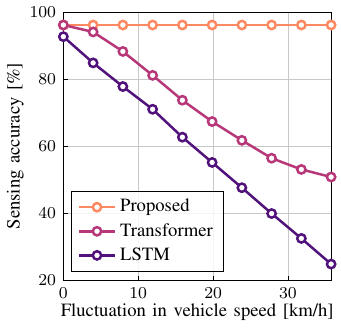}
        \caption{Out-of-distribution average sensing accuracy for different imputation models.}
        \label{fig:mm_paper_ood_forecast}
    \end{minipage}
    \hfill
    \begin{minipage}[t]{.5\linewidth}
        \centering
        \includegraphics{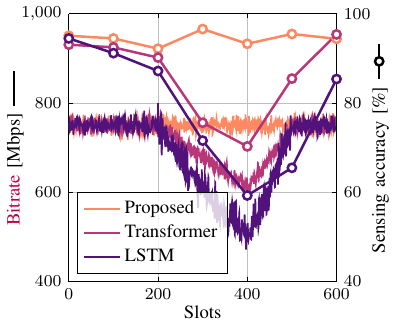}
        \caption{Rate and sensing accuracy for different imputation models.}
        \label{fig:mm_paper_ood_sensing_rate}
    \end{minipage}
\end{figure*}
\end{ttwocolumn}

\subsubsection{Impact of Crossmodal Imputation}
We now examine our proposed crossmodal imputation framework, allowing the agent to recover a missing modality.
Fig.~\ref{fig:mm_paper_proposed_imputation} shows a particular sample of ground-truth vehicle locations from the environment (left).
As indicated, two users are not observed by the camera images (blue) and the agent has no \gls{csi} estimates for two other users (green).
In the middle part, the agent processes the images and \gls{csi} to extract the aligned vehicles' sensing data.
Finally, as shown on the right, the agent can perform localization for all users by fusing both modalities, hence acquiring a full situational awareness, even with missing modalities.

\begin{oonecolumn}
\begin{figure}[!t]
\begin{minipage}[b]{.45\linewidth}
    \centering
    \includegraphics[width=0.8\linewidth]{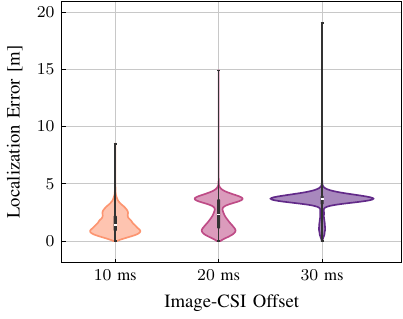}
    \caption{Violin plot showing the distribution of localization errors under different timing offsets between the camera and CSI acquisition modules.}
    \label{fig:mm_paper_timing_offset}
\end{minipage}
\hfill
\begin{minipage}[b]{.45\linewidth}
    \centering
    \includegraphics[width=\linewidth]{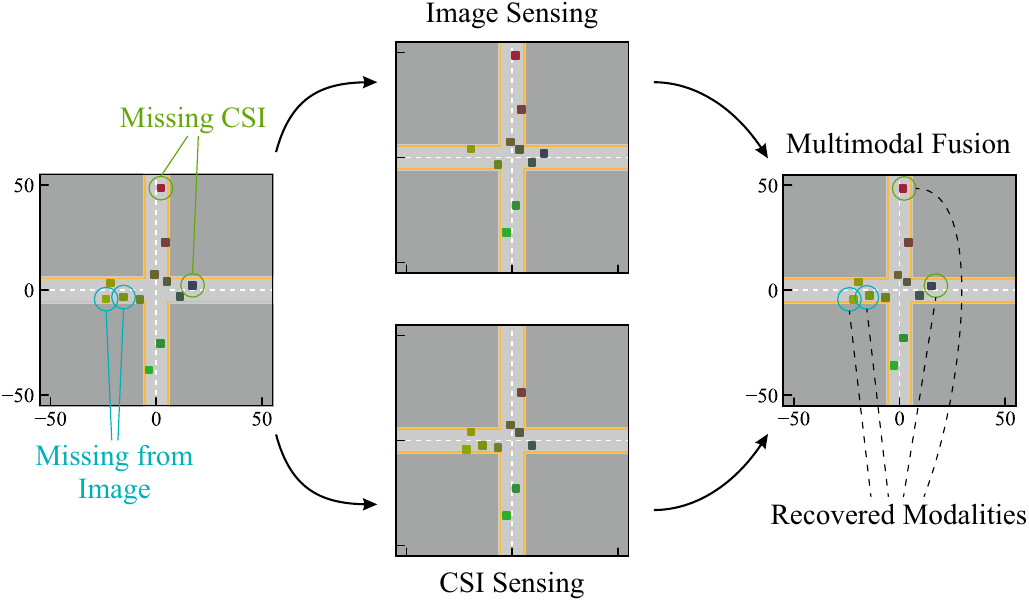}
    \caption{Proposed crossmodal imputation on an environment with missing data from each modality.}
    \label{fig:mm_paper_proposed_imputation}
\end{minipage}
\end{figure}
\end{oonecolumn}

We now compare our proposed method with two supervised auto-regressive baselines from the literature:
\begin{itemize}[nolistsep, leftmargin=*]
    \item \textbf{Transformer}: a transformer architecture is trained to impute missing sensing data from a time-series, following~\cite{wu2020deep}. The transformer's architecture consists of encoder and decoder layers. First the input series is embedded using a linear layer to a $512$ feature vector, followed by element-wise addition of a positional encoding vector to encode the series' order. This vector is then passed through four identical encoder layers: each layer consists of $8$ attention heads, a normalization layer, a $1024$ neuron fully-connected layer and another normalization layer. The decoder mirrors the encoder’s architecture with the additional cross-attention layer over the encoder's output. The output of the last decoder layer is fed to a linear layer to form the prediction target. We train the model in a supervised manner with a mean squared error loss, a batch size of $32$ and $0.0001$ learning rate, to infer a user's $5$ future locations given its last $20$ locations.
    \item \textbf{LSTM}: a standard \gls{lstm}~\cite{gers99lstm} for time-series forecasting is trained for the same task. The model employs three stacked recurrent layers, each with $64$ neurons. The hidden state of the final layer is projected through a linear linear that outputs the prediction vector of size $5$. We also train the LSTM in a supervised manner with a mean squared error loss, a batch size of $32$ and $0.005$ learning rate.
\end{itemize}

\begin{oonecolumn}
\begin{figure}[!t]
    \begin{minipage}[t]{.45\linewidth}
        \centering
        \includegraphics{sim-figures/ood-forecast.pdf}
        \caption{Out-of-distribution average sensing accuracy for different imputation models.}
        \label{fig:mm_paper_ood_forecast}
    \end{minipage}
    \hfill
    \begin{minipage}[t]{.45\linewidth}
        \centering
        \includegraphics{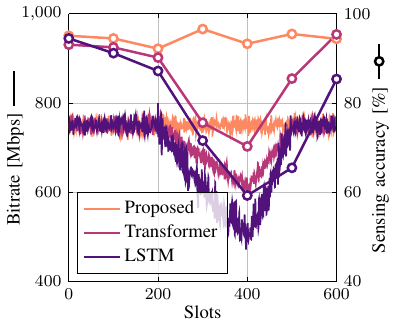}
        \caption{Rate and sensing accuracy for different imputation models.}
        \label{fig:mm_paper_ood_sensing_rate}
    \end{minipage}
\end{figure}
\end{oonecolumn}
To train these two auto-regressive models, we collect time-series sensing data from our environment and train them in a supervised learning manner.
The training data is collected while the vehicles navigate our environment at an average speed of 40 km/h.
However, during testing we introduce random fluctuations in the speed which creates a distribution shift of training data.

Fig.~\ref{fig:mm_paper_ood_forecast} displays the sensing accuracy of the different models when tested on data corresponding to varying speed fluctuations.
With no fluctuations in vehicles’ speed, all models provide the same accuracy with 3-4\% differences.
The weakness of the baselines appears with their performance under out-of-distribution settings, in which, they yield an accuracy below 60\% and 70\% for the LSTM and transformer, respectively, at around 20 km/h fluctuation in speed.
In contrast, our cross-modal recovery method maintains the same sensing accuracy around 95\% since it bases its prediction on the counterpart of each modality in the same timeslot (as shown in Fig.~\ref{fig:mm_paper_proposed_imputation}), while the auto-regressive baselines cannot extrapolate well to unseen sensing data patterns.

Besides, since our \gls{marl} training relies on fine-grained sensing data to exploit symmetries, we inspect the impact of different imputation models on its real-time performance.
Fig.~\ref{fig:mm_paper_ood_sensing_rate} plots the bitrate and running average of the sensing accuracy for a particular navigating vehicle, whose speed starts fluctuating between slots 200 and 400.
We notice that both imputation baselines suffer from a rate decrease of 20\% and 30\% respectively for the transformer and \gls{lstm} under unseen data patterns, while our approach adapts seamlessly with its unaffected sensing accuracy, guaranteeing high-rates for the user.

\subsubsection{Impact of Equivariant MARL}
We now analyze the performance of our proposed equivariant \gls{marl} training scheme.
We compare with four baselines:
\begin{itemize}[nolistsep, leftmargin=*]
    \item \textbf{Baseline 1}: we implement the proposed \gls{gnn} used standard neural network layers instead of equivariant layers to study the impact of accounting for symmetries in the environment. Note that this benchmark utilizes the estimated vehicles' locations obtained from multimodal data as input, but trains a standard \gls{marl} policy.
    \item \textbf{Baseline 2}: a data augmentation baseline where a standard non-equivariant \gls{gnn} is used for the \gls{marl} policy, and the agents' networks are also trained with \gls{ppo}. At each policy update step, the collected environment data batch $\cbrk{\estimatedstate^t_\agentindex, \beam^t_\agentindex, \reward^t_\agentindex}_{t=1}^B$ are synthetically augmented to $\bigcup_{g \in \group} \cbrk{\groupaction_g \sbrk{\estimatedstate^t_\agentindex}, \groupactiontwo_g^\state \sbrk{\beam^t_\agentindex}, \reward^t_\agentindex}_{t=1}^B$ applying group transformations to the experienced samples and actions. Hence for a batch size $B$, this baseline amount to training a standard network with $\left\lvert \group \right\rvert B$ samples, and serves for comparison to our proposed network that satisfies the equivariance property through its architecture.
    \item \textbf{Baseline 3}: a centralized controller baseline, where a single \glsentryshort{rl} agent observes all the states and takes a decision for all \glspl{bs}.
    \item \textbf{Baseline 4}: a baseline with no communication, where each \gls{bs} agent acts independently given its own observation as a single \glsentryshort{rl} agent. This benchmark evaluates the impact of communication on the coordination between the agents.
\end{itemize}
\begin{oonecolumn}
\begin{figure*}
    \centering
    \subfloat[$\left\lvert\beamset_\agentindex\right\rvert=64,\;\; \left\lvert\userset_\agentindex\right\rvert=4$]{
        \includegraphics{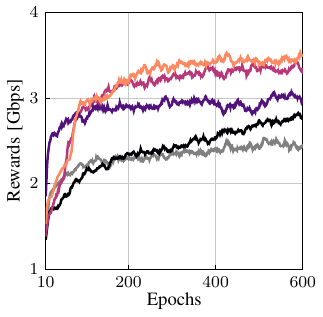}
        \label{fig:mm_paper_marl1}}
    \hfill
    \subfloat[$\left\lvert\beamset_\agentindex\right\rvert=256,\;\; \left\lvert\userset_\agentindex\right\rvert=4$]{
        \includegraphics{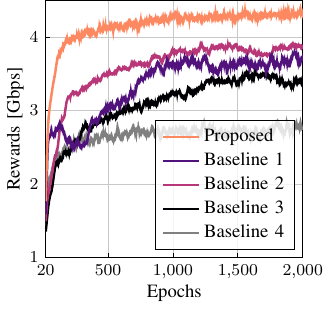}
        \label{fig:mm_paper_marl2}}
    \hfill
    \subfloat[$\left\lvert\beamset_\agentindex\right\rvert=64,\;\; \left\lvert\userset_\agentindex\right\rvert=8$]{
        \includegraphics{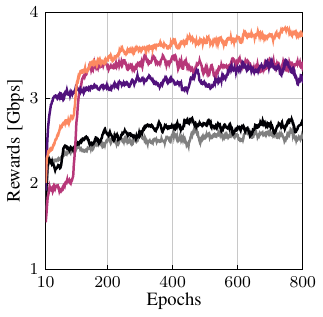}
        \label{fig:mm_paper_marl3}}
    \caption{Convergence of \gls{marl} algorithms for different system parameters affecting the state and action sets.}
    \label{fig:mm_paper_marl}
\end{figure*}
\end{oonecolumn}
\begin{ttwocolumn}
\begin{figure*}
    \centering
    \subfloat[$\left\lvert\beamset_\agentindex\right\rvert=64,\;\; \left\lvert\userset_\agentindex\right\rvert=4$]{
        \includegraphics{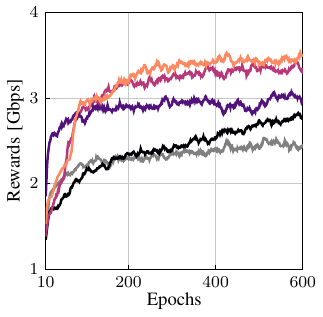}
        \label{fig:mm_paper_marl1}}
    \hfill
    \subfloat[$\left\lvert\beamset_\agentindex\right\rvert=256,\;\; \left\lvert\userset_\agentindex\right\rvert=4$]{
        \includegraphics{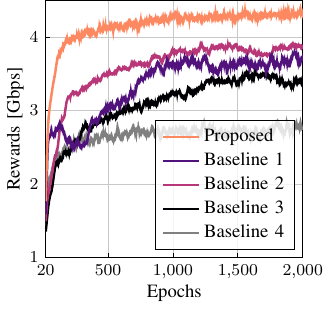}
        \label{fig:mm_paper_marl2}}
    \hfill
    \subfloat[$\left\lvert\beamset_\agentindex\right\rvert=64,\;\; \left\lvert\userset_\agentindex\right\rvert=8$]{
        \includegraphics{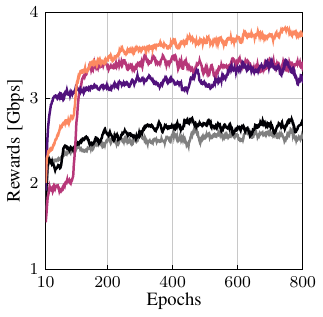}
        \label{fig:mm_paper_marl3}}
    \caption{Convergence of \gls{marl} algorithms for different system parameters affecting the state and action sets.}
    \label{fig:mm_paper_marl}
\end{figure*}
\end{ttwocolumn}
\begin{ttwocolumn}
\begin{figure}[!t]
    \centering
    {\setlength{\fboxrule}{1pt}\setlength\fboxsep{-0.5pt}\fcolorbox{purple}{white}{
    \includegraphics[width=.6\linewidth]{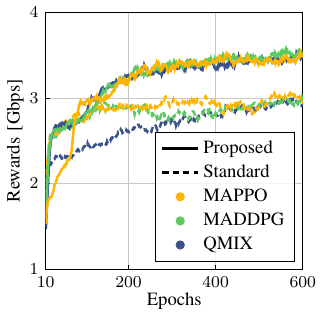}}}
    \caption{\revise{Impact of different MARL training schemes using our proposed equivariant policy network and a standard non-equivariant network $\rbrk{\left\lvert\beamset_\agentindex\right\rvert=64,\;\; \left\lvert\userset_\agentindex\right\rvert=4}$.}}
    \label{fig:mm_paper_marl_ablation}
\end{figure}
\end{ttwocolumn}
Fig.~\ref{fig:mm_paper_marl1} shows the convergence of all \gls{marl} algorithms in the case where $\left\lvert\beamset_\agentindex\right\rvert=64$ and $\left\lvert\userset_\agentindex\right\rvert=4$ users.
We notice that our equivariant method outperforms a typical non-equivariant network (baseline 1) by $20\%$, where the baseline achieves rates of $3$ Gbps after $600$ training epochs, which is attained by our approach after $100$ epochs, and increased to more than $3.5$ Gbps after convergence.
Further, in this setting, the augmentation approach (baseline 2) performs similarly to our algorithm until convergence, where our method achieves $5\%$ higher rewards, noting that the benchmark necessitates hand-crafted processing steps to augment the agents' training data, whereas our network satisfies the equivariance property by design.
The single agent method (baseline 3) is the slowest method to converge due to the large state and action sets, and is outperformed by our method by $30\%$ higher rewards, whereas the benchmark with no communication (baseline 4) performs the worst attaining a sum rate less than $2.5$ Gbps, underlying the importance of communication in distributed coordination problems.

\subsubsection{Impact of Scalability}
To study the scalability of our proposed method, we examine its robustness with respect to the local state and action sets.
Fig.~\ref{fig:mm_paper_marl2} plots the convergence of all algorithms when the number of antennas per \gls{bs} is increased so that $\left\lvert\beamset_\agentindex\right\rvert=256$.
First, we note that our proposed approach maintains its swift convergence in this setting, achieving a return higher than $4$ Gbps in less than $300$ epochs.
This underscores $50\%$ and $30\%$ gains over the non-equivariant and augmentation baselines respectively, in the low data regime.
To attain the same data rates, the augmentation baseline requires around $1500$ epochs, resulting in more than $80\%$ savings in terms of training data for our method.
Furthermore, we observe that in this setting, the augmentation baseline can no longer match the performance of our proposed technique that achieves $12\%$ higher returns of $4.3$ Gbps, and matches the rewards attained by the baseline with $70\%$ less training epochs.
The non-equivariant baseline performs only $4\%$ worse than the augmentation approach in terms of rewards, but requires around $300$ more training rounds to converge, underscoring the importance of inducing symmetries into the policy training, particularly under high-dimensional action sets.
The single agent and no communication baselines are again outperformed by our proposed method by $30\%$ and $55\%$ reward gains after convergence respectively.

\begin{oonecolumn}
\begin{figure}[!t]
    \centering
    \includegraphics[width=.3\linewidth]{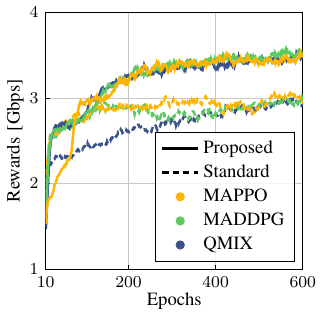}
    \caption{Impact of different MARL training schemes using our proposed equivariant policy network and a standard non-equivariant network $\rbrk{\left\lvert\beamset_\agentindex\right\rvert=64,\;\; \left\lvert\userset_\agentindex\right\rvert=4}$.}
    \label{fig:mm_paper_marl_ablation}
\end{figure}
\end{oonecolumn}

\begin{ttwocolumn}
\begin{figure*}
    \centering
    {\setlength{\fboxrule}{1pt}\setlength\fboxsep{-0.5pt}\fcolorbox{purple}{white}{
    \subfloat[Road and RSU angle perturbations breaking exact $C_4$ symmetry.]{
        \includegraphics[width=0.34\textwidth, valign=b]{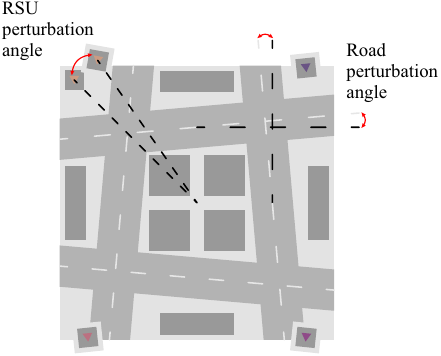}
        \label{fig:mm_paper_partial_symmetry_def}}
    \hfill
    \subfloat[Impact of road angle asymmetry.]{
        \includegraphics[valign=b]{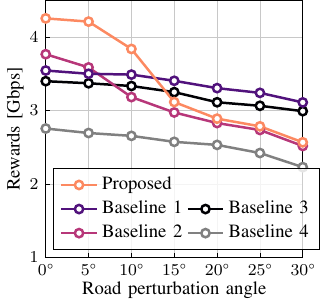}
        \label{fig:mm_paper_partial_symmetry_road}}
    \hfill
    \subfloat[Impact of RSU angle asymmetry.]{
        \includegraphics[valign=b]{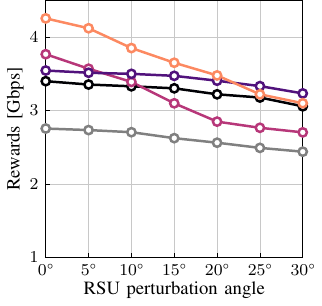}
        \label{fig:mm_paper_partial_symmetry_rsu}}
    }}
    \caption{\revise{Comparison between different MARL algorithms under partial symmetry ($\left\lvert\beamset_\agentindex\right\rvert=256,\;\; \left\lvert\userset_\agentindex\right\rvert=4$).}}
    \label{fig:mm_paper_partial_symmetry_sim}
\end{figure*}
\end{ttwocolumn}
\begin{ttwocolumn}
\begin{figure}[!t]
    \centering
    {\setlength{\fboxrule}{1pt}\setlength\fboxsep{3pt}\fcolorbox{purple}{white}{
    \includegraphics[width=.95\linewidth]{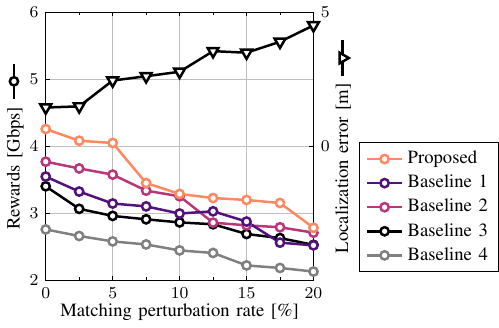}}}
    \caption{\revise{Impact of localization error on the performance of different MARL algorithms ($\left\lvert\beamset_\agentindex\right\rvert=256,\;\; \left\lvert\userset_\agentindex\right\rvert=4$). We inject such errors by randomly flipping some of the rows of the optimized matching matrix. Note the different scaling on the two vertical axes.}}
    \label{fig:mm_paper_injected_noise}
\end{figure}
\end{ttwocolumn}
Fig.~\ref{fig:mm_paper_marl3} shows the convergence of different algorithms when more vehicles are present in the network, i.e., $\left\lvert\userset_\agentindex\right\rvert=8$ users.
In this setting, our approach still outperforms the data augmentation and non-equivariant baselines by $15\%$ higher rewards, underscoring its scalable behavior.
Both baselines perform very similarly in this setting and achieve a sum rate around $4.3$ Gbps, highlighting the weakness of data augmentations in reaching strong performances under large observation spaces, while the single agent approach can only obtain a return of $2.7$ Gbps.
Altogether, Fig.~\ref{fig:mm_paper_marl} shows the importance of identifying the underlying symmetries in the optimization of a wireless network, achieving important gains in all settings.

To further examine our proposed \gls{marl} policy, we study its sensitivity to PPO gradient updates, by comparing it with two other popular \gls{marl} variants, namely MADDPG~\cite{lowe2017multi} and QMIX~\cite{rashid2020monotonic}.
We plot the reward convergence for all three variants in Fig.~\ref{fig:mm_paper_marl_ablation}, while comparing our equivariant policy to a standard non-equivariant policy (baseline 1) in the case $\left\lvert\beamset_\agentindex\right\rvert=64$ and $\left\lvert\userset_\agentindex\right\rvert=4$.
We notice that regardless of the \gls{marl} training scheme, equipping agents with our proposed equivariant network guarantees stable convergence, as different policy improvement methods persistently yield almost the same convergence behavior (solid curves).
This is in contrast to the case where agents train non-equivariant policy networks, that are clearly more sensitive to the \gls{marl} updates (dashed curves).
For instance, we notice that MAPPO and MADDPG perform similarly, both converging at around $200$ epochs, however MAPPO showcases a more stable behavior than MADDPG whose performance fluctuates.
Whereas, QMIX requires around $400$ epochs to converge to similar rewards as the other baselines, a two fold decrease in sample efficiency.
Altogether, this experiment underscores the importance of infusing symmetry inductive biases to the downstream \gls{marl} training, as it stabilizes policy updates while also achieving significant reward gains.

\begin{oonecolumn}
\begin{figure*}
    \centering
    \subfloat[Road and RSU angle perturbations breaking exact $C_4$ symmetry.]{
        \includegraphics[width=0.34\textwidth, valign=b]{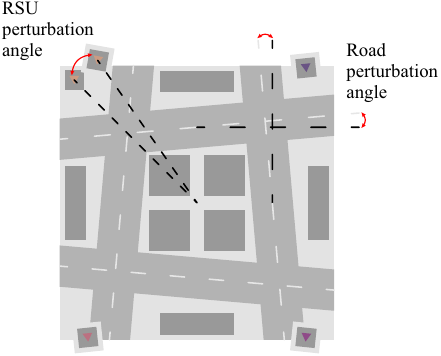}
        \label{fig:mm_paper_partial_symmetry_def}}
    \hfill
    \subfloat[Impact of road angle asymmetry.]{
        \includegraphics[valign=b]{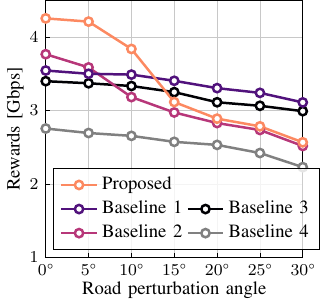}
        \label{fig:mm_paper_partial_symmetry_road}}
    \hfill
    \subfloat[Impact of RSU angle asymmetry.]{
        \includegraphics[valign=b]{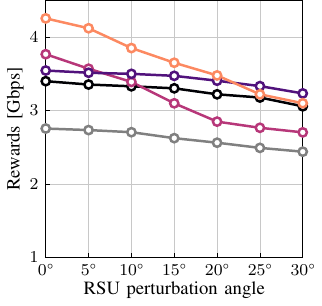}
        \label{fig:mm_paper_partial_symmetry_rsu}}
    \caption{Comparison between different MARL algorithms under partial symmetry ($\left\lvert\beamset_\agentindex\right\rvert=256,\;\; \left\lvert\userset_\agentindex\right\rvert=4$).}
    \label{fig:mm_paper_partial_symmetry_sim}
\end{figure*}
\end{oonecolumn}

\subsubsection{Impact of Partial Symmetry}
We now study the performance of our proposed equivariant \gls{marl} under imperfect rotation symmetries.
To that end, we introduce two variants of our original environment where we imbue asymmetric perturbations.
In the first experiment, we introduce a `road perturbation angle' whereby the roads of the environment are rotated in an asymmetric fashion, thus the vehicles no longer navigate by adhering to exact location symmetries with respect to the \glspl{rsu}.
Likewise, in the second experiment, we perturb the locations of the \glspl{rsu} through an `RSU perturbation angle', asymmetrically pushing the top and bottom RSUs closer to each other, breaking the exact $C_4$ symmetry, since the relative vehicle locations to each \gls{rsu} no longer transform under $90$\textdegree\ rotations.
The two asymmetrical variants are presented in Fig.~\ref*{fig:mm_paper_partial_symmetry_def}, and the aim of this experiment is to examine the performance of the different \gls{marl} benchmarks under a controlled notion of asymmetry.
We note that the performance of model-free methods (baselines 1 and 3) is only slightly affected by such perturbations since we re-train them from scratch for each road or \gls{rsu} deviation angle for a fair comparison.
For example, as shown in Fig.~\ref*{fig:mm_paper_partial_symmetry_road}, the performance of the non-equivariant method (baseline 1) degrades by $13\%$ and $9\%$ respectively when tested on road / \gls{rsu} perturbation angles of $30$\textdegree, achieving rewards of $3.1$ and $3.2$ Gbps compared to the original $3.5$ Gbps under perfect $C_4$ symmetry.
This performance loss is only due to the more intricate learnability of irregular observation patterns in the asymmetrical environments. On the other hand, our equivariant method and the data augmentation method (baseline 2) restrain policy updates to satisfy $C_4$ equivariance, and hence lose performance under more environment perturbations (since by increasing the perturbation angles, the environment no longer adheres to such equivariance).
First, under road angle deviations, our proposed algorithm maintains its strong performance achieving around $4$ Gbps rewards up to $10$\textdegree\ road perturbations, and starts losing performance to the model-free \gls{marl} after $15$\textdegree\ road angle deviation.
To the contrary, the augmentation baseline can only sustain up to $5$\textdegree\ perturbations as its performance is directly exceeded by the non-equivariant algorithm.
Hence, under such asymmetries, our proposed method can still mark $15\%$ gains compared to benchmarks, showcasing its generalizability due to the effective inductive bias we infuse to its policy training.
When the road perturbation angle exceeds $15$\textdegree, it becomes more advantageous to train a classical non-equivariant policy, since our proposed method and the augmentation baseline perform similarly around $13\%$ worse than the model-free algorithm.
Besides, as reported in Fig.~\ref*{fig:mm_paper_partial_symmetry_rsu}, when the asymmetry is due to \gls{rsu} displacement, we notice that our proposed technique performs much better and can sustain up to $20$\textdegree\ angle perturbation while still outperforming all baselines, as the rewards achieved slowly decrease from $4.3$ to around $3$ Gbps under $30$\textdegree\ deviations, less than $5\%$ worse than the non-equivariant baseline 1.
This is in contrast to the data augmentation method, whose performance matches the classical \gls{marl} after only $5$\textdegree\ perturbation, and is outperformed by $20\%$ under a $30$\textdegree\ by our equivariant method.

\begin{oonecolumn}
\begin{figure}[!t]
    \centering
    \includegraphics[width=.5\linewidth]{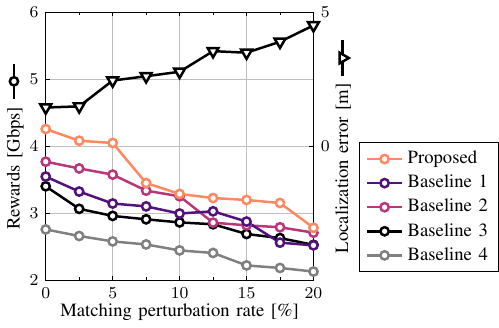}
    \caption{Impact of localization error on the performance of different MARL algorithms ($\left\lvert\beamset_\agentindex\right\rvert=256,\;\; \left\lvert\userset_\agentindex\right\rvert=4$). We inject such errors by randomly flipping some of the rows of the optimized matching matrix. Note the different scaling on the two vertical axes.}
    \label{fig:mm_paper_injected_noise}
\end{figure}
\end{oonecolumn}

In addition, while the optimal \gls{marl} policy for the rate maximization problem is equivariant under rotation symmetries of the users' positions, the ground-truth vehicle positions are unknown, and the policy takes as input the estimated locations given by the multimodal sensing algorithm developed in Section~\ref{section:mm_paper_ssl}.
However, those estimates are prone to localization error which affect the equivariance of the policy that no longer transform exactly under the rotation group action.
To examine the impact of such errors on different \gls{marl} policies, we first select a fixed percentage of the rows of the optimized matching matrix $\matchingmatrix_\agentindex^\star$ (solving the alignment problem~\eqref{eq:mm_paper_alignment_problem}) and permute them at random. 
We refer to the percentage of selected rows as the `matching perturbation rate'.
As such, since this matrix is used to train our localization model, we inject controlled errors to our position estimates, and can therefore compare different \gls{marl} benchmark under such noise, as reported in Fig.~\ref{fig:mm_paper_injected_noise}.
First, we note that the mean localization error gradually increases from $1.4$ to approach $5$m at $20\%$ perturbation rate, which is underscores its robustness to alignment errors due to the self-supervised \gls{csi} loss (first term in eq.~\eqref{eq:mm_paper_cc_loss} remains unaffected in this case).
Further, we observe that our equivariant policy performs the best compared to all other baselines under varying perturbation rates, sustaining its performance of more than $4$ Gbps rewards for noise levels less than $5\%$, achieving $30\%$ and $15\%$ gains compared to the non-equivariant and data augmentations baselines (1 and 2).
At flipping rates of $10\%$, our proposed method loses $30\%$ of its achieved rewards and reaches $3.3$ Gbps, matching the augmentation baseline, while achieving around $15\%$ gains to the classical \gls{marl} and single agent baselines (1 and 3); noting that the latter's performances quickly degrade by $20\%$ even at noise rates of $5\%$, highlighting the sensitivity of model-free methods to observation drifts.
We remark that our method can constantly sustain this performance up to a $20\%$ perturbation rate, after which it reaches rewards less than $3$ Gbps.
We also affirm the importance of communication in our partially observable setting, as baseline 4 where agents that do not communicate perform the worst and lose $30\%$ of performance under a perturbation rate of $20\%$.

\begin{figure}
    \centering
    \includegraphics{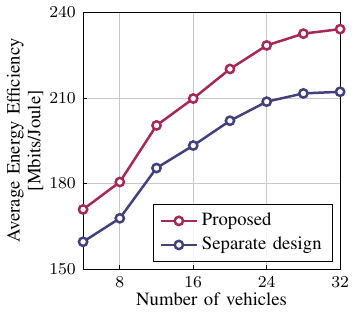}
    \caption{Energy efficiency of our multimodal system design compared to a separate design.}
    \label{fig:mm_paper_energy_efficiency}
\end{figure}

\begin{oonecolumn}
\begin{table}
    \caption{Computing costs for different modalities}
    \centering
    \begin{tabular}{ccc}
    \toprule
    & CSI model & Image model \\
    \midrule
    Number of parameters & 34.2 M & 36.9 M \\
    FLOPs & 18.18 G & 104.7 G \\
    \bottomrule
    \end{tabular}
    \label{tab:mm_paper_params}
\end{table}
\end{oonecolumn}

\subsubsection{Impact of Power and Computing Resources}
In Fig.~\ref{fig:mm_paper_energy_efficiency}, we report the energy efficiency of our \textbf{Proposed} multimodal approach where wireless \gls{csi} is re-used for sensing by aligning its features with those extracted from the images, as compared to a standalone communication and sensing design (\textbf{Separate design}), where both are performed on separate bands.
We measure the energy efficiency as the network's sum rate in bits/s divided by the system's energy consumption in Joule/s or Watts.
We notice that regardless of the number of vehicles in the system, our proposed design is always more energy efficient than the separate design.
In particular, with more than $20$ users in the network, our proposed multimodal framework achieves $15\%$ gains compared to the standalone design that sacrifices the communication bandwidth for sensing.
With $32$ users in the network, our approach realizes an energy efficiency of $235$ Mbits/Joule compared to $212$ Mbits/Joule for the separate design benchmark.
This proves the importance of recycling the wireless modality for sensing by learning to align its features with image features in a self-supervised manner.
Nevertheless, this favorable network spectral and energy efficiency comes at the price of running computationally heavy models, where the number of trainable parameters and floating point operations per second (FLOPs) for each modality's network are shown in Table~\ref{tab:mm_paper_params}.
We notice that both the image and \gls{csi} processing models have around $35$ million parameters, with varying numbers in terms of FLOPs resulting from the different model architectures.
This indicates that significant computing resources are required to achieve the proposed performance.

\begin{ttwocolumn}
\begin{table}
    \caption{Computing costs for different modalities}
    \centering
    \begin{tabular}{ccc}
    \toprule
    & CSI model & Image model \\
    \midrule
    Number of parameters & 34.2 M & 36.9 M \\
    FLOPs & 18.18 G & 104.7 G \\
    \bottomrule
    \end{tabular}
    \label{tab:mm_paper_params}
\end{table}
\end{ttwocolumn}
\section{Conclusion}
\label{section:mm_paper_conclusion}
In this work, we proposed a novel self-supervised learning framework where \gls{mmwave} operated \glspl{rsu} align their wireless \gls{csi} acquired for communication purposes, with visual data from their equipped cameras to form high situational awareness of the environment.
Our proposed framework first localizes the vehicles in images leveraging \gls{ml} based object detection, and forms an analogous latent space to users' locations for \gls{csi} using self-supervised channel clustering.
The two representation spaces are then aligned by optimizing a pairing function that matches image locations with their corresponding \gls{csi} with no external supervision, and the obtained matching is used to fine-tune the \gls{csi} model to a multimodal sensing model.
\Glspl{bs} then train an equivariant \gls{marl} policy relying on the obtained sensing data that exploits the symmetries of the \gls{v2i} environment, i.e., when the locations of the vehicles rotate, the beam directions permute between the \glspl{bs}.
Extensive simulation results corroborate the superiority of the proposed multimodal framework compared to benchmarks, while underscoring its trade-offs.
Our framework assumes fully synchronized \gls{csi} and camera image sampling, and perfect inter-\gls{rsu} message passing.
In practice, timing offsets between modalities as well as packet losses can degrade the performance of our proposed algorithm.
We plan to investigate the impact of these parameters in future work, considering systems that exhibit symmetries beyond rotations, or are not fully symmetric.
Besides, future studies will focus on privacy concerns, where downsampling visual modalities is necessary before processing, and integrating more modalities (such as LiDar) in a similar framework.

\appendix

\subsection{Proof of Lemma~\ref{lemma:mm_paper_localization}}
\label{proof:mm_paper_localization}
By direct inspection of Fig.~\ref{fig:mm_paper_image_localization}, we can write $\userpositionxy_x = d \cos\rbrk{\azimuthangle}$ and $\userpositionxy_y = d \sin\rbrk{\azimuthangle}$.
To compute $d$, notice that given the pixel's elevation angle with respect to the camera, we have $\tan\rbrk{\elevationangle} = \frac{d}{\cameraheight_\cameraindex}$.
To estimate $\elevationangle$, observe that $\frac{\tan\rbrk{\elevationangle - \elevationangle^{\textnormal{Los}}_\cameraindex}}{\tan\rbrk{\elevationangle^{\textnormal{max}}_\cameraindex}} = \frac{\pixelcoordinate_h - \frac{\imageheight_\cameraindex}{2}}{\frac{\imageheight_\cameraindex}{2}} = \frac{2 \pixelcoordinate_h - \imageheight_\cameraindex}{\imageheight_\cameraindex}$, which yields $\elevationangle = \elevationangle^{\textnormal{Los}}_\cameraindex + \tan^{-1} \rbrk{\frac{2 \pixelcoordinate_h - \imageheight_\cameraindex}{\imageheight_\cameraindex} \tan \frac{\elevationangle^{\textnormal{max}}_\cameraindex}{2}}$, and $d = \cameraheight_\cameraindex \tan \rbrk{\elevationangle^{\textnormal{Los}}_\cameraindex + \tan^{-1} \rbrk{\frac{2 \pixelcoordinate_h - \imageheight_\cameraindex}{\imageheight_\cameraindex} \tan \frac{\elevationangle^{\textnormal{max}}_\cameraindex}{2}}}$.
To get the physical locations $\userpositionxy_x$ and $\userpositionxy_y$, we still need to estimate the azimuth angle $\azimuthangle$.
Notice that $\azimuthangle$ satisfies: $\frac{\tan\rbrk{\azimuthangle - \azimuthangle^{\textnormal{Los}}_\cameraindex}}{\tan\rbrk{\frac{\azimuthangle^{\textnormal{max}}_\cameraindex}{2}}} = \frac{\pixelcoordinate_w - \frac{\imagewidth_\cameraindex}{2}}{\frac{\imagewidth_\cameraindex}{2}} = \frac{2 \pixelcoordinate_w - \imagewidth_\cameraindex}{\imagewidth_\cameraindex}$, which yields $\azimuthangle = \azimuthangle^{\textnormal{Los}}_\cameraindex + \tan^{-1} \rbrk{\frac{2 \pixelcoordinate_w - \imagewidth_\cameraindex}{\imagewidth_\cameraindex} \tan \frac{\azimuthangle^{\textnormal{max}}_\cameraindex}{2}}$.
Plugging the obtained expressions of $d$ and $\azimuthangle$ in the equations of $\userpositionxy_x$ and $\userpositionxy_y$ provides the forms shown in eqs.~\eqref{eq:mm_paper_pixel_localizationx} and~\eqref{eq:mm_paper_pixel_localizationy}.

\subsection{Proof of Lemma~\ref{lemma:mm_paper_composition_equivariance}}
\label{proof:mm_paper_composition_equivariance}
$\forall g \in \group, \groupactiontwo_g\sbrk{f_2 \circ f_1 \rbrk{\cdot}} \overset{(a)}{=} f_2 \rbrk{\permutation_g \sbrk{f_1 \rbrk{\cdot}}} \overset{(b)}{=} f_2 \circ f_1 \rbrk{\groupaction_g \sbrk{\cdot}}$, where (a) is due to the equivariance of $f_2$ and (b) is due to the equivariance of $f_1$.

\subsection{Proof of Lemma~\ref{lemma:mm_paper_nonlinearity_equivariance}}
\label{proof:mm_paper_nonlinearity_equivariance}
Assume $f$ is $\rbrk{\groupaction_g, \groupactiontwo_g}$ equivariant, and $\varsigma: \mathbb{R} \to \mathbb{R}$ is a point-wise nonlinear function. $\forall g \in \group, \groupactiontwo_g \sbrk{\varsigma\rbrk{f \rbrk{\cdot}}} = \varsigma \rbrk{\groupactiontwo_g \sbrk{f \rbrk{\cdot}}} = \varsigma \rbrk{ f \rbrk{\groupaction_g\sbrk{\cdot}}}$. Hence $\varsigma \circ f$ is $\rbrk{\groupaction_g, \groupactiontwo_g}$ equivariant.

\subsection{Proof of Lemma~\ref{lemma:mm_paper_linearcombination_equivariance}}
\label{proof:mm_paper_linearcombination_equivariance}
$\forall g \in \group, \permutation_g \sbrk{\sum_{\rbrk{\agentindex^\prime, \agentindex} \in \commgraphedgeset} \messages^\layerindex_{\agentindex^\prime \to \agentindex}} \overset{(a)}{=} \sum_{\rbrk{\agentindex^\prime, \agentindex} \in \commgraphedgeset} \permutation_g \sbrk{\messages^\layerindex_{\agentindex^\prime \to \agentindex}} \overset{(b)}{=} \sum_{\rbrk{\agentindex^\prime, \agentindex} \in \commgraphedgeset} \messages_\agentindex \rbrk{\permutation_g \sbrk{\encoding_\agentindex^{\layerindex}}, \rotationmatrix_g \sbrk{\edgefeatures_{\agentindex^\prime, \agentindex}}}$, where (a) is due to the linearity of $\permutation_g$ and (b) is due to the equivariance of each function (eq.~\eqref{eq:mm_paper_message_equivariance}).

\subsection{Proof of Proposition~\ref{proposition:mm_paper_policy_equivariance}}
\label{proof:mm_paper_policy_equivariance}
Assume that the global state $\state$ transforms by a rotation to end up at $\groupaction_g\sbrk{\state}$ for some $g \in \group$.
At the level of each individual agent $\agentindex$ at location $\agentposition_\agentindex$, the perceived state $\state_\agentindex$ is a rotated version of the state perceived by the agent at location $\rotationmatrix_g \sbrk{\agentposition_\agentindex}$.
At the graph level, the edges of the communication graph become: $\rotationmatrix_g \sbrk{\agentposition_\agentindex} - \rotationmatrix_g \sbrk{\agentposition_{\agentindex^\prime}} = \rotationmatrix_g \sbrk{\agentposition_\agentindex - \agentposition_{\agentindex^\prime}} = \rotationmatrix_g \sbrk{\edgefeatures_{\agentindex, \agentindex^\prime}}$, due to the linearity of $\rotationmatrix_g$.
This means that edge features transform by $\rotationmatrix_g$.

By design, whenever each local state $\state_\agentindex$ transforms by $\groupaction_g$, meaning the location of each individual terminal is rotated by $\rotationmatrix_g$, each agent's initial encoding permutes as $\permutation_g\sbrk{\encoding_\agentindex^0}$.
We now study the behaviour of an arbitrary message passing layer $\layerindex = 1, \dots, \numlayers$,
\begin{oonecolumn}
\begin{align}
    \permutation_g \sbrk{\update_\agentindex \rbrk{\encoding_\agentindex^\layerindex, \aggregatedmessages_\agentindex^\layerindex}} 
    &= \permutation_g \sbrk{\update_\agentindex \rbrk{\encoding_\agentindex^\layerindex, \textstyle\sum_{\rbrk{\agentindex^\prime, \agentindex} \in \commgraphedgeset} \messages^\layerindex_{\agentindex^\prime \to \agentindex}}} \\ 
    &= \permutation_g \sbrk{\update_\agentindex \rbrk{\encoding_\agentindex^\layerindex, \textstyle\sum_{\rbrk{\agentindex^\prime, \agentindex} \in \commgraphedgeset} \messages_\agentindex \rbrk{\encoding_\agentindex^{\layerindex}, \edgefeatures_{\agentindex^\prime, \agentindex}}}} \\
    \label{eq:mm_paper_update_equi}
    &= \update_\agentindex \rbrk{\permutation_g \sbrk{\encoding_\agentindex^\layerindex}, \permutation_g \sbrk{\textstyle\sum_{\rbrk{\agentindex^\prime, \agentindex} \in \commgraphedgeset} \messages_\agentindex \rbrk{\encoding_\agentindex^{\layerindex}, \edgefeatures_{\agentindex^\prime, \agentindex}}}} \\
    \label{eq:mm_paper_message_equi}
    &= \update_\agentindex \rbrk{\permutation_g \sbrk{\encoding_\agentindex^\layerindex}, \textstyle\sum_{\rbrk{\agentindex^\prime, \agentindex,} \in \commgraphedgeset} \messages_\agentindex \rbrk{\permutation_g \sbrk{\encoding_\agentindex^{\layerindex}}, \rotationmatrix_g \sbrk{\edgefeatures_{\agentindex^\prime, \agentindex}}}}
\end{align}
\end{oonecolumn}%
\begin{ttwocolumn}
\begin{align}
    \permutation_g \sbrk{\update_\agentindex \rbrk{\encoding_\agentindex^\layerindex, \aggregatedmessages_\agentindex^\layerindex}} 
    &= \permutation_g \sbrk{\update_\agentindex \rbrk{\encoding_\agentindex^\layerindex, \textstyle\sum_{\rbrk{\agentindex^\prime, \agentindex} \in \commgraphedgeset} \messages^\layerindex_{\agentindex^\prime \to \agentindex}}} \\ 
    &\hspace{-6em}= \permutation_g \sbrk{\update_\agentindex \rbrk{\encoding_\agentindex^\layerindex, \textstyle\sum_{\rbrk{\agentindex^\prime, \agentindex} \in \commgraphedgeset} \messages_\agentindex \rbrk{\encoding_\agentindex^{\layerindex}, \edgefeatures_{\agentindex^\prime, \agentindex}}}} \\
    \label{eq:mm_paper_update_equi}
    &\hspace{-6em}= \update_\agentindex \rbrk{\permutation_g \sbrk{\encoding_\agentindex^\layerindex}, \permutation_g \sbrk{\textstyle\sum_{\rbrk{\agentindex^\prime, \agentindex} \in \commgraphedgeset} \messages_\agentindex \rbrk{\encoding_\agentindex^{\layerindex}, \edgefeatures_{\agentindex^\prime, \agentindex}}}} \\
    \label{eq:mm_paper_message_equi}
    &\hspace{-6em}= \update_\agentindex \rbrk{\permutation_g \sbrk{\encoding_\agentindex^\layerindex}, \textstyle\sum_{\rbrk{\agentindex^\prime, \agentindex} \in \commgraphedgeset} \messages_\agentindex \rbrk{\permutation_g \sbrk{\encoding_\agentindex^{\layerindex}}, \rotationmatrix_g \sbrk{\edgefeatures_{\agentindex^\prime, \agentindex}}}}
\end{align}
\end{ttwocolumn}%
where \eqref{eq:mm_paper_update_equi} and \eqref{eq:mm_paper_message_equi} are valid by construction of the update and message functions.
Notice that the output of every update layer is permuted only when, for all agents, both the local features permute by $\permutation_g$ (which are initially due to rotations of the initial local states by $\rotationmatrix_g$) and edges transform by $\rotationmatrix_g$, which is equivalent to a global state transformation $\groupaction_g\sbrk{\state}$.
Since all layers share the same group transformations, stacking the layers successively preserves equivariance (see Lemma~\ref{lemma:mm_paper_composition_equivariance}).
By construction, the policy head satisfies: 
\begin{equation}
\policy_\agentindex\rbrk{\groupaction_g \sbrk{ \encoding_\agentindex^\numlayers}} = \localpolicy_\agentindex\rbrk{\permutation_g \sbrk{ \encoding_\agentindex^\numlayers}} = \groupactiontwo_g \sbrk{\localpolicy_\agentindex\rbrk{\encoding_\agentindex^\numlayers}} = \groupactiontwo_g \sbrk{\policy_\agentindex\rbrk{\state_\agentindex}},
\end{equation}
where the permutation $\groupactiontwo_g$ is designed to guarantee global policy equivariance $\groupactiontwo_g^{\state} \sbrk{\policy\rbrk{\state}} = \policy\rbrk{\groupaction_g \sbrk{\state}}$.
Similarly, the value head satisfies: 
\begin{equation}
\valuefunction_\agentindex\rbrk{\groupaction_g \sbrk{ \encoding_\agentindex^\numlayers}} = \localvalue_\agentindex\rbrk{\permutation_g \sbrk{ \encoding_\agentindex^\numlayers}} = \localvalue_\agentindex \rbrk{\encoding_\agentindex^\numlayers} = \valuefunction_\agentindex \rbrk{\state_\agentindex},
\end{equation}
which concludes the proof.

\subsection{Proof of Proposition~\ref{proposition:mm_paper_partial_symmetry}}
\label{proof:mm_paper_partial_symmetry}
We start by recalling that the action-value function of an \gls{mmdp} under a joint policy $\policy$ is:
\begin{equation}
    \actionvaluefunction^\policy \rbrk{\state, \beam} = \reward\rbrk{\state, \beam} + \discountfactor \sum_{\state^\prime \in \statespace} \transitionfunction\rbrk{\state, \beam, \state^\prime} \valuefunction^\policy \rbrk{\state^\prime}
\end{equation}
Under an optimal policy $\policy^\star$ solving the \gls{mmdp}, the optimal action value function satisfies:
\begin{equation}
    \actionvaluefunction^\star \rbrk{\state, \beam} = \sum_{\state^\prime \in \statespace} \transitionfunction\rbrk{\state, \beam, \state^\prime} \sbrk{\reward\rbrk{\state, \beam} + \discountfactor \max_{\beam^\prime \in \beamset} \actionvaluefunction^\star \rbrk{\state^\prime, \beam^\prime}}.
\end{equation}
Following~\cite{ravindran2001symmetries}, the $m$-step optimal discounted action value function is defined recursively for all state action pairs $\rbrk{\state, \beam}$ and non-negative integers $m$ as:
\begin{equation}\label{eq:mm_paper_actionvaluefunction_m}
    \actionvaluefunction_m \rbrk{\state, \beam} = \reward\rbrk{\state, \beam} + \discountfactor \sum_{\state^\prime \in \statespace} \transitionfunction\rbrk{\state, \beam, \state^\prime} \max_{\beam^\prime \in \beamset} \actionvaluefunction_{m-1} \rbrk{\state^\prime, \beam^\prime},
\end{equation}
Herein, we set $\actionvaluefunction_{-1}\rbrk{\state,\beam}=0$, $\valuefunction_{m}\rbrk{\state} = \max_{\beam \in \beamset} \actionvaluefunction_{m} \rbrk{\state, \beam}$, and rewrite \eqref{eq:mm_paper_actionvaluefunction_m} as:
\begin{equation}\label{eq:mm_paper_Q_def}
    \actionvaluefunction_m \rbrk{\state, \beam} = \reward\rbrk{\state, \beam} + \discountfactor \sum_{\state^\prime \in \statespace} \transitionfunction\rbrk{\state, \beam, \state^\prime} \valuefunction_{m-1}\rbrk{\state^\prime}.
\end{equation}
The optimal action value function is obtained as: $\actionvaluefunction^\star\rbrk{\state,\beam} = \lim_{m\to\infty}\actionvaluefunction_m\rbrk{\state,\beam}$.

We prove the proposition by induction on $m$, building on the arguments of~\cite{yu2024leveraging}.

For the case of $m=0$, we have $\actionvaluefunction_0 \rbrk{\state, \beam} = \reward\rbrk{\state, \beam}$, hence $\forall g \in \group, \state \in \statespace, \beam \in \beamset$:
\begin{align}
    &\left\lvert \actionvaluefunction_0 \rbrk{\state, \beam} - \actionvaluefunction_0 \rbrk{\groupaction_g\sbrk{\state}, \groupactiontwo_g^\state\sbrk{\beam}} \right\rvert \\
    = &\left\lvert \reward \rbrk{\state, \beam} - \reward \rbrk{\groupaction_g\sbrk{\state}, \groupactiontwo_g^\state\sbrk{\beam}} \right\rvert \\
    \leq &\;\partialsymmetryR
\end{align}
where the inequality is due to the assumption in~\eqref{eq:mm_paper_mmdp_partial_symmetric_reward}.

For the case of $m=1$, we have $\forall g \in \group, \state \in \statespace, \beam \in \beamset$%
\begin{ttwocolumn}%
    , the upper bound derived on the top of the next page (eqs.~\eqref{eq:mm_paper_Q1_1},~\eqref{eq:mm_paper_Q1_2},~\eqref{eq:mm_paper_Q1_3}), where~\eqref{eq:mm_paper_Q1_2} is a direct application of the definition in~\eqref{eq:mm_paper_Q_def} to~\eqref{eq:mm_paper_Q1_1}, and~\eqref{eq:mm_paper_Q1_3} is due to the subadditivity of the absolute value.
\begin{table*}[t]
\revise{
\begin{align}
    \label{eq:mm_paper_Q1_1}
    &\left\lvert \actionvaluefunction_1 \rbrk{\state, \beam} - \actionvaluefunction_1 \rbrk{\groupaction_g\sbrk{\state}, \groupactiontwo_g^\state\sbrk{\beam}} \right\rvert \\
    \label{eq:mm_paper_Q1_2}
    = &\left\lvert \reward\rbrk{\state, \beam} + \discountfactor \sum_{\state^\prime \in \statespace} \transitionfunction \rbrk{\state, \beam, \state^\prime} \valuefunction_0 \rbrk{\state^\prime} - \reward\rbrk{\groupaction_g\sbrk{\state}, \groupactiontwo_g^\state\sbrk{\beam}} - \discountfactor \sum_{\state^\prime \in \statespace} \transitionfunction \rbrk{\groupaction_g\sbrk{\state}, \groupactiontwo_g^\state\sbrk{\beam}, \groupaction_g\sbrk{\state^\prime}} \valuefunction_0 \rbrk{\groupaction_g\sbrk{\state^\prime}} \right\rvert \\
    \label{eq:mm_paper_Q1_3}
    \leq & \underbracket[0.100ex][0.500ex]{\left\lvert \reward\rbrk{\state, \beam} - \reward\rbrk{\groupaction_g\sbrk{\state}, \groupactiontwo_g^\state\sbrk{\beam}} \right\rvert}_{\leq \partialsymmetryR} + \discountfactor \underbracket[0.100ex][0.500ex]{\left\lvert \mathbb{E}_{\state^\prime \sim \transitionfunction \rbrk{\state^\prime \mid \beam, \state}} \sbrk{\valuefunction_0 \rbrk{\state^\prime}} - \mathbb{E}_{\state^\prime \sim \transitionfunction\rbrk{\groupaction_g\sbrk{\state^\prime} \mid  \groupactiontwo_g^\state\sbrk{\beam}, \groupaction_g\sbrk{\state}}} \sbrk{\valuefunction_0 \rbrk{\groupaction_g\sbrk{\state^\prime}}} \right\rvert}_{\mathbb{X}_0}
\end{align}
\hrulefill
\begin{align}
    \label{eq:mm_paper_Xterm_1}
    \mathbb{X}_0 &\leq \left\lvert \mathbb{E}_{\state^\prime \sim \transitionfunction \rbrk{\state^\prime \mid \beam, \state}} \sbrk{\valuefunction_0 \rbrk{\groupaction_g\sbrk{\state^\prime}} + \partialsymmetryR} - \mathbb{E}_{\state^\prime \sim \transitionfunction\rbrk{\groupaction_g\sbrk{\state^\prime} \mid  \groupactiontwo_g^\state\sbrk{\beam}, \groupaction_g\sbrk{\state}}} \sbrk{\valuefunction_0 \rbrk{\groupaction_g\sbrk{\state^\prime}}} \right\rvert \\
    \label{eq:mm_paper_Xterm_2}
    &\leq \partialsymmetryR + \underbracket[0.100ex][0.500ex]{\left\lvert \mathbb{E}_{\state^\prime \sim \transitionfunction \rbrk{\state^\prime \mid \beam, \state}} \sbrk{\valuefunction_0 \rbrk{\groupaction_g\sbrk{\state^\prime}}} - \mathbb{E}_{\state^\prime \sim \transitionfunction\rbrk{\groupaction_g\sbrk{\state^\prime} \mid  \groupactiontwo_g^\state\sbrk{\beam}, \groupaction_g\sbrk{\state}}} \sbrk{\valuefunction_0 \rbrk{\groupaction_g\sbrk{\state^\prime}}} \right\rvert}_{\leq \partialsymmetryT}
\end{align}
\hrulefill
\begin{align}
    \label{eq:mm_paper_Xterm_3}
    \mathbb{X}_0 &\leq \left\lvert \mathbb{E}_{\state^\prime \sim \transitionfunction \rbrk{\state^\prime \mid \beam, \state}} \sbrk{\valuefunction_0 \rbrk{\state^\prime}} - \mathbb{E}_{\state^\prime \sim \transitionfunction\rbrk{\groupaction_g\sbrk{\state^\prime} \mid  \groupactiontwo_g^\state\sbrk{\beam}, \groupaction_g\sbrk{\state}}} \sbrk{\valuefunction_0 \rbrk{\state^\prime} + \partialsymmetryR} \right\rvert \\
    \label{eq:mm_paper_Xterm_4}
    &\leq \partialsymmetryR + \underbracket[0.100ex][0.500ex]{\left\lvert \mathbb{E}_{\state^\prime \sim \transitionfunction \rbrk{\state^\prime \mid \beam, \state}} \sbrk{ \valuefunction_0 \rbrk{\state^\prime}} - \mathbb{E}_{\state^\prime \sim \transitionfunction \rbrk{\groupaction_g\sbrk{\state^\prime} \mid \groupactiontwo_g^\state \sbrk{\beam}, \groupaction_g\sbrk{\state}}} \valuefunction_0 \rbrk{\state^\prime} \right\rvert}_{\leq \partialsymmetryT}
\end{align}
\hrulefill
}
\end{table*}
We now examine the term $\mathbb{X}_0$ in~\eqref{eq:mm_paper_Q1_3}. If the argument of the absolute value in $\mathbb{X}_0$ is non-negative, then we have the inequalities in~\eqref{eq:mm_paper_Xterm_1} and~\eqref{eq:mm_paper_Xterm_2}, where in~\eqref{eq:mm_paper_Xterm_1} we upper bound the one-step value difference as $\left\lvert \valuefunction_0 \rbrk{\state} - \valuefunction_0 \rbrk{\groupaction_g\sbrk{\state}} \right\rvert \leq \partialsymmetryR, \;\forall g \in \group, \state \in \statespace$ which holds due to~\eqref{eq:mm_paper_mmdp_partial_symmetric_reward}, and in~\eqref{eq:mm_paper_Xterm_2} we apply the subadditivity of $\lvert\cdot\rvert$. Likewise, if the argument of the absolute value in $\mathbb{X}_0$ is negative, then we have the inequalities in~\eqref{eq:mm_paper_Xterm_3} and~\eqref{eq:mm_paper_Xterm_4}, obtained using similar steps. It is worth mentioning that in~\eqref{eq:mm_paper_Xterm_2} and~\eqref{eq:mm_paper_Xterm_4}, we bound the second term using~\eqref{eq:mm_paper_mmdp_partial_symmetric_transition}, since the reward $\reward$ is bounded, hence $\actionvaluefunction_m$ and $\valuefunction_m$ are elements of $\mathcal{H}$.
\end{ttwocolumn}%
\begin{oonecolumn}%
    :
\begin{align}
    \label{eq:mm_paper_Q1_1}
    &\left\lvert \actionvaluefunction_1 \rbrk{\state, \beam} - \actionvaluefunction_1 \rbrk{\groupaction_g\sbrk{\state}, \groupactiontwo_g^\state\sbrk{\beam}} \right\rvert \\
    \label{eq:mm_paper_Q1_2}
    = &\left\lvert \reward\rbrk{\state, \beam} + \discountfactor \sum_{\state^\prime \in \statespace} \transitionfunction \rbrk{\state, \beam, \state^\prime} \valuefunction_0 \rbrk{\state^\prime} - \reward\rbrk{\groupaction_g\sbrk{\state}, \groupactiontwo_g^\state\sbrk{\beam}} - \discountfactor \sum_{\state^\prime \in \statespace} \transitionfunction \rbrk{\groupaction_g\sbrk{\state}, \groupactiontwo_g^\state\sbrk{\beam}, \groupaction_g\sbrk{\state^\prime}} \valuefunction_0 \rbrk{\groupaction_g\sbrk{\state^\prime}} \right\rvert \\
    \label{eq:mm_paper_Q1_3}
    \leq & \underbracket[0.100ex][0.500ex]{\left\lvert \reward\rbrk{\state, \beam} - \reward\rbrk{\groupaction_g\sbrk{\state}, \groupactiontwo_g^\state\sbrk{\beam}} \right\rvert}_{\leq \partialsymmetryR} + \discountfactor \underbracket[0.100ex][0.500ex]{\left\lvert \mathbb{E}_{\state^\prime \sim \transitionfunction \rbrk{\state^\prime \mid \beam, \state}} \sbrk{\valuefunction_0 \rbrk{\state^\prime}} - \mathbb{E}_{\state^\prime \sim \transitionfunction\rbrk{\groupaction_g\sbrk{\state^\prime} \mid  \groupactiontwo_g^\state\sbrk{\beam}, \groupaction_g\sbrk{\state}}} \sbrk{\valuefunction_0 \rbrk{\groupaction_g\sbrk{\state^\prime}}} \right\rvert}_{\mathbb{X}_0}
\end{align}
where~\eqref{eq:mm_paper_Q1_2} is a direct application of the definition in~\eqref{eq:mm_paper_Q_def} to~\eqref{eq:mm_paper_Q1_1}, and~\eqref{eq:mm_paper_Q1_3} is due to the subadditivity of the absolute value.
We now examine the term $\mathbb{X}_0$ in~\eqref{eq:mm_paper_Q1_3}. If the argument of the absolute value in $\mathbb{X}_0$ is non-negative, then we have:
\begin{align}
    \label{eq:mm_paper_Xterm_1}
    \mathbb{X}_0 &\leq \left\lvert \mathbb{E}_{\state^\prime \sim \transitionfunction \rbrk{\state^\prime \mid \beam, \state}} \sbrk{\valuefunction_0 \rbrk{\groupaction_g\sbrk{\state^\prime}} + \partialsymmetryR} - \mathbb{E}_{\state^\prime \sim \transitionfunction\rbrk{\groupaction_g\sbrk{\state^\prime} \mid  \groupactiontwo_g^\state\sbrk{\beam}, \groupaction_g\sbrk{\state}}} \sbrk{\valuefunction_0 \rbrk{\groupaction_g\sbrk{\state^\prime}}} \right\rvert \\
    \label{eq:mm_paper_Xterm_2}
    &\leq \partialsymmetryR + \underbracket[0.100ex][0.500ex]{\left\lvert \mathbb{E}_{\state^\prime \sim \transitionfunction \rbrk{\state^\prime \mid \beam, \state}} \sbrk{\valuefunction_0 \rbrk{\groupaction_g\sbrk{\state^\prime}}} - \mathbb{E}_{\state^\prime \sim \transitionfunction\rbrk{\groupaction_g\sbrk{\state^\prime} \mid  \groupactiontwo_g^\state\sbrk{\beam}, \groupaction_g\sbrk{\state}}} \sbrk{\valuefunction_0 \rbrk{\groupaction_g\sbrk{\state^\prime}}} \right\rvert}_{\leq \partialsymmetryT}
\end{align}
where in~\eqref{eq:mm_paper_Xterm_1} we upper bound the one-step value difference as $\left\lvert \valuefunction_0 \rbrk{\state} - \valuefunction_0 \rbrk{\groupaction_g\sbrk{\state}} \right\rvert \leq \partialsymmetryR, \;\forall g \in \group, \state \in \statespace$ which holds due to~\eqref{eq:mm_paper_mmdp_partial_symmetric_reward}, and in~\eqref{eq:mm_paper_Xterm_2} we apply the subadditivity of $\lvert\cdot\rvert$. Likewise, if the argument of the absolute value in $\mathbb{X}_0$ is negative, then we have:
\begin{align}
    \label{eq:mm_paper_Xterm_3}
    \mathbb{X}_0 &\leq \left\lvert \mathbb{E}_{\state^\prime \sim \transitionfunction \rbrk{\state^\prime \mid \beam, \state}} \sbrk{\valuefunction_0 \rbrk{\state^\prime}} - \mathbb{E}_{\state^\prime \sim \transitionfunction\rbrk{\groupaction_g\sbrk{\state^\prime} \mid  \groupactiontwo_g^\state\sbrk{\beam}, \groupaction_g\sbrk{\state}}} \sbrk{\valuefunction_0 \rbrk{\state^\prime} + \partialsymmetryR} \right\rvert \\
    \label{eq:mm_paper_Xterm_4}
    &\leq \partialsymmetryR + \underbracket[0.100ex][0.500ex]{\left\lvert \mathbb{E}_{\state^\prime \sim \transitionfunction \rbrk{\state^\prime \mid \beam, \state}} \sbrk{ \valuefunction_0 \rbrk{\state^\prime}} - \mathbb{E}_{\state^\prime \sim \transitionfunction \rbrk{\groupaction_g\sbrk{\state^\prime} \mid \groupactiontwo_g^\state \sbrk{\beam}, \groupaction_g\sbrk{\state}}} \valuefunction_0 \rbrk{\state^\prime} \right\rvert}_{\leq \partialsymmetryT}
\end{align}
obtained using similar steps. It is worth mentioning that in~\eqref{eq:mm_paper_Xterm_2} and~\eqref{eq:mm_paper_Xterm_4}, we bound the second term using~\eqref{eq:mm_paper_mmdp_partial_symmetric_transition}, since the reward $\reward$ is bounded, hence $\actionvaluefunction_m$ and $\valuefunction_m$ are elements of $\mathcal{H}$.
\end{oonecolumn}%
Therefore, \eqref{eq:mm_paper_Q1_3} yields:
\begin{align}
    \nonumber&\left\lvert \actionvaluefunction_1 \rbrk{\state, \beam} - \actionvaluefunction_1 \rbrk{\groupaction_g\sbrk{\state}, \groupactiontwo_g^\state\sbrk{\beam}} \right\rvert \\
    \leq&\; \partialsymmetryR + \discountfactor \mathbb{X}_0 \\
    \leq&\;\partialsymmetryR + \discountfactor \rbrk{\partialsymmetryR + \partialsymmetryT}
    = \rbrk{1 + \discountfactor} \partialsymmetryR + \discountfactor \partialsymmetryT
\end{align}

Now we assume that:
\begin{equation}\label{eq:mm_paper_induction_step}
    \left\lvert \actionvaluefunction_n\rbrk{\state, \beam} - \actionvaluefunction_n \rbrk{\groupaction_g\sbrk{\state}, \groupactiontwo_g^\state\sbrk{\beam}} \right\rvert \leq \sum_{l=0}^n \discountfactor^l \partialsymmetryR + \sum_{l=1}^n \discountfactor^l \partialsymmetryT
\end{equation}
holds $\forall\; n \leq m, g \in \group, \state \in \statespace, \beam \in \beamset$.
\begin{ttwocolumn}%
\begin{table*}[t]
\revise{
\begin{align}
    \label{eq:mm_paper_Qm_1}
    &\left\lvert \actionvaluefunction_{m+1} \rbrk{\state, \beam} - \actionvaluefunction_{m+1} \rbrk{\groupaction_g\sbrk{\state}, \groupactiontwo_g^\state\sbrk{\beam}} \right\rvert \\
    \label{eq:mm_paper_Qm_2}
    = &\left\lvert \reward\rbrk{\state, \beam} + \discountfactor \sum_{\state^\prime \in \statespace} \transitionfunction \rbrk{\state, \beam, \state^\prime} \valuefunction_m \rbrk{\state^\prime} - \reward\rbrk{\groupaction_g\sbrk{\state}, \groupactiontwo_g^\state\sbrk{\beam}} - \discountfactor \sum_{\state^\prime \in \statespace} \transitionfunction \rbrk{\groupaction_g\sbrk{\state}, \groupactiontwo_g^\state\sbrk{\beam}, \groupaction_g\sbrk{\state^\prime}} \valuefunction_m \rbrk{\groupaction_g\sbrk{\state^\prime}} \right\rvert \\
    \label{eq:mm_paper_Qm_3}
    \leq & \underbracket[0.100ex][0.500ex]{\left\lvert \reward\rbrk{\state, \beam} - \reward\rbrk{\groupaction_g\sbrk{\state}, \groupactiontwo_g^\state\sbrk{\beam}} \right\rvert}_{\leq \partialsymmetryR} + \discountfactor \underbracket[0.100ex][0.500ex]{\left\lvert \mathbb{E}_{\state^\prime \sim \transitionfunction \rbrk{\state^\prime \mid \beam, \state}} \sbrk{\valuefunction_m \rbrk{\state^\prime}} - \mathbb{E}_{\state^\prime \sim \transitionfunction\rbrk{\groupaction_g\sbrk{\state^\prime} \mid  \groupactiontwo_g^\state\sbrk{\beam}, \groupaction_g\sbrk{\state}}} \sbrk{\valuefunction_m \rbrk{\groupaction_g\sbrk{\state^\prime}}} \right\rvert}_{\mathbb{X}_m}
\end{align}
\hrulefill
}
\end{table*}
We study the $\rbrk{m+1}$-step optimal discounted action value function, as shown in~\eqref{eq:mm_paper_Qm_1},~\eqref{eq:mm_paper_Qm_2} and~\eqref{eq:mm_paper_Qm_3} on the top of the next page.
\end{ttwocolumn}%
\begin{oonecolumn}%
We study the $\rbrk{m+1}$-step optimal discounted action value function:
\begin{align}
    \label{eq:mm_paper_Qm_1}
    &\left\lvert \actionvaluefunction_{m+1} \rbrk{\state, \beam} - \actionvaluefunction_{m+1} \rbrk{\groupaction_g\sbrk{\state}, \groupactiontwo_g^\state\sbrk{\beam}} \right\rvert \\
    \label{eq:mm_paper_Qm_2}
    = &\left\lvert \reward\rbrk{\state, \beam} + \discountfactor \sum_{\state^\prime \in \statespace} \transitionfunction \rbrk{\state, \beam, \state^\prime} \valuefunction_m \rbrk{\state^\prime} - \reward\rbrk{\groupaction_g\sbrk{\state}, \groupactiontwo_g^\state\sbrk{\beam}} - \discountfactor \sum_{\state^\prime \in \statespace} \transitionfunction \rbrk{\groupaction_g\sbrk{\state}, \groupactiontwo_g^\state\sbrk{\beam}, \groupaction_g\sbrk{\state^\prime}} \valuefunction_m \rbrk{\groupaction_g\sbrk{\state^\prime}} \right\rvert \\
    \label{eq:mm_paper_Qm_3}
    \leq & \underbracket[0.100ex][0.500ex]{\left\lvert \reward\rbrk{\state, \beam} - \reward\rbrk{\groupaction_g\sbrk{\state}, \groupactiontwo_g^\state\sbrk{\beam}} \right\rvert}_{\leq \partialsymmetryR} + \discountfactor \underbracket[0.100ex][0.500ex]{\left\lvert \mathbb{E}_{\state^\prime \sim \transitionfunction \rbrk{\state^\prime \mid \beam, \state}} \sbrk{\valuefunction_m \rbrk{\state^\prime}} - \mathbb{E}_{\state^\prime \sim \transitionfunction\rbrk{\groupaction_g\sbrk{\state^\prime} \mid  \groupactiontwo_g^\state\sbrk{\beam}, \groupaction_g\sbrk{\state}}} \sbrk{\valuefunction_m \rbrk{\groupaction_g\sbrk{\state^\prime}}} \right\rvert}_{\mathbb{X}_m}.
\end{align}
\end{oonecolumn}%
Similarly to $\mathbb{X}_0$, we upper bound $\mathbb{X}_m$ by utilizing~\eqref{eq:mm_paper_induction_step} as follows:
\begin{equation}\label{eq:mm_paper_Xm}
    \mathbb{X}_m \leq \sum_{l=0}^m \discountfactor^l \partialsymmetryR + \sum_{l=0}^m \discountfactor^l \partialsymmetryT.
\end{equation}
Thus, plugging~\eqref{eq:mm_paper_Xm} in~\eqref{eq:mm_paper_Qm_3} yields:
\begin{align}
    &\left\lvert \actionvaluefunction_{m+1} \rbrk{\state, \beam} - \actionvaluefunction_{m+1} \rbrk{\groupaction_g\sbrk{\state}, \groupactiontwo_g^\state\sbrk{\beam}} \right\rvert \\
    \leq\; &\partialsymmetryR + \discountfactor \rbrk{\sum_{l=0}^m \discountfactor^l \partialsymmetryR + \sum_{l=0}^m \discountfactor^l \partialsymmetryT} \\
    \label{eq:mm_paper_final_bound_2}
    =& \sum_{l=0}^{m+1} \discountfactor^l \partialsymmetryR + \sum_{l=1}^{m+1} \discountfactor^l \partialsymmetryT
\end{align}
As such, the induction step is validated, and therefore~\eqref{eq:mm_paper_induction_step} holds for all non-negative integers. We can now bound the performance error using the relation between $\actionvaluefunction_m$ and $\actionvaluefunction^\star$ as follows, $\forall\; g \in \group, \state \in \statespace, \beam \in \beamset$:
\begin{align}
    \nonumber&\left\lvert \actionvaluefunction^\star \rbrk{\state, \beam} - \actionvaluefunction^\star \rbrk{\groupaction_g\sbrk{\state}, \groupactiontwo_g^\state\sbrk{\beam}} \right\rvert \\
    =& \left\lvert \lim_{m\to\infty} \actionvaluefunction_m \rbrk{\state, \beam} - \lim_{m\to\infty} \actionvaluefunction_m \rbrk{\groupaction_g\sbrk{\state}, \groupactiontwo_g^\state\sbrk{\beam}} \right\rvert \\
    \label{eq:mm_paper_cont_abs}
    =& \lim_{m\to\infty} \left\lvert \actionvaluefunction_m \rbrk{\state, \beam} - \actionvaluefunction_m \rbrk{\groupaction_g\sbrk{\state}, \groupactiontwo_g^\state\sbrk{\beam}} \right\rvert \\
    \label{eq:mm_paper_final_bound_1}
    \leq& \lim_{m\to\infty} \sbrk{\sum_{l=0}^{m} \discountfactor^l \partialsymmetryR + \sum_{l=1}^{m} \discountfactor^l \partialsymmetryT} \\
    \label{eq:mm_paper_final_bound_3}
    =&\rbrk{\partialsymmetryR + \discountfactor \partialsymmetryT}\rbrk{1-\discountfactor}^{-1}
\end{align}
where~\eqref{eq:mm_paper_cont_abs} follows from the continuity of the absolute value function, and in~\eqref{eq:mm_paper_final_bound_1} we use the bound obtained in~\eqref{eq:mm_paper_final_bound_2}.

\subsection{Details on Problem (\ref{eq:mm_paper_alignment_problem})}
\label{proof:mm_paper_details_matching_problem}
We propose to solve problem~\eqref{eq:mm_paper_alignment_problem} using a primal-dual method.
First, we re-write our optimization problem as follows:
\begin{align}\label{eq:mm_paper_alignment_problem_rewriten}
    &\underset{\matchingscalar_\agentindex, \; \matchingmatrix_\agentindex}{\text{minimize}}
    &&\left\lVert \distancematrix_\agentindex^{\text{im}} - \matchingscalar_\agentindex \, \matchingmatrix_\agentindex \, \distancematrix_\agentindex^{\text{csi}} \, \matchingmatrix_\agentindex^\trans \right\rVert_F^2, \\
    \nonumber&\text{subject to} &&\matchingscalar_\agentindex \geq 0 \\
    \nonumber&&& 0 \leq \sbrk{\matchingmatrix_{\agentindex}}_{i,j} \leq 1 \\
    \nonumber&&& \matchingmatrix_\agentindex^\trans \mathbf{1}_{\numsamples_\agentindex^{\text{im}}} \leq \mathbf{1}_{\numsamples_\agentindex^{\text{csi}}} \\
    \nonumber&&& \matchingmatrix_\agentindex \mathbf{1}_{\numsamples_\agentindex^{\text{csi}}} = \mathbf{1}_{\numsamples_\agentindex^{\text{im}}}
\end{align}
The Lagrangian for problem~\eqref{eq:mm_paper_alignment_problem_rewriten} is expressed as:
\begin{align}
    \begin{split}
    \label{eq:mm_paper_alignment_problem_lagrangian}
    &\mathcal{L}_\agentindex = \left\lVert \distancematrix_\agentindex^{\text{im}} - \matchingscalar_\agentindex \, \matchingmatrix_\agentindex \, \distancematrix_\agentindex^{\text{csi}} \, \matchingmatrix_\agentindex^\trans \right\rVert_F^2 + \mathbf{u}_\agentindex^\trans \rbrk{\matchingmatrix_\agentindex \mathbf{1}_{\numsamples_\agentindex^{\text{csi}}} - \mathbf{1}_{\numsamples_\agentindex^{\text{im}}}} + \mathbf{q}_\agentindex^\trans \rbrk{\matchingmatrix_\agentindex^\trans \mathbf{1}_{\numsamples_\agentindex^{\text{im}}} - \mathbf{1}_{\numsamples_\agentindex^{\text{csi}}} + \mathbf{a}_\agentindex}
    \end{split}
\end{align}
where $\mathbf{u}_\agentindex$ and $\mathbf{q}_\agentindex$ are Lagrange multipliers, and $\mathbf{a}_\agentindex$ is a slackness variable that transforms the inequality constraint to an equality.
The problem is solved by iteratively updating the optimization variables as shown in Algorithm~\ref{alg:mm_paper_primal_dual}.

\begin{algorithm}[t]
\DontPrintSemicolon
\caption{Primal-dual method for problem~\eqref{eq:mm_paper_alignment_problem}}\label{alg:mm_paper_primal_dual}
\hspace{-.7em}\textbf{Initialize} $\matchingmatrix_\agentindex, \matchingscalar_\agentindex, \mathbf{u}_\agentindex, \mathbf{q}_\agentindex, \mathbf{a}_\agentindex$, and learning rate $\varsigma$\;
\Repeat{\textnormal{convergence}}{
$\matchingmatrix_\agentindex \gets \matchingmatrix_\agentindex - \varsigma\, \nabla_{\matchingmatrix_\agentindex} \mathcal{L}_\agentindex$ (eq.~\eqref{eq:mm_paper_matchinf_deriv1})\;
$\mathbf{a}_\agentindex \gets \mathbf{a}_\agentindex - \varsigma\, \mathbf{q_a}$\;
$\mathbf{u}_\agentindex \gets \mathbf{u}_\agentindex + \varsigma\, \rbrk{\matchingmatrix_\agentindex \mathbf{1}_{\numsamples_\agentindex^{\text{csi}}} - \mathbf{1}_{\numsamples_\agentindex^{\text{im}}}}$\;
$\mathbf{q}_\agentindex \gets \mathbf{q}_\agentindex + \varsigma\, \rbrk{\matchingmatrix_\agentindex^\trans \mathbf{1}_{\numsamples_\agentindex^{\text{im}}} - \mathbf{1}_{\numsamples_\agentindex^{\text{csi}}} + \mathbf{a}_\agentindex}$\;
\vspace{0.2em}
$\matchingscalar_\agentindex \gets \frac{\text{Tr}\rbrk{\distancematrix_\agentindex^{\text{im}} \matchingmatrix_\agentindex \distancematrix_\agentindex^{\text{csi}} \matchingmatrix_\agentindex^\trans}}{\left\lVert \matchingmatrix_\agentindex \, \distancematrix_\agentindex^{\text{csi}} \, \matchingmatrix_\agentindex^\trans \right\rVert_F^2}$\;
}
\end{algorithm}

Note that for a fixed $\matchingmatrix_\agentindex$, the objective function is quadratic in $\matchingscalar_\agentindex$, since:
\begin{align}
    &\left\lVert \distancematrix_\agentindex^{\text{im}} - \matchingscalar_\agentindex \, \matchingmatrix_\agentindex \, \distancematrix_\agentindex^{\text{csi}} \, \matchingmatrix_\agentindex^\trans \right\rVert_F^2
    \\ \label{eq:mm_paper_scaling_deriv}\begin{split}=& \left\lVert \distancematrix_\agentindex^{\text{im}} \right\rVert_F^2 + \matchingscalar_\agentindex^2 \left\lVert \matchingmatrix_\agentindex \, \distancematrix_\agentindex^{\text{csi}} \, \matchingmatrix_\agentindex^\trans \right\rVert_F^2 -2\, \matchingscalar_\agentindex \text{Tr}\rbrk{\distancematrix_\agentindex^{\text{im}} \matchingmatrix_\agentindex \, \distancematrix_\agentindex^{\text{csi}} \, \matchingmatrix_\agentindex^\trans}.\end{split}
\end{align}
Hence the optimal scaling factor can be obtained by setting the derivative of~\eqref{eq:mm_paper_scaling_deriv} with respect to $\matchingscalar_\agentindex$ to $0$, which yields the update shown in the last line of the algorithm.

On the other hand, updating $\matchingmatrix_\agentindex$ necessitates the gradient term $\nabla_{\matchingmatrix_\agentindex} \mathcal{L}_\agentindex$. Looking at~\eqref{eq:mm_paper_alignment_problem_lagrangian}, we can write:
\begin{align}\label{eq:mm_paper_matching_grad}
\begin{split}
    \nabla_{\matchingmatrix_\agentindex} \mathcal{L}_\agentindex =& \nabla_{\matchingmatrix_\agentindex} \rbrk{\left\lVert \distancematrix_\agentindex^{\text{im}} - \matchingscalar_\agentindex \, \matchingmatrix_\agentindex \, \distancematrix_\agentindex^{\text{csi}} \, \matchingmatrix_\agentindex^\trans \right\rVert_F^2} + \mathbf{u}_\agentindex \mathbf{1}_{\numsamples_\agentindex^{\text{csi}}}^\trans + \mathbf{1}_{\numsamples_\agentindex^{\text{im}}} \mathbf{q}_\agentindex^\trans.
\end{split}
\end{align}
We further expand the first term as follows:
\begin{align}
    &\nabla_{\matchingmatrix_\agentindex} \rbrk{\left\lVert \distancematrix_\agentindex^{\text{im}} - \matchingscalar_\agentindex \, \matchingmatrix_\agentindex \, \distancematrix_\agentindex^{\text{csi}} \, \matchingmatrix_\agentindex^\trans \right\rVert_F^2} \\
    =& \nabla_{\matchingmatrix_\agentindex} \rbrk{\matchingscalar_\agentindex^2 \left\lVert \matchingmatrix_\agentindex \, \distancematrix_\agentindex^{\text{csi}} \, \matchingmatrix_\agentindex^\trans \right\rVert_F^2 -2\, \matchingscalar_\agentindex \text{Tr}\rbrk{\distancematrix_\agentindex^{\text{im}} \matchingmatrix_\agentindex \, \distancematrix_\agentindex^{\text{csi}} \, \matchingmatrix_\agentindex^\trans}} \\
    \begin{split}=& \matchingscalar_\agentindex^2 \, \nabla_{\matchingmatrix_\agentindex} \rbrk{\text{Tr}\rbrk{\matchingmatrix_\agentindex \, \distancematrix_\agentindex^{\text{csi}} \, \matchingmatrix_\agentindex^\trans \, \matchingmatrix_\agentindex \, \distancematrix_\agentindex^{\text{csi}} \, \matchingmatrix_\agentindex^\trans}} -2\, \matchingscalar_\agentindex \nabla_{\matchingmatrix_\agentindex} \rbrk{\text{Tr}\rbrk{\distancematrix_\agentindex^{\text{im}} \, \matchingmatrix_\agentindex \, \distancematrix_\agentindex^{\text{csi}} \, \matchingmatrix_\agentindex^\trans}}\end{split} \\
    \label{eq:mm_paper_matchinf_deriv}
    =& 4 \, \matchingscalar_\agentindex^2 \, \matchingmatrix_\agentindex \, \distancematrix_\agentindex^{\text{csi}} \, \matchingmatrix_\agentindex^\trans\matchingmatrix_\agentindex \, \distancematrix_\agentindex^{\text{csi}} -4\, \matchingscalar_\agentindex \, \distancematrix_\agentindex^{\text{im}} \, \matchingmatrix_\agentindex \, \distancematrix_\agentindex^{\text{csi}}
\end{align}
where in the last equality we applied identities \cite[eq. (118)]{petersen2008matrix} and \cite[eq. (123)]{petersen2008matrix}.
Plugigng~\eqref{eq:mm_paper_matchinf_deriv} in~\eqref{eq:mm_paper_matching_grad} yields the gradient update of $\matchingmatrix_\agentindex$ as:
\begin{align}\label{eq:mm_paper_matchinf_deriv1}
\begin{split}
    \nabla_{\matchingmatrix_\agentindex} \mathcal{L}_\agentindex =& 
    4 \, \matchingscalar_\agentindex^2 \, \matchingmatrix_\agentindex \, \distancematrix_\agentindex^{\text{csi}} \, \matchingmatrix_\agentindex^\trans\matchingmatrix_\agentindex \, \distancematrix_\agentindex^{\text{csi}} -4\, \matchingscalar_\agentindex \, \distancematrix_\agentindex^{\text{im}} \, \matchingmatrix_\agentindex \, \distancematrix_\agentindex^{\text{csi}} + \mathbf{u}_\agentindex \mathbf{1}_{\numsamples_\agentindex^{\text{csi}}}^\trans + \mathbf{1}_{\numsamples_\agentindex^{\text{im}}} \mathbf{q}_\agentindex^\trans.
\end{split}
\end{align}

Finally, the computational complexity of the each primal-dual iteration is governed by the update of the matching matrix.
Thus, by examining~\eqref{eq:mm_paper_matchinf_deriv1}, we evince that computing the first term requires $O\rbrk{\rbrk{\numsamples_\agentindex^{\text{csi}}}^3 + \rbrk{\numsamples_\agentindex^{\text{csi}}}^2 \numsamples_\agentindex^{\text{im}}}$ operations, while the second term requires $O\rbrk{\rbrk{\numsamples_\agentindex^{\text{csi}}}^2 \numsamples_\agentindex^{\text{im}} + \rbrk{\numsamples_\agentindex^{\text{im}}}^2 \numsamples_\agentindex^{\text{csi}}}$ operations, and computing the last two terms is negligible requiring a complexity $O\rbrk{\numsamples_\agentindex^{\text{csi}} \numsamples_\agentindex^{\text{im}}}$.
Combining the complexity of the first two terms, and since the optimization is iterated for a finite number of steps, we conclude that our primal-dual algorithm has an $O\rbrk{\rbrk{\numsamples_\agentindex^{\text{csi}}}^3 + \rbrk{\numsamples_\agentindex^{\text{csi}}}^2 \numsamples_\agentindex^{\text{im}} + \rbrk{\numsamples_\agentindex^{\text{im}}}^2 \numsamples_\agentindex^{\text{csi}}}$ computational complexity.


\bibliographystyle{IEEEtran}
\bibliography{references}
\let\mybibitem\bibitem

\end{document}